\documentclass[11pt, oneside]{mnthesis}
\usepackage{epsfig,epic,eepic,units}
\usepackage{url}
\usepackage{longtable}
\usepackage{mathrsfs}
\usepackage{multirow}
\usepackage{mathtools}
\usepackage{algorithm, algorithmic}

\usepackage{mathtools}
\usepackage{framed}
\usepackage{multirow} 
\usepackage{makecell}
\usepackage{tabularx}
\usepackage{booktabs} 
\usepackage{ragged2e} 
\usepackage{graphicx,graphics,amsthm,amssymb,enumerate,tikz}
\usepackage{hyperref}
\usepackage{caption} 
\usepackage{float} 
\usepackage{arydshln}
\usepackage[table]{xcolor}

\theoremstyle{plain}
\newtheorem{theorem}{Theorem}
\newtheorem{proposition}{Proposition}
\newtheorem{lemma}{Lemma}

\theoremstyle{definition}
\newtheorem{definition}{Definition}
\newtheorem{assumption}{Assumption}
\newtheorem{example}{Example}
\theoremstyle{remark}
\newtheorem{remark}{Remark}
\usepackage{subcaption}
\usepackage{booktabs}       
\usepackage{makecell, multirow} 
\usepackage{colortbl}

\definecolor{magenta}{RGB}{255,0,255} 

\definecolor{lightergray}{RGB}{190,190,190}
\def\R{\mathbb{R}}
\def\z{\mathbf{z}}

\def\d{\mathbf{d}}
\def\c{\mathbf{c}}
\def\Z{\mathbf{Z}}

\def\Q{\mathbf{Q}}
\def\P{\mathbf{P}}

\def\V{\mathcal{L} \, V}
\def\N{\mathcal{N}}

\begin{document}
\bibliographystyle{hunsrt} 

\ms 

\title{\bf Distributed GNEP Algorithms without Multiplier Sharing and Applications to Multi-Robot Coordination and Contextual Bandit-Based Active Learning}
\author{Shao-An Yin}
\campus{University of Minnesota} 
\program{Electrical and Computer Engineering} 
\degree{DOCTOR OF PHILOSOPHY}
\director{Advisor: Dr. Nicola Elia} 

\submissionmonth{May} 
\submissionyear{2026} 

\abstract{

Recent advances in artificial intelligence have expanded the focus from classical optimization to include equilibrium analysis in noncooperative games. Many such games involve shared constraints, leading to Generalized Nash Equilibrium Problems (GNEPs). Existing distributed algorithms typically require agents to exchange Lagrange multipliers to enforce consensus and compute variational-GNEs (v-GNEs).

This work introduces fully distributed continuous-time algorithms and establishes convergence without requiring multiplier exchange, thereby reducing information exchange per iteration while improving privacy preservation. The analysis focuses on strongly monotone games with convex individual constraints and linear shared constraints. I also propose several discretization schemes for the continuous-time algorithms. The proposed approach converges to general GNEs, rather than being restricted to v-GNEs, with the attained equilibrium depending on the initialization. The effectiveness of the proposed method is demonstrated through applications in multi-robot coordination and placement.

In the second part, this work includes research conducted in collaboration with Amazon scientists. One of the most challenging problems in real-world machine learning is labeled data collection, which typically requires substantial human effort and cost. Active learning aims to reduce this labeling requirement. Existing handcrafted active learning strategies, however, generally perform well only on specific types of datasets, which are often unknown in advance. In this work, I propose using contextual bandits to adaptively select the most suitable active learning strategy. The effectiveness of the proposed approach is demonstrated on both Amazon’s internal datasets and publicly available external datasets, although only the external dataset results are presented due to confidentiality constraints.
}
\words{331}    
\copyrightpage 
\acknowledgements{

I would like to express my sincere gratitude to my advisor, Prof. Nicola Elia, for his guidance, encouragement, and support throughout my Ph.D. journey. His rigorous approach to research and countless discussions greatly shaped my understanding of distributed optimization, game theory, and mathematical analysis. I am especially grateful for his mentorship and patience during the development of this dissertation.

I also sincerely thank Prof. Mingyi Hong for his collaboration and many insightful research discussions that significantly influenced this work. I am grateful to Prof. Murti Salapaka, Prof. Ryan Caverly, and Prof. Donatello Materassi for agreeing to be on my dissertation committee and for their valuable feedback during my final defense. I also thank Prof. Maria Gini for serving on my preliminary examination committee and for her support during my graduate studies.

I would also like to thank my undergraduate advisor, Prof. Wen-Pin Shih, for introducing me to research and encouraging me to pursue graduate studies. I am similarly grateful to my master’s advisor, Prof. Blake Hannaford, for his guidance and mentorship during my time at the University of Washington.

Part of this dissertation includes research conducted in collaboration with Amazon scientists. I would particularly like to thank Dr. Jiacong Li for the close collaboration and many insightful discussions throughout these projects. I also thank Dr. Cecile Levasseur, Dr. Tianpei Xie, Dr. Dmitry Pavlov, and Wojciech Kowalinski for their collaboration, support, and insights.

In addition, I would like to acknowledge my collaborators, mentors, managers, and colleagues at Amazon, TSMC, the Allen Institute for Brain Science, Google, and the University of Minnesota for the valuable discussions, experiences, and opportunities to work on challenging real-world problems across both academia and industry. I am also grateful to the students I worked with during my time as a teaching assistant, whose questions and interactions often provided valuable perspectives and motivation throughout my graduate studies.

I would also like to thank Beñat Froemming-Aldanondo, Dr. Renat Sergazinov, my friends from internships, and my labmate for their companionship, support, and many valuable discussions throughout my Ph.D. journey.

This work was supported in part by NSF ECPN Award No. 2311007.

I would also like to thank my family and friends outside of academia for their continuous encouragement and support throughout this journey.

}


\beforepreface 

\figurespage
\tablespage

\afterpreface            


\chapter{Introduction}
\label{intro_chapter}


One recent development in machine learning has shifted the focus from finding optimal solutions to analyzing equilibrium points in noncooperative games \cite{NEURIPS2022_24f420aa}. In particular, algorithms for computing Nash Equilibria (NE) play an important role as subroutines in many multi-agent and distributed machine learning methods. Examples include multi-agent reinforcement learning and the associated stochastic games \cite{shapley_stochastic_1953, NashQ, littman_1994, meanfield, MPG}, as well as generative models such as generative adversarial networks \cite{GAN, pmlr-v119-jin20e, pmlr-v139-liu21d}.

While Nash Equilibrium Problems (NEPs) offer a strong theoretical foundation, applying them to real-world settings often introduces additional challenges. Many practical applications require solving an NEP with coupling constraints across different agents, such as total resource limits or collision avoidance \cite{NEURIPS2021_174a61b0, zhou-generalized-2005, maiorano-dynamics-2000, contreras_numerical_2004, le_cleach_algames_2020, Shen2023, ma_decentralized_2013}. 

A simple illustration is an autopilot system. In a standard NEP, each autopilot has its own cost function and enforces its own operational constraints, such as acceleration limits. However, since all autopilots operate on the same road, they must jointly satisfy shared safety constraints. This leads to a formulation beyond the classical NEP. When shared constraints are present, the appropriate framework is the Generalized Nash Equilibrium Problem (GNEP). GNEPs are strictly more general than NEPs, and significantly more difficult to solve.

Beyond modeling considerations, computational challenges also limit the real-world use of game-theoretic formulations in AI and machine learning. Even a two-agent NEP can be difficult, because NE are equilibrium points rather than optima \cite{comple}. As a result, one cannot easily monitor objective values to determine whether an algorithm is approaching a solution. This difficulty becomes even more pronounced in GNEPs.

Moreover, for large-scale systems, relying on a central authority to compute a Generalized Nash Equilibrium (GNE) is impractical. Centralized computation demands substantial communication and computational resources, and agents may be unwilling to share their local cost information due to privacy concerns. This motivates distributed computation. Each agent should compute its own decision locally, exchange only limited information, and collectively work toward a GNE.

There is an extensive body of existing literature focused on communication complexity in distributed convex optimization, particularly emphasizing communication overhead per iteration \cite{10.5555/2969239.2969435,McMahan2016CommunicationEfficientLO,NEURIPS2022_56bd2125,pmlr-v162-mishchenko22b,pmlr-v70-suresh17a,NEURIPS2018_3328bdf9}. Consensus-based methods inherently rely on augmented variables to achieve agreement among agents—a requirement that our proposed approach eliminates. Consequently, these augmented variables significantly increase information exchange per iteration. Moreover, the number of augmented variables typically depends on the network topology of the shared constraints; thus, for complex network structures, the number of required augmented variables increases quadratically.

\section{Overview of Thesis}

\subsection{Part I - Fully Distributed Algorithms Based on Continuous-Time Dynamics}

In this thesis, I focus on a fully distributed algorithm based on primal–dual dynamics. Here, fully distributed means that agents update their decisions using only local and neighbor information, without exchanging Lagrangian multipliers. This is in contrast to existing distributed methods, which typically require multiplier sharing or consensus to enforce shared constraints. By eliminating the need for multiplier exchange, the proposed algorithm is not restricted to variational-GNE and can converge to any generalized Nash equilibrium (GNE), with variational-GNE arising as a special case. I provide a theoretical analysis and establish convergence conditions for strongly monotone games with globally shared linear equality constraints. Future work will extend this approach to more general classes of shared constraints.

The contributions are summarized as follows:
\begin{enumerate}
    \item To the best of my knowledge, this is the first distributed algorithm that does not require consensus of Lagrangian multipliers, thereby significantly reducing information exchange.
    \item The proposed algorithm is the first distributed method capable of converging to any GNE, rather than only variational-GNE.
    \item Future work includes extending the algorithm to handle shared inequality constraints, and subsequently, more general convex shared constraints.
\end{enumerate}

\subsection{Part II - Contextual Bandit Based Active Learning}

In the second part, I incorporate active learning methods in the context of machine learning, where the objective is to minimize the effort required for querying labels from human annotators within a pool of unlabeled datasets while still achieving satisfactory performance from the underlying supervised learning model \cite{settles_active_2010}. The majority of existing work focuses on hand-crafted strategies for reducing sample complexity \cite{balcan_true_2010}. These strategies provide theoretical guarantees based on statistical assumptions about the unlabeled dataset. However, they typically perform well only for certain types of data, and in practice, the underlying data distribution is often unknown in advance.

On the other hand, \cite{baram_online_2004} addressed this challenge by combining multiple hand-crafted active learning strategies, aiming to leverage the strengths of each. This idea was further developed by \cite{Hsu_Lin_2015}, which emphasized adaptability and established connections to bandit algorithms through an adversarial bandit framework. However, adversarial bandit methods are often overly conservative and tend to perform similarly to random selection with equal probability. This behavior arises because they maintain persistent exploration without making explicit assumptions about the environment. In the context of active learning, the improvement in the performance of a machine learning model as more labeled data is queried provides useful contextual information. Therefore, I propose incorporating such contextual information into bandit-based methods to reduce conservativeness and enhance adaptability.

The contributions of this work are summarized as follows:
\begin{enumerate}
\item Compared to the recent trend of deep active learning, I focus on meta-learner-based active learning, as deep learning methods often perform worse on tabular data \cite{shwartz-ziv_tabular_2022}, which constitutes the majority of datasets in industry. This \textbf{meta-strategy} approach is more \textbf{industry-friendly} and can be easily integrated into existing pipelines.
\item While most prior work focuses on selecting a single unlabeled data point at each iteration, I instead consider batch selection in active learning and emphasize its \textbf{robustness across different batch sizes}, which is more aligned with practical industrial settings.
\item I emphasize adaptability to the best-performing hand-crafted active learning strategies under unknown datasets.
\item My approach shows that incorporating additional environmental context information for \textbf{reward prediction} improves adaptability compared to existing state-of-the-art methods, which primarily rely on adversarial bandits.
\end{enumerate}


\chapter{Generalized Nash Equilibrium Problem (GNEP)}
\label{background_chapter}

\section{Introduction}
In recent years, algorithms for noncooperative games have received significant research attention \cite{NEURIPS2022_24f420aa}. Among various solution concepts, the Nash Equilibrium (NE) has become a fundamental analytical tool in multi-agent systems and distributed machine learning. NE formulations have been widely applied in settings such as multi-agent reinforcement learning for stochastic games \cite{shapley_stochastic_1953, NashQ, littman_1994, meanfield, MPG} and in generative modeling frameworks, including generative adversarial networks (GANs) \cite{GAN, pmlr-v119-jin20e, pmlr-v139-liu21d}.

In many practical scenarios, however, agents must compute equilibria while satisfying shared constraints \cite{NEURIPS2021_174a61b0, tsaknakis_minimax_2023, pmlr-v48-pedregosa16}. This requirement leads to the Generalized Nash Equilibrium Problem (GNEP), a class of noncooperative games in which the feasible decision sets of agents are interdependent because of coupling constraints. GNEPs arise naturally across a wide range of engineering and economic applications, including power systems, communication networks, and traffic management. For example, in routing problems, each vehicle aims to minimize its travel time under common road capacity limits \cite{zhou-generalized-2005, maiorano-dynamics-2000}, and in power grid management, agents maximize utilities subject to shared resource constraints \cite{contreras_numerical_2004, le_cleach_algames_2020, Shen2023, ma_decentralized_2013}. Since these applications are often large-scale and involve distributed or privacy-sensitive information, a central research challenge is the development of scalable distributed algorithms that can compute generalized equilibria efficiently.

The study of GNEPs has a long history. In this class of games, each agent seeks to minimize its own local cost function, which depends on both its decision and the decisions of other agents. At the same time, each agent’s feasible set is influenced by the decisions of others, for example, when agents compete for limited resources or must satisfy shared safety constraints. A Generalized Nash Equilibrium (GNE) is a joint decision profile in which no agent can reduce its local cost by unilaterally changing its strategy to another feasible decision.

The study of the Generalized Nash Equilibrium Problem (GNEP) was initiated in the classical works of \cite{debreu-social-1952, rosen_existence_1965}. Early research concentrated on guaranteeing the existence of a solution, and an extensive survey of these foundational results can be found in \cite{facchinei_generalized_2010}. However, it is important to note that the existence of a solution does not imply that the solution can be easily computed, even when a centralized computing unit is available to compute the equilibrium and communicate the result to all agents.

In early centralized approaches, \cite{fukushima_restricted_2011} investigated the use of augmented penalty functions for solving the GNEP. The proposed method required solving an internal Nash Equilibrium Problem while progressively increasing the penalty associated with the shared constraint. Subsequent works, including \cite{facchinei_penalty_2010, kanzow_augmented_2016, dreves_nonsmooth_2011}, extended this idea and aimed to apply it to more general classes of GNEPs. However, \cite{dreves_nonsmooth_2011} identified a key limitation: the inner problem is often non-smooth, which creates significant difficulties when designing algorithms that operate in a distributed manner.

On the other hand, \cite{harker-1991} attempted to reduce the GNEP to a Quasi-Variational Inequality (QVI) problem \cite{pang_quasi-variational_2005, facchinei_computation_2011, migot_nonsmooth_2020, migot_parametrized_2020}. However, algorithmic research on QVIs is still relatively limited. As a result, \cite{harker-1991} and \cite{facchinei_generalized_2007} pursued an alternative reduction to a Variational Inequality (VI) problem. This reduction makes it possible to apply established VI algorithms \cite{facchinei-fischer-piccialli-2007, dreves_nonsmooth_2011}. The corresponding solution concept is often referred to as the normalized-GNE or variational-GNE.

\cite{dupuis_dynamical_1993} was the first to establish a formal connection between Variational Inequalities and Projected Dynamical Systems (PDS). The authors showed that the equilibrium points of the PDS coincide with the solutions of $VI(K,F)$, which provides a potential avenue for computing a variational-GNE. However, this correspondence alone does not ensure that the projected dynamics will converge. Subsequent work by \cite{zhang_stability_1995} provided fundamental results on attraction and stability for PDS, under the condition that the underlying problem is monotone. In addition, \cite{Xia_Wang_2000} proposed an alternative form of projected dynamics and analyzed its stability and convergence, although their results were restricted to symmetric problems.

Although computing a GNE point through a central authority has been widely studied, such an approach becomes impractical when the number of agents is large, due to the computational and communication burden associated with centralized processing. In addition, many agents may be unwilling to disclose their local cost functions to a central coordinator in order to preserve privacy. For these reasons, it is desirable for each agent to compute its own local decision in a distributed manner, using only limited information exchange, while still achieving a GNE.

Research on distributed methods for solving GNEPs remains limited, and most existing work focuses on solutions within the class of variational-GNEs. In particular, \cite{yi-pavel-2019, huang-distributed-2021} studied the generalized Nash problem through operator splitting techniques, whereas \cite{lin_distributed_2022} employed a singular-perturbation-based approach. However, all state-of-the-art methods rely on the uniqueness of the equilibrium point. As a result, these works primarily aim to enforce consensus among the Lagrangian multipliers, which ensures that the obtained solution corresponds to a variational-GNE associated with a unique VI solution. Table \ref{tab:algs} provides a summary of the main existing algorithms.

\begin{table*}[ht]
  \centering
  \renewcommand{\arraystretch}{1.3}
  \begin{tabularx}{\textwidth}{|>{\centering\arraybackslash}p{2.3cm}|X|X|}
    \hline
    \textbf{Perspective} & \textbf{Approaches} & \textbf{Comment}\\
    \hline
    \multirow{2}{*}[-1.5ex]{\makecell[l]{\textbf{Central}\\\textbf{Computing}}}
      & Penalty/ Augmented Lagrangian Method \cite{fukushima_restricted_2011, facchinei_penalty_2010, kanzow_augmented_2016, dreves_nonsmooth_2011}
      & Reduced the GNEP problem to the NEP problem with a penalty function. \\
    \cline{2-3}
      & VI/ Parameterized VI Reduction \cite{harker-1991,pang_quasi-variational_2005, facchinei-fischer-piccialli-2007, facchinei_computation_2011,facchinei_generalized_2007, nabetani_parametrized_2011, migot_parametrized_2020}
      & Reduced to VI/Parameterized VI problem and utilized existing VI solvers. \\
    \hline
    \multirow{2}{*}[-1.5ex]{\makecell[l]{\textbf{Distributed}\\\textbf{Computing}}}
      & Discrete-Time Dynamics \cite{paccagnan_distributed_2016, yi-pavel-2019, huang-distributed-2021}
      & Operator Splitting Method, typically targets \textbf{variational-GNE}. \\
    \cline{2-3}
      & Continuous-Time Dynamics \cite{bianchi-continuous-time-2021,lin_distributed_2022}
      & Targets \textbf{variational-GNE}. \\
    \hline
  \end{tabularx}
  \caption{Summary of Main Existing Algorithms for GNEPs}
  \label{tab:algs}
\end{table*}

\section{Generalized Nash Equilibrium Problems (GNEPs)}
I first introduced the GNEP in a formal manner. Consider a game with $N$ agents. Each agent $v$ controls decision variables $z_v \in \mathbb{R}^{n_v}$ and has a local cost function $f_v(z_v,\, \mathbf{z}_{-v}^f)$, which depends on its own decision variables as well as those of other agents. The vector $\mathbf{z}_{-v}^f$ denotes the subset of decision variables from all agents other than $v$ that influence $f_v(\cdot)$, thereby inducing a cost-function dependency graph $\mathcal G_f = (\mathcal V, \mathcal E_f)$. Each agent is also subject to individual constraints of the form $h_v(z_v) \le \mathbf{0}_{q_v}$. Let $\mathbf{z} = [z_1^\top, \dots, z_N^\top]^\top \in \mathbb{R}^n$ denote the joint decision vector, where $n = \sum_{v=1}^N n_v$, and let $q = \sum_{v=1}^N q_v$ denote the total number of individual constraints.

In addition, each agent may face shared constraints that depend on the decisions of others. In particular, the shared equality and inequality constraints of agent $v$ are given by
\[
    \psi_v(z_v,\, \mathbf{z}_{-v}^\psi) = \mathbf{0}_{l_v}, \qquad
    g_v(z_v,\, \mathbf{z}_{-v}^g) \le \mathbf{0}_{m_v},
\]
where $\mathbf{z}_{-v}^\psi$ and $\mathbf{z}_{-v}^g$ denote the sets of decision variables from the remaining agents that influence $\psi_v(\cdot)$ and $g_v(\cdot)$, respectively. These dependencies induce the constraint graphs $\mathcal G_\psi = (\mathcal V, \mathcal E_\psi)$ and $\mathcal G_g = (\mathcal V, \mathcal E_g)$.

Formally, each agent $v$ solves the following optimization problem:
\begin{equation}\label{equ:GNEP_All}
    \begin{split}
        \begin{matrix*}[l]
            \min_{z_v} & f_v(z_v,\, \mathbf{z}_{-v}^f) \\
            \text{subject to} 
                & h_v(z_v) \le \mathbf{0}_{q_v}, \\
                & \psi_v(z_v,\, \mathbf{z}_{-v}^\psi) = \mathbf{0}_{l_v}, \\
                & g_v(z_v,\, \mathbf{z}_{-v}^g) \le \mathbf{0}_{m_v}. 
        \end{matrix*}
    \end{split}
\end{equation}

To simplify notation, we use $\mathbf{z}_{-v}$ to denote the collection of decision variables of all agents other than $v$ whenever the context is clear. That is, we write $f_v(z_v, \mathbf{z}_{-v})$ instead of $f_v(z_v, \mathbf{z}_{-v}^f)$, $\psi_v(z_v, \mathbf{z}_{-v})$ instead of $\psi_v(z_v, \mathbf{z}_{-v}^\psi)$, and $g_v(z_v, \mathbf{z}_{-v})$ instead of $g_v(z_v, \mathbf{z}_{-v}^g)$.

Furthermore, to shorten the notation, we sometimes use $\Omega_v(\mathbf{z}_{-v})$ to denote the feasible set of agent $v$ as a function of the decision variables of the other agents. For a fixed collection of decisions from the other agents, denoted $\bar{\mathbf{z}}_{-v}$, we define
\begin{equation*}
    \Omega_v(\bar{\mathbf{z}}_{-v})
    :=
    \left\{
        z_v \,\middle|\,
        h_v(z_v) \le \mathbf{0}_{q_v},\ 
        \psi_v(z_v,\, \bar{\mathbf{z}}_{-v}) = \mathbf{0}_{l_v},\ 
        g_v(z_v,\, \bar{\mathbf{z}}_{-v}) \le \mathbf{0}_{m_v}
    \right\}.
\end{equation*}

The solution concept for the GNEP in \eqref{equ:GNEP_All} is called a \emph{Generalized Nash Equilibrium} (GNE).
\begin{definition}[Generalized Nash Equilibrium (GNE)]
    A point $\mathbf{z}^\star$ is a Generalized Nash Equilibrium if for every agent $v$,
    \begin{equation}
        f_v(z_v^\star,\, \mathbf{z}_{-v}^\star) 
        \le 
        f_v(z_v,\, \mathbf{z}_{-v}^\star)
        \qquad 
        \forall\, z_v, \, z_v^\star  \in \Omega_v(\z_{-v}^\star).
    \end{equation}
\end{definition}
For comparison, in a classical Nash Equilibrium (NE), each agent solves
\begin{equation}
    \min_{z_v \in \Omega_v} \; f_v(z_v,\, \z_{-v}),
\end{equation}
where $\Omega_v$ is a fixed feasible set independent of other agents. A point $\z^\star$ is a Nash Equilibrium if for every agent $v$,
\begin{equation}
    f_v(z_v^\star,\, \z_{-v}^\star) 
    \le 
    f_v(z_v,\, \z_{-v}^\star)
    \qquad 
    \forall\, z_v \in \Omega_v.
\end{equation}

The key difference is that in a GNEP the feasible set $\Omega_v(\z_{-v})$ depends on the decisions of other agents through shared constraints, whereas in a classical NE the feasible sets $\Omega_v$ are independent across agents. As a result, in a GNEP, unilateral deviations must remain feasible with respect to shared constraints, which introduces coupling not present in the classical NE formulation.

\subsection{QVI, VI, and the GNEP}
First, define the feasible set mapping
\begin{equation}\label{equ:qviset}
    \Omega(\z)
    :=
    \prod_{v=1}^N \Omega_v(\z_{-v}),
\end{equation}
which is a point-to-set map that assigns to each $\z \in \mathbb{R}^n$ the Cartesian product of the individual feasible sets, where each $\Omega_v(\z_{-v})$ depends on the decisions of the other agents.

Assuming that the cost function of each agent is continuous and convex in its own decision variable, the first-order optimality conditions for all agents can be written in a compact form. By stacking these conditions, solving the GNEP is equivalent to solving the quasi-variational inequality problem $QVI(\Omega, F)$:
\begin{equation}
    \begin{split}
        &\text{Find } \z^\star \in \Omega(\z^\star) \text{ such that} \\
        &(\z - \z^\star)^\top F(\z^\star) \ge 0,
        \qquad 
        \forall \, \z \in \Omega(\z^\star),
    \end{split}
\end{equation}
where the operator $F:\mathbb{R}^n \to \mathbb{R}^n$ is defined as the pseudo-gradient:
\begin{equation}\label{equ:pseudo_grad}
    F(\z) 
    :=
    \begin{bmatrix}
        \nabla_{z_1} f_1(z_1,\, \z_{-1}) \\
        \nabla_{z_2} f_2(z_2,\, \z_{-2}) \\
        \vdots \\
        \nabla_{z_N} f_N(z_N,\, \z_{-N})
    \end{bmatrix}.
\end{equation}
When the dependence on other agents' decision variables can be removed, that is, when
\begin{equation}\label{equ:viset}
    \Omega(\mathbf{z}) = \Omega \subseteq \R^n
\end{equation}
is a fixed closed and convex set, the QVI reduces to the classical variational inequality problem $VI(\Omega, F)$:
\begin{equation}\label{equ:VI}
    \begin{split}
        &\text{Find } \mathbf{z}^\star\in\Omega \text{ such that} \\
        &(\mathbf{z} - \mathbf{z}^\star)^\top F(\mathbf{z}^\star) \ge 0,
        \qquad \forall\, \mathbf{z} \in \Omega.
    \end{split}
\end{equation}

In \cite{harker-1991}, the author provided the first reduction from the QVI formulation to the VI formulation. The key result is given below.
\begin{theorem}[\cite{harker-1991}, Corollary 3.1]
    Let $\Omega$ and $\Omega(\mathbf{z})$ be defined as in \eqref{equ:viset} and \eqref{equ:qviset}, respectively, and suppose the following condition holds:
    \begin{equation}\label{equ:restric}
        \Omega(\z) \subseteq \Omega,
        \qquad
        \forall\, \z \in \Omega.
    \end{equation}
    Then every solution of the VI problem \eqref{equ:VI} is a GNE of the original GNEP. However, the converse does not hold in general.
\end{theorem}

\begin{framed}
\begin{example}[Harker's Game \cite{harker-1991}]\label{exa:harker}
Consider the two-player GNEP
\begin{equation*}
    \begin{matrix*}[l]
        \min_{z_1} & f_1(z_1,\, \z_{-1})
        &
        \min_{z_2} & f_2(z_2,\, \z_{-2}) \\[0.2cm]
        \text{subject to} 
        & h_1(z_1) \le \mathbf{0},
        &
        \text{subject to} 
        & h_2(z_2) \le \mathbf{0} \\[0.1cm]
        &
        g_1(z_1,\, \mathbf{z}_{-1}) \le 0,
        &&
        g_2(z_2,\, \mathbf{z}_{-2}) \le 0,
    \end{matrix*}
\end{equation*}
where
\begin{equation*}
\begin{aligned}
    f_1(z_1,\, \z_{-1}) &= z_1^2 + \tfrac{8}{3} \, z_1 \, z_2 - 34 \, z_1, \\
    f_2(z_2,\, \z_{-2}) &= z_2^2 + \tfrac{5}{4} \, z_2 \, z_1 - 24.25\, z_2, \\[0.15cm]
    h_1(z_1) &= \begin{bmatrix} z_1 - 10 \\ -z_1 \end{bmatrix}, 
    \qquad
    h_2(z_2) = \begin{bmatrix} z_2 - 10 \\ -z_2 \end{bmatrix}, \\[0.15cm]
    g_1(z_1,\, \z_{-1}) &= g_2(z_2,\, \z_{-2}) = z_1 + z_2 - 15.
\end{aligned}
\end{equation*}

The operator $F$ is given by
\begin{equation*}
    F(\mathbf{z}) 
    = 
    \begin{bmatrix}
        \nabla_{z_1} f_1(z_1,\, \z_{-1}) \\
        \nabla_{z_2} f_2(z_2,\, \z_{-2})
    \end{bmatrix}
    =
    \begin{bmatrix}
        2 \, z_1 + \tfrac{8}{3} \, z_2 - 34 \\
        \tfrac{5}{4} \, z_1 + 2 \, z_2 - 24.25
    \end{bmatrix}.
\end{equation*}

The individual feasible sets are
\begin{equation*}
\begin{aligned}
    \Omega_1(\bar{z}_2)
        &= 
        \{ z_1 \mid z_1 \le -\bar{z}_2 + 15,\; 0 \le z_1 \le 10 \}, \\
    \Omega_2(\bar{z}_1)
        &= 
        \{ z_2 \mid z_2 \le -\bar{z}_1 + 15,\; 0 \le z_2 \le 10 \}.
\end{aligned}
\end{equation*}

\noindent\textbf{A QVI solution.}
Consider the point $\mathbf{z}^\star = (9,\, 6)$.  
We compute
\begin{equation*}
\begin{aligned}
    \Omega_1(6) 
        &= \{ z_1 \mid 0 \le z_1 \le 9 \}, \\
    \Omega_2(9) 
        &= \{ z_2 \mid 0 \le z_2 \le 6 \}.
\end{aligned}
\end{equation*}
Hence the Cartesian product is
\begin{equation*}
    \Omega(\mathbf{z}^\star)
    =
    \{ (z_1, z_2) \mid 0 \le z_1 \le 9,\; 0 \le z_2 \le 6 \}.
\end{equation*}

One can verify directly that
\begin{equation*}
    (\mathbf{z} - \mathbf{z}^\star)^\top F(\mathbf{z}^\star) \ge 0,
    \qquad
    \forall\, \mathbf{z} \in \Omega(\mathbf{z}^\star),
\end{equation*}
and therefore $(9,\, 6)$ is a solution of the associated QVI, and thus a GNE of the original GNEP.

\noindent\textbf{A variational-GNE (v-GNE).}
Consider instead the fixed feasible set
\begin{equation*}
    \Omega 
    =
    \{ (z_1, z_2) \mid z_1 + z_2 \le 15,\; 0 \le z_1 \le 10,\; 0 \le z_2 \le 10 \}.
\end{equation*}
The point $\mathbf{z}^\star = (5,\, 9)$ satisfies
\begin{equation*}
    (\mathbf{z} - \mathbf{z}^\star)^\top F(\mathbf{z}^\star) \ge 0,
    \qquad
    \forall\, \mathbf{z} \in \Omega,
\end{equation*}
and is therefore the solution of $VI(\Omega, F)$. This point is a variational-GNE, which differs from the QVI solution.

Figure~\ref{fig:harker_allsol} illustrates the full solution set of GNEs for Harker's Game. The red dot represents the v-GNE at $\z = (5, 9)$, and the light green region denotes the feasible set $\Omega$ for the corresponding VI problem. The red line segment from $\z = (9, 6)$ to $\z = (10, 5)$ shows the remaining GNEs. The dark green region represents the feasible set $\Omega(\z)$ for the QVI evaluated at $\z = (9, 6)$. The figure illustrates that the GNEP admits multiple solutions, even though the v-GNE is unique in this instance. It also highlights that the solution set of the GNEP may consist of several disconnected components, which presents challenges for algorithm design.
\end{example}
\end{framed}
\begin{figure}
    \centering
    \includegraphics[width=0.8\linewidth]{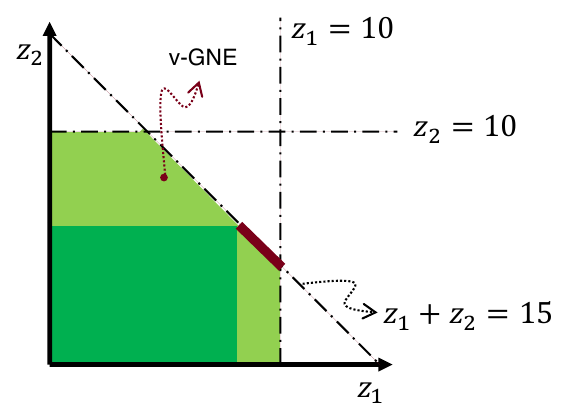}
    \caption{The solution set of Harker's game.}
    \label{fig:harker_allsol}
\end{figure}

Finally, the uniqueness of the VI solution is often guaranteed through the assumption of strong monotonicity, which is frequently used to establish the existence of a GNE.

\begin{definition}[Strongly Monotone Game]
    We say that the GNEP in \eqref{equ:GNEP_All} is strongly monotone if there exists a constant $\delta > 0$ such that for all $\mathbf{z}, \mathbf{z}' \in \mathbb{R}^n$,
    \begin{equation}
        (\mathbf{z} - \mathbf{z}')^\top 
        \big( F(\mathbf{z}) - F(\mathbf{z}') \big)
        \ge 
        \delta \, \|\mathbf{z} - \mathbf{z}'\|^2.
    \end{equation}
    Moreover, if $F$ is continuously differentiable, this condition is equivalent to
    \begin{equation}
        y^\top \nabla F(\mathbf{z})\, y 
        \ge 
        \delta \, \|y\|^2,
        \qquad
        \forall\, \mathbf{z} \in \mathbb{R}^n,\; y \in \mathbb{R}^n,
    \end{equation}
    where $\nabla F(\mathbf{z})$ denotes the Jacobian of $F$.
\end{definition}

Strong monotonicity ensures the existence of a unique solution to the variational inequality $VI(\Omega, F)$, which in turn guarantees the existence of a solution to the GNEP \cite{glynn-robinson-facchinei-2004}. Importantly, this holds despite the potential \textbf{non-uniqueness of the GNE itself}.  

The strong monotonicity assumption is widely used in existing distributed GNE algorithms and applies to many engineering applications that require distributed computation \cite{zhou-generalized-2005, maiorano-dynamics-2000, yi-pavel-2019, cenedese-asynchronous-2021, huang-distributed-2021, bianchi-continuous-time-2021}. In classical Nash Equilibrium Problems, strong monotonicity guarantees a unique equilibrium. In GNEPs, however, it only guarantees a unique v-GNE (variational-GNE), while general GNEs may still be non-unique. This is why most existing distributed algorithms enforce consensus among multipliers so that the solution lies within the v-GNE class.

\begin{theorem}[\cite{glynn_robinson_facchinei_2004}, Theorem. 2.3.3]
    If $VI(\Omega, F)$ is strongly monotone, then there exists a unique solution $\z^\star$.
\end{theorem}
\subsection{KKT System and GNEs}
Although the above discussion establishes the connections between the QVI and VI formulations and the original GNEP, these relationships are not directly useful for algorithm design, since QVI algorithms remain relatively immature and lack practical distributed implementations. On the other hand, when suitable constraint qualifications are satisfied for all agents, we can derive the KKT conditions for each individual optimization problem and stack them together. This KKT system provides a more tractable framework for algorithm development and serves as the basis for many existing approaches.

\begin{assumption}[KKT System]\label{ass:KKT_system}
    Let $\mathbf{z}_{-v}^\star$ be given, and suppose that $z_v^\star$ is a solution of the GNEP \eqref{equ:GNEP_All}. Then, for each agent $v$, there exist multipliers 
        \begin{equation*}
        \lambda_{v,e}^\star \in \R^{l_v}, \qquad
        \lambda_{v,i}^\star \in \R_{\ge0}^{m_v}, \qquad
        \mu_v^\star \in \R_{\ge0}^{q_v},
        \end{equation*}
    such that
\begin{equation}\label{equ:kkt} 
\begin{split} 
&L_v(\z, \lambda_{v, e}, \lambda_{v, i}, \mu_v): = f_v(z_v, \z_{-v}) + \psi_{v}^\top(z_v, \z_{-v})\,\lambda_{v, e} + g_{v}^\top(z_v, \z_{-v})\,\lambda_{v, i} + h_v^\top(z_v) \mu_v\\ 
&\nabla_{z_v} L_v(\z^\star, \lambda_{v, e}^\star, \lambda_{v, i}^\star, \mu_v^\star) = \mathbf{0}_{n_v} \\ 
& \psi_{v}(z_v^\star,\, \z_{-v}^\star) = \mathbf{0}_{l_v}\\ 
&\mathbf{0}_{m_v} \le \lambda_{v, i}^\star \, \bot \, g_{v}
(z_v^\star, \z_{-v}^\star) \le \mathbf{0}_{m_v}\\ 
&\mathbf{0}_{q_v} \le \mu_v^\star \, \bot \, h_{v}( z_v^\star) \le \mathbf{0}_{q_v} 
\end{split}. 
\end{equation}
For two vectors $a,b \in \mathbb{R}^n$, the compact complementarity notation
\begin{equation*}
    \mathbf{0}_n \le a \,\bot\, b \le \mathbf{0}_n
\end{equation*}
simultaneously expresses primal feasibility, dual feasibility, and complementary slackness:  
(i) $a \ge \mathbf{0}_n$, (ii) $b \le \mathbf{0}_n$, and (iii) $a_i b_i = 0$ for all $i$.
\end{assumption}

Assumption \ref{ass:KKT_system} implies that a suitable constraint qualification holds for each agent. Under the convexity assumptions  stated in Assumption~\ref{ass:standard_Eq}, Section~4.2 of \cite{facchinei-generalized-2010} shows that $\mathbf{z}^\star$ is a GNE if and only if it satisfies the KKT system \eqref{equ:kkt_eq}. In addition, we assume the Linear Independence Constraint Qualification (LICQ).

\begin{theorem}[\cite{facchinei_generalized_2007}, Theorem 3.1]\label{thm:consensus}
    Consider the VI problem $VI(K,F)$ defined in \eqref{equ:VI}.  
    Suppose that:
    \begin{enumerate}[(a)]
        \item for every agent $v$, the function $f_v(\cdot)$ is continuously differentiable in $\mathbf{z}$;
        \item for every agent $v$, the function $f_v(z_v, \mathbf{z}_{-v})$ is pseudo-convex in $z_v$.
    \end{enumerate}
    Let $\mathbf{z}^\star$ be a KKT point of the GNEP \eqref{equ:GNEP_All}.  
    Assume that for each shared constraint index $i$, all agents involved in that constraint share the same multiplier value, that is,
    \[
        \lambda_{1,i}^\star = \lambda_{2,i}^\star = \cdots = \lambda_{N,i}^\star = \bar{\lambda}_i.
    \]
    Then $\mathbf{z}^\star$ is a solution of the VI problem $VI(K,F)$. In particular, $\mathbf{z}^\star$ is a variational-GNE.
\end{theorem}
The above theorem states that if a v-GNE exists, then enforcing consensus among the Lagrangian multipliers corresponding to each shared constraint ensures that the resulting KKT point coincides with the v-GNE. This establishes a fundamental connection between consensus on Lagrangian multipliers and v-GNEs, and it has become a cornerstone of modern distributed algorithm design. Building on this principle, existing distributed methods \cite{yi-pavel-2019, huang-distributed-2021, bianchi-continuous-time-2021, cenedese-asynchronous-2021} explicitly enforce multiplier consensus in order to compute a v-GNE.

However, convergence to the unique v-GNE typically requires introducing auxiliary multiplier variables and enforcing consensus after their dynamics stabilize. This increases algorithmic complexity and often demands more information exchange among agents than is intrinsically necessary, which may limit scalability in large or communication-constrained systems.

\begin{framed}
\begin{example}[Harker's Game \cite{harker-1991}]\label{exa:harker_KKT}
Continue the Harker's Game example from Example~\ref{exa:harker}.  
The two-player GNEP is
\begin{equation*}
    \begin{matrix*}[l]
        \min_{z_1} & f_1(z_1,\, z_2)
        &
        \min_{z_2} & f_2(z_2,\, z_1) \\[0.2cm]
        \text{subject to} 
        & h_1(z_1) \le \mathbf{0},
        &
        \text{subject to} 
        & h_2(z_2) \le \mathbf{0} \\[0.1cm]
        &
        g_1(z_1,\, z_2) \le 0,
        &&
        g_2(z_2,\, z_1) \le 0.
    \end{matrix*}
\end{equation*}

The Lagrangians for the two agents are
\begin{equation}\label{eq:harker_Lag}
\begin{split} 
L_1(\z, \, \lambda_{1}, \, \mu_1) = &z_1^2 + \tfrac{8}{3} \, z_1 \, z_2 - 34 \, z_1 + \lambda_1 \, (z_1 + z_2 - 15) \\ 
&+ \mu_1^1\, (z_1 - 10)+ \mu_{1}^2 \, (-z_1)\\ L_1(\z, \, \lambda_{2}, \, \mu_2) = &z_2^2 + \tfrac{5}{4} \, z_2 \, z_1 - 24.25\, z_2 + \lambda_2 \, (z_1 + z_2 - 15) \\ 
&+ \mu_2^1\, (z_2 - 10)+ \mu_{2}^2 \, (-z_2)\\ 
\end{split} 
\end{equation}

The KKT conditions for the GNEP are
\begin{equation*}
\begin{split}
    \text{(Stationarity)} \quad
    &2 z_1 + \tfrac{5}{4} z_2 - 34 + \lambda_1 + \mu_1^1 - \mu_1^2 = 0, \\
    &2 z_2 + \tfrac{8}{3} z_1 - 24.25 + \lambda_2 + \mu_2^1 - \mu_2^2 = 0, \\[0.2cm]
    \text{(Primal feasibility)} \quad
    &z_1 + z_2 - 15 \le 0,\quad 0 \le z_1 \le 10,\quad 0 \le z_2 \le 10, \\[0.2cm]
    \text{(Dual feasibility)} \quad
    &\lambda_1 \ge 0,\; \lambda_2 \ge 0, \qquad \mu_1^1,\mu_1^2,\mu_2^1,\mu_2^2 \ge 0, \\[0.2cm]
\end{split}
\end{equation*}
\begin{equation*}
    \begin{split}
    \text{(Complementary slackness)} \quad
    &\lambda_1 (z_1 + z_2 - 15) = 0,\quad 
     \lambda_2 (z_1 + z_2 - 15) = 0, \\
    &\mu_1^1 (z_1 - 10) = 0,\quad \mu_1^2 (-z_1) = 0, \\
    &\mu_2^1 (z_2 - 10) = 0,\quad \mu_2^2 (-z_2) = 0. 
    \end{split}
\end{equation*}
\noindent\textbf{A variational-GNE (v-GNE).}  
If we enforce consensus of multipliers and set  
\begin{equation*}
    \lambda_1 = \lambda_2 = 0,
\end{equation*}
then solving it yields the unique v-GNE:
\begin{equation*}
    (z_1^\star,\, z_2^\star) = (5,\, 9).
\end{equation*}
\end{example}
\end{framed}

\section{Main Existing Algorithms}
In the preceding sections, I introduced the GNEP and its fundamental properties, including its connections to the QVI and VI formulations, as well as the corresponding KKT systems. However, these structural properties do not directly imply practical algorithms for computing GNEs. In this section, we present several seminal classical algorithms developed for both centralized and distributed computing frameworks.

It is important to emphasize that the key distinction between centralized and distributed computing lies in the level of information sharing. In a centralized setting, it is typically assumed that a global coordinator has access to all problem data, including the local cost functions and constraints of each agent, and computes a solution to the overall problem using this global information. In contrast, such an assumption is often unrealistic in practice: local objectives and constraints are frequently private, and agents may be unwilling or unable to share them. This motivates the need for distributed approaches, where agents compute their decisions using only local information and limited communication with others, without relying on global access to the full problem data.

For this reason, distributed computing algorithms have been developed, in which agents compute their decisions using only local information and limited communication with neighboring agents. To the best of our knowledge, regardless of the specific distributed architecture, all existing distributed GNE algorithms focus exclusively on computing v-GNEs, and they accomplish this by enforcing consensus among the Lagrangian multipliers associated with the shared constraints. 

The goal of our work is to investigate whether it is possible to design a distributed algorithm that does not rely on multiplier consensus, especially given that a v-GNE is only one particular solution of the GNEP and does not represent the entire equilibrium set.

\subsection{Central Computing}
In centralized computing, there are two main categories of algorithmic approaches that can potentially converge to a Generalized Nash Equilibrium without restricting the solution to be a v-GNE. In other words, these methods can converge to any solution of the QVI reformulation and are not limited to solutions of the corresponding VI.

\subsubsection{Augmented Penalty Methods}
These algorithms aim to simplify the GNEP by reducing it to a sequence of standard Nash Equilibrium Problems (NEPs) via a penalty mechanism. At iteration $k$, the algorithm introduces a vector of penalty parameters
\begin{equation*}
    \rho(k) := \big(\rho_v(k)\big)_{v \in \mathcal{V}},
\end{equation*}
where each component $\rho_v(k) > 0$ is associated with agent $v$. The vector $\rho(k)$ is treated as fixed when solving the inner problem.

Given $\rho(k)$, the algorithm solves an NEP in which each agent $v$ minimizes a penalized objective, yielding the updated decision $\mathbf{z}(k+1)$. The penalty vector is then updated component-wise based on the new iterate, resulting in $\rho(k+1)$.

In particular, for each agent $v$, the penalty-based approach requires solving the following optimization problem:
\begin{equation}
    \min_{z_v} P_v(\z, \rho_v(k)) := f_v(z_v,\, \z_{-v}) + \rho_v(k)\,\big\|\big[ g_v(z_v, \z_{-v}) \big]_{+}\big\|_\gamma,
\end{equation}
where $\rho_v(k)$ is the $v$-th component of the penalty vector at iteration $k$, and
\begin{equation*}
    \big[ g_v(z_v, \z_{-v}) \big]_{+} := \max\{0,\, g_v(z_v, \z_{-v})\}
\end{equation*}
is applied component-wise.

By appropriately updating the penalty vector $\rho(k)$ and repeatedly solving the induced NEPs, these algorithms can converge to a solution of the original GNEP under suitable assumptions on the problem structure.

However, although each inner problem is an NEP, the penalized objective $P_v(\cdot)$ is generally not differentiable due to the presence of the term $\big\|\big[ g_v(\cdot) \big]_{+}\big\|_\gamma$. Consequently, even though the reformulated NEP is structurally simpler than the original GNEP, the lack of differentiability may still introduce numerical challenges.

To address this non-differentiability issue, \cite{kanzow_augmented_2016} proposes Algorithm~\ref{alg:penalty}, which remains the state-of-the-art method in this category. While the objective functions in \eqref{equ:inner_p} are first-order continuous, the induced gradient mapping is not continuous. As a result, although one could in principle apply a zero-finding method on the gradient to solve the inner NEP, such an approach is not distributed, since evaluating the gradient requires global information.

As an extreme case, consider setting $u_{\max} = 0$. In this case, $u(k) = 0$ for all $k$. When the algorithm terminates at iteration $k = T$, it indicates that the KKT conditions of the problem in \eqref{equ:inner_p} (line 3 of Algorithm~\ref{alg:penalty}) are satisfied at $\mathbf{z}(T)$ with the corresponding value of $\rho(T)$. More precisely, the gradient of the cost function in \eqref{equ:inner_p} can be written as
\begin{equation}
    \nabla_{z_v} f_v(z_v, \, \z_{-v}) 
    + \left [ \rho_v(T) \cdot g_v(z_v(T), \z_{-v}(T)) \right]_{+} \, \nabla_{z_v} g_v(z_v, \, \z_{-v}),
\end{equation}
which shows that the gradient of the Lagrangian of the original GNEP is zero at iteration $T$, since line 4 sets
\begin{equation*}
    \lambda_v(T) = \left[ \rho_v(T) \cdot g_v(z_v(T), \, \z_{-v}(T)) \right]_{+}.
\end{equation*}
In addition, lines 6 and 7 specify that the update of $\rho_v$ stops once
\begin{equation*}
    \left\| \min \left \{ -g_v(\z(k+1)),\; \lambda_v(k+1) \right \} \right\|
\end{equation*}
is sufficiently small. If
\begin{equation*}
    \left\| \min \{ -g_v(\mathbf{z}(T)),\; \lambda_v(T) \} \right\| = 0,
\end{equation*}
then $\lambda_v(T) = 0$ when the constraint is inactive, and otherwise the constraint is active.

The overall feasibility of this method depends critically on the ability to solve the subproblems in line 3 of Algorithm~\ref{alg:penalty}, that is, the optimization problem in \eqref{equ:inner_p}. Unfortunately, there is still no established algorithmic framework for solving these inner problems reliably. In \cite{kanzow_augmented_2016}, the authors selected a Levenberg--Marquardt-type method as a heuristic, but its global convergence properties remain unclear. The approach is centralized because the inner NEP involves a discontinuous projection, and there is no distributed algorithm available to solve it.

\begin{algorithm}
\caption{Augmented Lagrangian Method for GNEPs \cite{kanzow_augmented_2016}}
\label{alg:penalty}
\begin{algorithmic}[1]
\STATE \textbf{Input:} $(\z^{0},\lambda^{0},u^{0},\rho^{0})$, $\tau \in (0,1)$, $u_{\max}\ge 0$, $\gamma>1$; set $k \gets 0$.
\WHILE{$(\z^{k},\lambda^{k})$ is not an approximate KKT point}
    \STATE Compute $\z^{k+1}$ with augmented objectives for all $v = 1, \dots,N$: 
    \begin{equation}\label{equ:inner_p}
    \min_{z_v}\;
    f_v(\z)
    +
    \frac{1}{2}
    \big\|
    \big[g(\z) + (\rho^{k})^{-1} u^{k}\big]_+
    \big\|_{\rho^{k}}^{2},
    \, \text{ where } \|x\|_{\rho^{k}}^{2} := x^\top \operatorname{diag}(\rho^{k}) \, x.
    \end{equation}
    \STATE $\displaystyle
    \lambda^{k+1}
    \gets
    \big[u^{k} + \operatorname{diag}(\rho^{k})\, g(\z^{k+1})\big]_+.
    $ \hfill (Update multipliers).
    \STATE $\rho^{k+1} \gets
    \begin{cases}
    \gamma\,\rho^{k}, &
    \text{if }
    \|\min\{-g(\z^{k+1}),\lambda^{k+1}\}\|
    >
    \tau\,\|\min\{-g(\z^{k}),\lambda^{k}\}\|,\\
    \rho^{k}, & \text{otherwise}.
    \end{cases}$
    \STATE $u^{k+1} \gets \min\{\lambda^{k+1},\,u_{\max}\}$ \hfill (elementwise).
    \STATE $k \gets k+1$.
\ENDWHILE
\end{algorithmic}
\end{algorithm}

\subsubsection{Parameterized Variational Inequality}
Similar to augmented penalty methods that reduce the GNEP to an NEP, parameterized VI algorithms aim to simplify the QVI by fixing certain parameters in advance and thereby converting the problem into a VI. A major difficulty for these methods is that the solution of the resulting parameterized VI is not guaranteed to be a GNE of the original problem. Consequently, the algorithms must test different parameter values until a GNE is found.

One approach is to treat the Lagrangian multiplier as a fixed parameter. For a given multiplier $\lambda_v$, each agent $v$ solves
\begin{equation}
\begin{matrix*}[l]
        \min_{z_v} & f_v(z_v, \, \z_{-v}) + g_v(z_v, \, \z_{-v})^\top \lambda_v \\
        \text{subject to} & z_v \in \Omega.
\end{matrix*}
\end{equation}
The algorithm then attempts to compute a solution of the VI problem $VI(F^\lambda, K)$, where
\begin{align*}
    F^\lambda(\mathbf{z}) = 
    \begin{bmatrix}
        \nabla_{z_1} f_1(z_1, \, \z_{-1}) + \nabla_{z_1} g_1(z_1, \, \z_{-1})^\top \lambda_1 \\
        \nabla_{z_2} f_2(z_2, \, \z_{-2}) + \nabla_{z_2} g_2(z_2, \, \z_{-2})^\top \lambda_2 \\
        \vdots \\
        \nabla_{z_N} f_N(z_N, \, \z_{-N}) + \nabla_{z_N} g_N(z_N, \, \z_{-N})^\top \lambda_N \\
    \end{bmatrix}.
\end{align*}
This parametrization technique was first proposed by \cite{nabetani_parametrized_2011}, known as price-directed parametrization, and was later extended in \cite{migot_parametrized_2020}.

Another approach is to vary the initial feasible point $\mathbf{z}(0)$ and treat it as a fixed parameter. Algorithm~\ref{alg:PVI} presents a classical parameterized VI method. By exploring different initial conditions, the algorithm may identify a GNE, and this approach is not restricted to obtaining a v-GNE.

These methods depend heavily on advances in VI solvers. Furthermore, since a solution to a parameterized VI is not necessarily a GNE, the algorithm typically selects an initial condition at random and iteratively explores different starting points until it identifies a solution that satisfies the KKT conditions of the original GNEP.

\begin{algorithm}
\caption{Parameterized-VI Algorithm \cite{facchinei_computation_2011}}
\label{alg:PVI}
\begin{algorithmic}[1]
\STATE \textbf{Input:} choose any $\z^{0} \in \Omega$; set $k \gets 0$.
\WHILE{true}
    \STATE Find $\z^{k+1} \in K(\z^{k})$ such that
    $\displaystyle
    \begin{bmatrix}
        \nabla_{z_1} f_{1}(\z^{k+1}) \\
        \nabla_{z_2} f_{2}(\z^{k+1}) \\
        \vdots \\
        \nabla_{z_N} f_{N}(\z^{k+1})
    \end{bmatrix}^{\!\top}
    (y - \z^{k+1})
    \ge 0,
    \qquad \forall\, y \in \Omega(\z^{k}).
    $
    \IF{$\z^{k+1}$ is a solution of the GNEP}
        \STATE \textbf{Output:} $\z^{k+1}$; \textbf{stop}.
    \ELSE
        \STATE $k \gets k+1$.
    \ENDIF
\ENDWHILE
\end{algorithmic}
\end{algorithm}

\subsection{Distributed Computing}
In distributed computing, existing algorithmic strategies can generally be classified into two categories: discrete-time dynamics and continuous-time dynamics. Existing algorithms primarily focus on strongly monotone games with linear shared constraints (i.e., $A\, \z - \c \le 0$ or $A\, \z - \c = 0$). To the best of my knowledge, all existing distributed algorithms rely on enforcing consensus among the Lagrange multipliers, and their convergence guarantees typically require the uniqueness of the fixed point. As a result, these methods converge only to v-GNEs, meaning that the computed solution corresponds to the VI reformulation rather than a general solution of the QVI.

\subsubsection{Discrete-Time Dynamics}
In \cite{yi-pavel-2019}, Algorithm~\ref{alg:discre} was proposed for a network game. Line 3 of the algorithm shows that each $\mathbf{z}_v$ evolves according to a gradient descent step on a local Lagrangian function. As a result, $z_v$ reaches a stationary point when 
\begin{equation*}
        \nabla_{z_v} L_v(z_v, \lambda_v) = 0
\end{equation*}

The variable $\xi_v$ is introduced to facilitate satisfaction of the coupling constraints and to enforce agreement among local multipliers. From line 4, $\xi_v$ reaches a stationary state when 
\begin{equation*}
    \lambda_v = \lambda_u \quad \text{for all } (u, v) \in -v.
\end{equation*}

In line 5, $\lambda_v$ is updated through a combination of projected gradient ascent on the local Lagrangian and proportional–integral dynamics that compensate for consensus errors among the multipliers. The multiplier $\lambda_v$ reaches a stationary point when the shared constraint is active, or it is projected to zero when the constraint is inactive.

Consequently, when the algorithm converges, it converges to a KKT point of the GNEP with consensus among all $\lambda_v$. This corresponds to a v-GNE, since agreement of the multipliers ensures that the limiting point solves the VI reformulation rather than a general QVI.

To establish the convergence of the algorithm, the authors relied on existing convergence results from operator splitting theory. In particular, they showed that, with suitable choices of the step sizes $\tau_v$, $\nu_v$, and $\sigma_v$ for each agent $v$, the operator form induced by the algorithm can be written as a fixed-point iteration in which the associated operator is averaged. An averaged operator is a generalization of a contraction operator and still guarantees convergence of the corresponding fixed-point iteration. As a result, the algorithm converges to the desired fixed point.

\begin{algorithm} 
\begin{algorithmic}[1] 
\STATE {\bfseries Input:} Initialized $z_v(0)$, $\lambda_v(0)$, $\xi_v(0)$. $\tau_v$, $\nu_v$, $\sigma_v>0$ for player $v$. 
\WHILE{ Not Converge } 
\STATE $ z_v(k+1) \leftarrow P_{\Omega_v} \left \{ z_v(k) - \tau_v \left ( \nabla_{z_v} f_v (\z(k)) + A_v^\top \lambda_v(k) \right ) \right \}$ 
\STATE $ \xi_v(k+1) \leftarrow \xi_v(k) + \nu_v \sum_{u \in -v}{\left (\lambda_v(k) - \lambda_u(k) \right)} $ 
\STATE \begin{flalign*} \lambda_v(k+1) \leftarrow &P_{\R_{\ge 0}^{m_v}} \{ \lambda_v(k) - \sigma_v [-A_vz_v(k+1)+ c_v \\ &- A_v(z_v(k+1) - z_v(k))+ \sum_{u \in -v}{ (\lambda_v(k) - \lambda_u(k))} \\ &+ \sum_{u \in -v}{ 2(\xi_v(k+1) - \xi_u(k+1))}- \sum_{u \in -v }{(\xi_v(k) - \xi_u(k))}] \} 
\end{flalign*} 
\STATE $k \leftarrow k +1$. 
\ENDWHILE 
\end{algorithmic} 
\caption{Consensus-based algorithm for each agent $v$, \cite{yi-pavel-2019}} 
\label{alg:discre} 
\end{algorithm}

\subsubsection{Continuous-Time Dynamics}
A direct continuous-time counterpart of Algorithm \ref{alg:discre} is presented in Algorithm \ref{alg:continuous_pd}, following \cite{bianchi-continuous-time-2021}. This algorithm is obtained via a multi-integrator continuous-time formulation of the discrete scheme and aims to compute a v-GNE using augmented multiplier dynamics. The convergence proof is based on a standard Lyapunov argument. The Lyapunov function is straightforward to construct because the underlying operator is monotone.

\begin{algorithm}
\begin{algorithmic}[1]
\STATE {\bfseries Input:} Initialize $z_v(0)\in\Omega_v$, $\lambda_v(0)\in\R_{\ge0}^{m_v}$, $\xi_v(0)$.
\STATE $\dot z_v =
\Pi_{\Omega_v}\!\big[
z_v,\,
-\nabla_{z_v} f_v(\z) - A_v^\top \lambda_v
\big].$
\STATE $\dot \xi_v =
\sum_{u \in -v}\big(\lambda_v-\lambda_u\big).$
\STATE $\dot \lambda_v =
\Pi_{\mathbb{R}_{\ge 0}^{m_v}}\!\Big[
\lambda_v,\,
- A_v z_v + c_v
- A_v \dot z_v
+ \sum_{u \in-v}\big(\lambda_v-\lambda_u\big)
+ \sum_{u \in -v}
\big(2(\xi_v-\xi_u)-(\xi_v-\xi_u)\big)
\Big].$
\end{algorithmic}
\caption{Continuous-time primal--dual operator splitting dynamics (agent $v$)}
\label{alg:continuous_pd}
\end{algorithm}

Moreover, in \cite{lin_distributed_2022}, Algorithm~\ref{alg:continuous} proposes a continuous-time method based on singular perturbation dynamics, which achieves exponential convergence rather than asymptotic convergence. The authors restrict their analysis to cost functions of the form
\begin{equation*}
        f_v(z_v,\, \z) = f_v\big(z_v,\, \varrho(\z)\big), \qquad 
    \varrho(\mathbf{z}) = \sum_{v} \theta_v(z_v),
\end{equation*}
meaning that each player’s cost depends only on a linear aggregation of the other agents’ decisions. This structure allows each agent to estimate the aggregate term through local communication over the graph, instead of requiring direct access to all other agents’ decision variables. In addition, the authors consider only common shared equality constraints.

In the fast dynamics (lines 4–6 of Algorithm~\ref{alg:continuous}), lines 4 and 5 enforce consensus among the direct dual variables $\lambda_v$. At equilibrium, line 5 yields $\bar{\lambda}_u = \bar{\lambda}_v$ for all neighboring agents, and line 4 further ensures $\bar{\lambda}_v = \bar{\mu}_v$. Once consensus among the dual variables and agreement on the aggregated quantity are achieved, the distributed saddle-point–type dynamics in the slow subsystem (lines 2–3) guide the agents toward a GNE under the globally coupled equality constraints.

The convergence analysis relies on the standard stability theorem for singular perturbation systems. As a consequence, uniqueness of the equilibrium is required, and the algorithm converges only to a v-GNE.

\begin{algorithm}
\begin{algorithmic}[1]
\STATE {\bfseries Input:}  Initialized $z_v(0)$, $\lambda_v(0)$, $\mu_v(0)$, $\sum_{v=1}^N \nu_v(0) = 0$, $\sum_{v=1}^N \zeta_v(0) = 0$. $\epsilon \in (0, 1)$.
    \STATE $\dot{z}_v = - \nabla_{z_v}f_v(z_v, \varrho(\z))\rvert_{\varrho(\z) = \eta_v} - A_v^\top \lambda_v $
    \STATE $\dot{\mu}_v = A_v z_v - c_v $
    \STATE $\epsilon \, \dot{\lambda}_v = - \alpha (\lambda_v - \mu_v) - \beta \sum_{u=1}^N (A)_{vu}(\lambda_v - \lambda_u) - \nu_v$
    \STATE $\epsilon \, \dot{\nu}_v = \alpha \beta \sum_{u=1}^N (A)_{vu} (\lambda_v - \lambda_u)$
    \STATE $\epsilon \, \dot{\zeta}_v = \gamma \sum_{u = 1}^N (A)_{uv} \cdot \text{sign} (\eta_u - \eta_v)$
    \STATE $\eta = \zeta_v + \theta_v(z_v)$
\end{algorithmic}
\caption{Singular Perturbed Algorithm for each agent $v$, \cite{lin_distributed_2022}}
\label{alg:continuous}	
\end{algorithm}

\section{The Contribution of this Thesis}

As discussed earlier, since a GNE may not be unique, I do not intend to focus only on a special subclass of equilibria such as v-GNEs. However, computing a GNE algorithmically is challenging, even in centralized settings. Under restricted subclasses of the problem, both augmented penalty methods and parameterized variational inequality approaches can be used to compute approximate GNEs that are not limited to v-GNEs, but only under additional restrictive assumptions.

First, these methods typically require a central authority that has access to all agents’ private information, which is impractical due to privacy concerns.
Second, the resulting inner problems usually still require an additional solver or oracle, whose availability and implementation are often unclear.
Third, for parameterized variational inequality methods, it can only be shown that the set of GNEs is contained in the solution set of the parameterized variational inequality. As a result, non-GNE solutions may also be produced, and additional steps are required to identify and discard these undesired outcomes.

On the other hand, existing distributed GNE solvers, including both discrete-time and continuous-time methods, typically restrict their attention to v-GNEs. This restriction is mainly technical: the variational inequality associated with v-GNEs is usually monotone, which allows existing methods to rely on well-established analytical tools. In addition, these approaches often require augmented multipliers to enforce consensus among dual variables. This requirement is also unrealistic in practice, especially when compared with distributed optimization methods, where additional auxiliary variables are usually not needed.

Accordingly, the contribution of this thesis is summarized as follows:
\begin{enumerate}
    \item I study general GNEs, rather than restricting attention to v-GNEs, in a fully distributed setting.
    \item Our distributed framework preserves each agent’s privacy. Moreover, by not restricting the solution to v-GNEs, the proposed method does not require sharing augmented multipliers, which significantly reduces the information exchange per iteration.
    \item I develop a novel proof technique that can handle the potentially non-monotone structure of the underlying problem, which is typically regarded as a challenging case in this area.
\end{enumerate}


\chapter{Fully Distributed Algorithms Based on Continuous-Time Dynamics}
\label{continuous_chapter}

In this section, I present our main contribution: a continuous-time algorithm that converges to GNEs without enforcing consensus on the multipliers associated with the shared constraints. I begin by specifying the multi-agent setup and introducing the projected dynamics framework, which serves as our primary analytical tool. I then establish convergence for games with shared linear equality constraints and individual constraints under general convex and monotone settings. After that, I prove convergence for games with shared inequality constraints in the case of quadratic convex monotone costs.

From a theoretical perspective, several challenges arise in the convergence analysis. First, the proposed dynamics exhibit switching behavior with discontinuous right-hand sides, which precludes the use of standard stability arguments. Second, the non-uniqueness of equilibria leads to difficulties in characterizing algorithmic convergence. To address these issues, I introduce a transformation of the shared multipliers that separates each multiplier into a common base component and a difference component. This transformation makes the equilibrium unique for any fixed initialization.

For shared equality constraints, the transformed dynamics allow us to determine not only the convergence of the difference vector but also its exact limiting value, even though agents do not have direct access to this quantity. In contrast, for shared inequality constraints, the difference vector still converges, but its limiting value cannot be characterized explicitly, which leads to a more subtle analysis. In addition, since the induced primal–dual dynamics are not necessarily monotone even when the game is strongly monotone, I develop a proof based on bounded-input bounded-output and bounded trajectory analysis to establish convergence.

\section{Preliminaries}
\subsection{Notation}

We denote $\R^n$ as the space of $n$-dimensional Euclidean vectors. The sets $\R_{\ge 0}$ and $\R_{>0}$ represent the nonnegative and positive real numbers, respectively. The notation $[m]$ denotes the index set $\{1, 2, \dots, m\}$. Bold symbols $\mathbf{1}$ and $\mathbf{0}$ denote all-one and all-zero vectors of appropriate dimensions, respectively.  

For a vector $a \in \R^m$, $(a)_i$ denotes its $i$-th component, and $a^\top$ its transpose. For two vectors $a, b \in \R^m$, the inequality $a \le b$ is understood component-wise, and $a^\top b$ denotes their inner product. For a matrix $A \in \R^{p \times q}$, $(A)_{ij}$ denotes the $(i,j)$-th entry, and $A^\top$ its transpose.  

Following \cite{cherukuri-asymptotic-2016, ebrahimi-robust-2019}, we define the projection operator:
\begin{equation}\label{equ:proj_oper}
    [a]^+_b :=
    \begin{cases}
        a, & \text{if } b > 0, \\[2pt]
        \max \{ 0,\, a \}, & \text{if } b = 0,
    \end{cases}
\end{equation}
for any $a\in \R$ and $b \in \R_{\ge0}$. For $a \in \R^n$ and $b \in \R_{\ge0}$, $[a]^+_b$ denotes the vector whose $i$-th component is $[a_i]^+_{b_i}$.  

The sublevel set of a function $V: \R^n \to \R$ with parameter $\delta > 0$ is defined as
\begin{equation*}
    V^{-1}(\le \delta)
    := \{\z \in \R^n \mid V(\z) \le \delta \}.
\end{equation*} 

Finally, for two vectors $a, b \in \R^n$, the compact complementarity notation
\begin{equation*}
    \mathbf{0}_n \le a \, \bot \, b \le \mathbf{0}_n
\end{equation*}
simultaneously expresses primal feasibility, dual feasibility, and complementary slackness, which explicitly means (1) $a \ge \mathbf{0}_n$, (2) $b \le \mathbf{0}_n$, (3) $(a)_i \cdot (b)_i = 0, \quad \forall \, i \in [n]$.

\subsection{Generalized Nash Equilibrium Problems (GNEPs)}

Consider the GNEP in \eqref{equ:GNEP_All} with $N$ agents. In a distributed setup, each agent $v$ controls a local decision variable $z_v \in \mathbb{R}^{n_v}$ and has a private cost function $f_v(z_v, \mathbf{z}_{-v}^f)$ that depends on both its own decision and those of the agents that influence it. Here, $\mathbf{z}_{-v}^f$ denotes the subset of decision variables that affect $f_v(\cdot)$, inducing a cost-dependency graph $\mathcal{G}_f = (\mathcal V, \mathcal E_f)$. Each agent also has private individual constraints $h_v(z_v) \le \mathbf{0}_{q_v}$. Let $\mathbf{z} = [z_1^\top,\dots,z_N^\top]^\top \in \mathbb{R}^n$ denote the joint decision vector, where $n = \sum_{v=1}^N n_v$ and $q = \sum_{v=1}^N q_v$.

In addition to their private constraints, agents may be coupled through shared constraints. Specifically, agent $v$ is subject to equality and inequality constraints of the form
\begin{equation*}
    \psi_v(z_v, \mathbf{z}_{-v}^\psi) = \mathbf{0}_{l_v}, \qquad
    g_v(z_v, \mathbf{z}_{-v}^g) \le \mathbf{0}_{m_v},
\end{equation*}
where $\mathbf{z}_{-v}^\psi$ and $\mathbf{z}_{-v}^g$ denote the decision variables of other agents that influence these constraints. These dependencies give rise to the constraint graphs $\mathcal G_\psi = (\mathcal V, \mathcal E_\psi)$ and $\mathcal G_g = (\mathcal V, \mathcal E_g)$, where each edge indicates that the corresponding agents are jointly involved in a shared constraint. Note that while $\psi_v$ and $g_v$ represent shared constraints, they typically involve only a subset of agents, and therefore are not globally known to all agents.

In distributed optimization and game-theoretic settings, agents are typically assumed to keep the following information private:
\begin{itemize}
    \item their local cost function $f_v(\cdot)$ and its gradient structure,
    \item their individual constraints $h_v(\cdot)$,
    \item their local copies of dual variables associated with shared constraints.
\end{itemize}
Agents are generally unwilling or unable to share this information because it may contain sensitive parameters, proprietary cost models, or private operational requirements. Consequently, an algorithm cannot assume global access to the full collection of functions $\{f_v, h_v, \psi_v, g_v\}_{v=1}^N$, nor to all dual variables or decision variables.

I focus on the setting in which agents are coupled through aggregate linear constraints:
\begin{align}
    \psi(\mathbf{z}) &= A \, \z - \c, \label{equ:linear_eq}\\
    g(\mathbf{z})    &= B \, \z - \d, \label{equ:linear_ineq}
\end{align}
where $A \in \mathbb{R}^{o \times n}$ and $B \in \mathbb{R}^{p \times n}$, and $\mathbf{c} \in \mathbb{R}^{o}$ and $\mathbf{d} \in \mathbb{R}^{p}$ represent the shared resource limits. Each agent has access only to the rows and columns of $A$ and $B$ associated with the constraints and decision variables in which it participates. Its local views are therefore
\begin{equation*}
    \begin{split}
    \psi_v(z_v, \mathbf{z}_{-v})
        &= A(l_v, n_v) z_v
           + \!\!\sum_{u \in \tilde{\mathcal{N}}_\psi(v)}\!\! A(l_v, n_u) z_u
           - \mathbf{c}(l_v),\\
    g_v(z_v, \mathbf{z}_{-v})
        &= B(m_v, n_v) z_v
           + \!\!\sum_{u \in \tilde{\mathcal{N}}_g(v)}\!\! B(m_v, n_u) z_u
           - \mathbf{d}(m_v),
   \end{split}
\end{equation*}
where $A(l_v, n_v)$ and $B(m_v, n_v)$ denote the corresponding submatrices, and $\tilde{\mathcal{N}}(v)$ represents the neighboring agents of agent $v$ in the connectivity graph.

This structure yields a GNEP with feasible set
\begin{equation*}
    \Omega
    =
    \left\{
        \mathbf{z}
        \,\middle|\,
        \psi(\mathbf{z}) = \mathbf{0},\;
        g(\z) \le \mathbf{0},\;
        h_v(z_v) \le \mathbf{0}
        \quad \forall\, v
    \right\}.
\end{equation*}

Our algorithm is designed to compute the full set of GNEs rather than only the v-GNE. This allows application to a broader family of problems in which enforcing consensus among multipliers is not justified or not feasible. When centralized coordination or uniform pricing is appropriate, the proposed framework can still impose multiplier consensus to recover a v-GNE, but consensus is not required in general.

\subsection{Projected Dynamical System (PDS)}
In this section, I provide a brief overview of Projected Dynamical Systems and several existing convergence theorems that form the foundation of our analysis. I also examine a special class of Projected Dynamical Systems known as primal–dual dynamics, which have been widely used as continuous-time algorithms for solving a broad range of convex optimization problems.
\begin{definition}[Projected Dynamical System]\cite{debreu-social-1952}.
    Given $x \in K$ and $v \in \R^n$, where $K \subseteq \R^n$ is a closed convex set, define the projection of the vector $v$ at $x$ with respect to $K$ by  
    \begin{equation*}
        \Pi_{K}(x, v) := \lim_{\epsilon \to 0} \frac{P_{K}(x + \epsilon \, v) - x}{\epsilon},
    \end{equation*}
    where 
    \begin{equation*}
        P_{K}(x) := \arg \min_{y \in K} \| x - y\|.
    \end{equation*}
    
    Let $\hat{F}:\R^n \to \R^n$ be a given mapping. The PDS is ordinary differential equations as a form of 
    \begin{equation}\label{equ:PDS}
        \dot{x} = \Pi_{K}(x, -\hat{F}(x)).
    \end{equation}
\end{definition}
Figure~\ref{fig:pds_illu} provides an illustration of a Projected Dynamical System (PDS). Given a feasible set $K$, the PDS ensures that the trajectory never leaves $K$. When the state lies in the interior of $K$, the system follows the original vector field. However, once the trajectory reaches the boundary of $K$ and the vector field points outward, the projection forces the motion to evolve tangentially along the boundary, preventing violation of the feasibility constraints.
\begin{figure}
    \centering
    \includegraphics[width=0.5\linewidth]{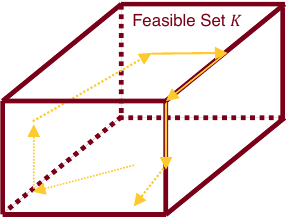}
    \caption{An illustration of the movement of a projected dynamical system over a feasible set $K$.}
    \label{fig:pds_illu}
\end{figure}

The next two propositions establish fundamental properties of Projected Dynamical Systems. The first concerns the uniqueness of trajectories and guarantees their existence for all time, ruling out Zeno behavior despite the switching that may occur at the boundary of the feasible set. The second provides sufficient conditions for the stability of equilibrium points.
\begin{proposition}[\cite{zhang_projected_1996} Theorem 2.5]\label{thm:uniq_tra}
    Let $-\hat{F}(x): \R^n \to \R^n$ be Lipschitz on a closed convex polyhedron $K \in \R^n$, then for any $x_0 \in K$, there exists a unique solution $t \to x(t)$ to the projected dynamic system with $x(0) = x_0$ defined over the domain $[0, \infty)$.
\end{proposition}

\begin{proposition}[\cite{zhang_stability_1995}, Theorem 4.2.]\label{thm:PDS_sta} Suppose that $x^\star$ is an equilibrium point of the PDS (\ref{equ:PDS}). If $\hat{F}(x)$ is globally monotone at $x^\star$, then $x^\star$ is globally stable.
\end{proposition}

Additionally, due to the discontinuous nature of projected dynamical systems \cite{bacciotti-nonpathological-2006, cherukuri-asymptotic-2016}, the Lie derivative of a continuously differentiable function $V: \R^n \to \R$ with respect to the dynamics~(\ref{equ:PDS}) at a point $x$ is defined as
\begin{equation}\label{equ:lie_der}
\V := (\nabla_x V)^\top \Pi_{K}(x, -\hat{F}(x)).
\end{equation}

This allows us to extend standard stability analysis and apply invariance principles to the $\omega$-limit set, with an appropriate extension to discontinuous dynamics that are discontinuous only on a set of measure zero.


\section{Linear Equality Shared Constraints}

I first consider a subclass of the GNEP~\eqref{equ:GNEP_All} in which agents share linear equality constraints in addition to their individual constraints. This subclass is particularly important because it includes a wide range of multi-agent coordination problems, such as distributed multi-robot placement, which can be naturally formulated as a GNEP with shared equality constraints. It also generalizes the classical consensus constraints widely studied in distributed optimization, thereby serving as a bridge between distributed optimization and distributed equilibrium computation. Formally, each agent~$v$ solves the following optimization problem:
\begin{equation}\label{equ:GNEP_eq}
    \begin{split}
        \begin{matrix*}[l]
            \min\limits_{z_v} & f_v(z_v,\, \z_{-v}^f) \\[2pt]
            \text{s.t.} & h_v(z_v) \le \mathbf{0}_{q_v}, \\[2pt]
            & \psi_v(z_v,\, \z_{-v}^\psi) = \mathbf{0}_{l_v},
        \end{matrix*}
    \end{split}
\end{equation}
where $f_v(\cdot)$ is the local cost function of agent~$v$, $h_v(\cdot)$ represents individual inequality constraints, and $\psi_v(\cdot)$ denotes the shared equality constraints that couple the agents’ decisions according to the global structure in \eqref{equ:linear_eq}, i.e.,
\begin{align*}
\psi_v(z_v,, \z_{-v}) = A(l_v, n_v)\, z_v + \sum_{u \in \tilde{\mathcal{N}}_{\psi}(v)} A(l_v, n_u)\, z_u - \c(l_v).
\end{align*}
\begin{assumption}\label{ass:stronglymonotone}
    The GNEP (\ref{equ:GNEP_All}) is strongly monotone.
\end{assumption}
In addition to the strong monotonicity assumption in Assumption~\ref{ass:stronglymonotone}, I impose the following standard conditions.
\begin{assumption}\label{ass:standard_Eq}
For every player $v$, $f_v(z_v, \z_{-v})$ and $h_v(z_v)$ are continuously differentiable, convex, and have locally Lipschitz continuous gradients with respect to $z_v$. Moreover, the stacked pseudo-gradient of all agents, $F$, defined in (\ref{equ:pseudo_grad}), is locally Lipschitz in $\z$, i.e., Lipschitz continuous on every compact set $K' \subseteq \mathbb{R}^n$.
\end{assumption}

\begin{assumption}\label{ass:KKT_Eq}
    Let $\z_{-v}^\star$be given, and suppose that $z_v^\star$ is a solution of the GNEP (\ref{equ:GNEP_eq}), then for all $v$, $\lambda_{v}^\star \in \R^{l_v}$, $\mu_v^\star \in \R^{q_v}$ exist such that
    \begin{equation}\label{equ:kkt_eq}
    \begin{split}
        &L_v(\z, \lambda_{v}, \mu_v): = f_v(z_v, \z_{-v}) + \psi_{v}^\top(z_v, \z_{-v}) \lambda_{v} + h_v^\top(z_v) \mu_v\\
        &\nabla_{z_v}  L_v(\z^\star, \lambda_{v}^\star, \mu_v^\star) = \mathbf{0}_{n_v} \\
        & \psi_{v}(z_v^\star,\, \z_{-v}^\star) = \mathbf{0}_{l_v}\\
        &\mathbf{0}_{q_v} \le \mu_v^\star \, \bot \, h_{v}( z_v^\star) \le \mathbf{0}_{q_v}
    \end{split}.
    \end{equation}
\end{assumption}
Assumption~\ref{ass:KKT_Eq} ensures that an appropriate \textbf{constraint qualification} holds for each agent~$v$. Together with the convexity assumption in Assumption~\ref{ass:standard_Eq}, stacking the KKT conditions of all agents and applying the corresponding constraint qualifications yields the result in Section~4.2 of~\cite{facchinei-generalized-2010}, which states that $\mathbf{z}^\star$ is a GNE if and only if it satisfies the KKT system in~\eqref{equ:kkt_eq}.

\subsection{A Fully Distributed Algorithm}

Based on the problem formulation in~\eqref{equ:GNEP_eq}, the distributed structure of the game can be represented using two interaction graphs. The first graph, $\mathcal G_f = (\mathcal V, \mathcal E_f)$, captures the coupling among agents through their cost functions, while the second graph, $\mathcal G_\psi = (\mathcal V, \mathcal E_\psi)$, captures the coupling induced by the shared linear constraints. Each agent has access only to its own local information, namely its cost function $f_v(\cdot)$, its individual constraints $h_v(\cdot)$, and its local component of the shared constraint function $\psi_v(\cdot)$. The goal is to design a distributed algorithm that guarantees convergence to a GNE while requiring only limited information exchange along these graphs, thereby preserving both privacy and scalability.

In the worst case, both graphs may be complete, meaning that $f_v(\cdot)$ and $\psi_v(\cdot)$ depend on the decision variables of all other agents. Even in this scenario, our method does not rely on any central coordinator or global multiplier sharing. Building on this distributed setup, our main continuous-time algorithm is presented in Algorithm~\ref{alg:equa}.

\begin{algorithm}
\begin{algorithmic}[1]
\STATE {\bfseries Input:}Initial $z_v(0)$, $\lambda_v(0)$, $\mu_v(0) \geq 0$
\STATE $\dot{z}_v = -\nabla_{z_v} L_v(\z, \lambda_v, \mu_v)$.
\STATE $\dot{\lambda}_{v} = \psi_v(z_v, \z_{-v})$.
\STATE $\dot{\mu}_v= [h_v(z_v)]_{\mu_v}^+$.
\end{algorithmic}
\caption{Continuous Time Distributed Primal-Dual Dynamics for each agent $v$}
\label{alg:equa}	
\end{algorithm}

This fully distributed method preserves the privacy of each agent’s cost function, multipliers, and constraints: agents exchange only their decision variables at each iteration. In particular, in Line~2 of Algorithm~\ref{alg:equa}, agents communicate decision variables over the cost-function graph $\mathcal G_f$, and in Line~3, they communicate over the shared-constraint graph $\mathcal G_\psi$. Consequently, the information exchanged at each iteration is restricted to the decision variables transmitted over $\mathcal G_f$ and $\mathcal G_\psi$.

This communication structure stands in sharp contrast to existing v-GNE algorithms (e.g., \cite{yi-pavel-2019, cenedese-asynchronous-2021, huang-distributed-2021, bianchi-continuous-time-2021}), which require agents to exchange not only their decision variables but also Lagrangian multipliers and auxiliary variables over $\mathcal G_\psi$ in order to enforce consensus. Such exchanges substantially increase communication overhead and reveal sensitive information.

To quantify the savings, consider the case where $\mathcal G_f = \mathcal G_\psi$ and each edge transmits a vector of dimension $d$. In our algorithm, each communication step requires exchanging only decision variables. In contrast, consensus-based v-GNE methods require additional messages for multipliers and auxiliary variables, resulting in approximately $2|\mathcal E_\psi|\, d$ extra transmissions per communication round. Thus, in this setting, our approach uses only \textbf{one third} of the communication cost required by consensus-based algorithms.

\subsection{Convergence Analysis}
I first establish the connection between Algorithm~\ref{alg:equa} and the projected dynamical system framework, which allows us to guarantee the existence of trajectories and apply Lyapunov and Lie-derivative–based arguments for convergence.

\begin{lemma}\label{lem:pds}
Algorithm~\ref{alg:equa} is an instance of the PDS~\eqref{equ:PDS}, where $\hat{F}$ is obtained by stacking the right-hand sides of the dynamics of all agents before applying any projection, and 
\begin{equation*}
    K \;=\; \big\{\, (\mathbf{z},\, \lambda_v,\, \mu_v) \;\big|\; 
    \mathbf{z} \in \mathbb{R}^n,\;
    \lambda_v \in \mathbb{R}^{l_v},\;
    \mu_v \ge \mathbf{0}_{q_v},
    \;\forall v
    \big\}.
\end{equation*}
\end{lemma}

Consider the set of multipliers $\lambda_v^i(t)$ $\forall v$ associated with the $i$-th shared constraint $\psi^i(\mathbf{z})$, which corresponds to the $i$-th component of the global constraint $\psi(\mathbf{z})$. I arbitrarily select one of these multipliers as a reference state and denote it by $\lambda^i(t)$. Thus, there are $o$ such reference states in total.

\subsection{Auxiliary States and Multiplier Deviations}
For every multiplier $\lambda_u^i(t)$ corresponding to the same shared constraint $\psi^i(\mathbf{z})$ but not chosen as the reference, I define its deviation from the reference as
\begin{equation}\label{equ:difference_vec}
    \Delta_{u}^i := \lambda_u^i - \lambda^i.
\end{equation}
The associated dynamics satisfy
\begin{equation}\label{equ:delta}
    \dot{\Delta}_{u}^i
    = \dot{\lambda}_u^i - \dot{\lambda}^i
    = \psi^i(\mathbf{z}) - \psi^i(\mathbf{z})
    = 0,
\end{equation}
which implies that each $\Delta_{u}^i(t)$ remains constant over time. If the multiplier of agent $u$ is chosen as the reference for constraint $i$, then $\Delta_u^i(t) = 0$. Finally, I denote by $\Delta_{[v]}(t)$ the stacked vector containing all deviations $\Delta_v^i(t)$ for $i \in l_v$, ordered consistently.
\begin{lemma}\label{lem:equipt}
    \textbf{(Existence, Uniqueness, and KKT Consistency).} 
    Consider the dynamical system described by Algorithm \ref{alg:equa} under the assumption that the pseudo-gradient mapping $F$ is strongly monotone. For any fixed set of initial auxiliary states $\{\Delta_{[v]}(0)\}_{v \in \mathcal{V}}$: 
    \begin{enumerate}
        \item Any equilibrium point $(\bar{\mathbf{z}}, \bar{\lambda}, \bar{\boldsymbol{\mu}})$ of the algorithm satisfies the KKT conditions \eqref{equ:kkt_eq} of the GNEP defined in \eqref{equ:GNEP_eq}.
        \item There exists a unique equilibrium point in the primal variables $\mathbf{z}^\star$.
    \end{enumerate}
\end{lemma}
\begin{remark}
    While the equilibrium is mathematically unique for fixed initial conditions, agents cannot compute this point \textit{a priori} as they lack global knowledge of the aggregate deviation vector $\Delta(0)$.
\end{remark}

\renewcommand{\proofname}{Proof of Lemma \ref{lem:equipt}}
\begin{proof}
    The proof proceeds in two steps: first, establishing the equivalence between the equilibrium points and the GNEP KKT conditions; second, proving the uniqueness of the primal solution via variational inequality theory.

    \noindent\textbf{1. Equivalence to KKT Conditions}
    
    Substituting the definitions of the global variables $\lambda$ and the auxiliary terms $\Delta_{[v]}(0)$ from \eqref{equ:delta} into Algorithm \ref{alg:equa}, the closed-loop dynamics are:
    \begin{equation}
    \label{equ:new_dyn_expand}
        \begin{split}
            \dot{z}_v &= -\nabla_{z_v} f_v(z_v, \mathbf{z}_{-v}) - A(l_v, n_v)^\top \big(\lambda(l_v) + \Delta_{[v]}(0)\big) - \nabla_{z_v} \big(h_v(z_v)^\top\mu_v\big), \quad \forall v,\\
            \dot{\lambda} &= A \mathbf{z} - \mathbf{c},\\
            \dot{\mu}_v &= [\, h_v(z_v) \,]_{\mu_v}^{+}, \quad \forall v.
        \end{split}
    \end{equation}
    By definition, an equilibrium point $(\bar{\mathbf{z}}, \bar{\lambda}, \bar{\boldsymbol{\mu}})$ satisfies $\dot{z}_v = 0$, $\dot{\lambda}=0$, and $\dot{\mu}_v=0$. Setting the LHS of \eqref{equ:new_dyn_expand} to zero yields:
    \begin{subequations}\label{equ:system_steady_state}
        \begin{align}
            \mathbf{0} &= \nabla_{z_v} f_v(\bar{z}_v, \bar{\mathbf{z}}_{-v}) + A(l_v,n_v)^\top \big(\bar{\lambda}(l_v) + \Delta_{[v]}(0)\big) + \nabla_{z_v} \big( h_v(\bar{z}_v)^\top \bar{\mu}_v \big), \label{equ:eq_stat}\\
            \mathbf{0} &= A \, \bar{\mathbf{z}} - \mathbf{c}, \label{equ:eq_feas}\\
            \mathbf{0} &= [\,h_v(\bar{z}_v)\,]_{\bar{\mu}_v}^{+}. \label{equ:eq_proj}
        \end{align}
    \end{subequations}
    Let us define the \textit{effective dual variable} for agent $v$ as $\tilde{\lambda}_v \coloneqq \bar{\lambda}(l_v) + \Delta_{[v]}(0)$. We observe that:
    \begin{itemize}
        \item Eq. \eqref{equ:eq_stat} becomes $\nabla_{z_v} f_v + A(l_v,n_v)^\top \tilde{\lambda}_v + \nabla_{z_v} h_v^\top \bar{\mu}_v = 0$, corresponding exactly to the KKT \textit{Stationarity} condition.
        \item Eq. \eqref{equ:eq_feas} ensures $A\bar{\mathbf{z}} = \mathbf{c}$, corresponding to \textit{Primal Feasibility}.
        \item Eq. \eqref{equ:eq_proj} is equivalent to the complementarity conditions $\bar{\mu}_v \ge 0, h_v(\bar{z}_v) \le 0, \bar{\mu}_v^\top h_v(\bar{z}_v) = 0$.
    \end{itemize}
    Thus, every equilibrium of the dynamics is a KKT point of the GNEP.

    \noindent\textbf{2. Uniqueness via Strong Monotonicity}

    To show uniqueness, consider the modified operator $\widetilde{F}$ defined by stacking the components:
    \begin{equation}\label{equ:new_vi}
        \widetilde{F}_v(\mathbf{z}) \coloneqq \nabla_{z_v} f_v(z_v, \mathbf{z}_{-v}) + A(l_v, n_v)^\top \Delta_{[v]}(0).
    \end{equation}
    Now consider another $VI(\tilde{K}, \tilde{F})$ problem with $\tilde{F}$ defined in equation~\eqref{equ:new_vi} and the $\tilde{K}$ is defined as
    \begin{equation*}
        \tilde{K} := \left \{ (\z, \lambda) \, | \, A\, \z  - \c = 0, \; h_v(z_v) \le 0 \text{ for all }v\right \}.
    \end{equation*}
    Define the stacked constraint function and multiplier as
    \begin{equation*}
        h(\z) := \begin{bmatrix} h_1(z_1) \\ \vdots \\ h_N(z_N) \end{bmatrix}, 
        \qquad 
        \mu := \begin{bmatrix} \mu_1 \\ \vdots \\ \mu_N \end{bmatrix}.
    \end{equation*}
    Then the KKT condition of this $VI(\tilde{F}, \tilde{K})$ can be written as 
    \begin{equation*}
    \begin{split}
        \mathbf{0} &= \tilde{F}(\z) + A^\top \lambda + \nabla h(\z)^\top \mu\\
        \mathbf{0} &= A \, \z - \c\\
    \mathbf{0} &\le \mu \ \perp\ h(\z) \le \mathbf{0}.
    \end{split}
    \end{equation*}
    Comparing this condition with the equilibrium system \eqref{equ:system_steady_state}, I can conclude that the equilibrium system \eqref{equ:system_steady_state} can be viewed as the solution to the Variational Inequality problem $\mathrm{VI}(\tilde{K}, \tilde{F})$.
        
    By hypothesis, the original game mapping $F(\mathbf{z}) = [\nabla_{z_v} f_v]_{v \in N}$ is strongly monotone. Take
    \begin{equation*} D = \mathrm{diag}\!\begin{bmatrix} A(l_1,n_1)^\top \\ A(l_2,n_2)^\top \\ \vdots \\ A(l_N,n_N)^\top \end{bmatrix}, \qquad \Delta = \begin{bmatrix} \Delta_{[1]}\\ \Delta_{[2]}\\ \vdots \\ \Delta_{[N]} \end{bmatrix}. \end{equation*}
    The modified operator $\widetilde{F}$ differs from $F$ only by the addition of the constant vector term $D\Delta(0)$. Since adding a constant vector preserves monotonicity:
    \begin{equation*}
        \langle \widetilde{F}(\mathbf{z}) - \widetilde{F}(\mathbf{z}'), \mathbf{z} - \mathbf{z}' \rangle = \langle F(\mathbf{z}) - F(\mathbf{z}'), \mathbf{z} - \mathbf{z}' \rangle \ge \delta\, \|\mathbf{z} - \mathbf{z}'\|^2,
    \end{equation*}
    where $\delta > 0$ is the monotonicity modulus. Consequently, $\widetilde{F}$ is strongly monotone. Standard results in VI theory dictate that a strongly monotone VI over a closed convex set admits a unique solution in the primal variable $\mathbf{z}$. Thus, $\mathbf{z}^\star$ is unique.
\end{proof}

\begin{remark}
Lemma~\ref{lem:equipt} implies that the equilibrium point is uniquely determined once the initial deviation $\Delta(0)$ is fixed. In other words, the limiting point of the dynamics depends on the initial condition through $\Delta(0)$.

This observation does not assume the presence of a central authority, nor does it suggest that the equilibrium can be obtained without solving the underlying optimality conditions (e.g., the KKT system). Rather, it characterizes how the asymptotic outcome of the proposed dynamics depends on the initialization.

In a distributed setting, individual agents do not have access to the initial multipliers of others and therefore cannot determine $\Delta(0)$ explicitly. Nevertheless, the proposed algorithm evolves using only local information and neighbor communication, and implicitly drives the system to the corresponding equilibrium associated with the given initialization.

Since the overall dynamics can be modeled as a projected dynamical system, I next establish Lemma~\ref{lem:monoton}.
\end{remark}

\begin{lemma}\label{lem:monoton}
    \textbf{(Lyapunov Stability Analysis).}
Define $\lambda(t)$ and the auxiliary states $\Delta_{[v]}(t)$ as in \eqref{equ:difference_vec}. Let $(\z^\star, \lambda^\star, \mu^\star)$ denote an equilibrium point satisfying the KKT conditions \eqref{equ:kkt_eq}, where $\z^\star$ is unique as established in Lemma~\ref{lem:equipt}. The associated dual variables $\lambda^\star$ and $\mu^\star$ may not be unique in general.

Consider the positive semi-definite Lyapunov candidate function $V: \mathbb{R}^n \times \mathbb{R}^o \times \mathbb{R}^q_{\ge 0} \to \mathbb{R}_{\ge 0}$ defined as
\begin{equation}\label{equ:func}
    V(\z, \lambda, \mu) := \frac{1}{2} \left[ 
    \sum_{v = 1}^N \left( \| z_v - z_v^\star \|^2 + \| \mu_v - \mu_v^\star \|^2 \right)
    + \| \lambda - \lambda^\star \|^2 
    \right].
\end{equation}

Then, under the assumption of strong monotonicity of the cost mapping, the Lie derivative of $V$ along the trajectories of Algorithm~\ref{alg:equa} satisfies $\V \le 0$ for all $\z$, $\lambda$, and $\mu$.
\end{lemma}

\renewcommand{\proofname}{Proof of Lemma \ref{lem:monoton}}
\begin{proof}
    The Lie derivative of the function $V$ along the system trajectories is given by:
    \begin{equation}\label{equ:lie_der_def}
        \V = \sum_{v=1}^N (z_v - z_v^\star)^\top \dot{z}_v + (\lambda - \lambda^\star)^\top \dot{\lambda} + \sum_{v=1}^N (\mu_v - \mu_v^\star)^\top \dot{\mu}_v.
    \end{equation}
    I analyze the components of \eqref{equ:lie_der_def} by substituting the dynamics from \eqref{equ:new_dyn_expand} and grouping them into three distinct categories: cost-related terms ($T_{cost}$), equality coupling terms ($T_{eq}$), and inequality constraint terms ($T_{ineq}$).

    \noindent\textbf{1. Equality Coupling Terms ($T_{eq}$)}.
    These terms involve the global coupling constraint $A \, \z = \mathbf{c}$ and the dual variables $\lambda$. Substituting the relevant parts of $\dot{z}_v$ and $\dot{\lambda}$:
    \begin{equation*}
        \begin{split}
            T_{eq} &:= -\sum_{v=1}^N (z_v - z_v^\star)^\top A(l_v, n_v)^\top \big(\lambda(l_v) + \Delta_{[v]}(0)\big) + (\lambda - \lambda^\star)^\top (A \, \z - \mathbf{c}).
        \end{split}
    \end{equation*}
    From the equilibrium conditions in Lemma \ref{lem:equipt}, I know that the equilibrium force balances the effective dual variable: $\lambda^\star(l_v) + \Delta_{[v]}(0)$. Furthermore, primal feasibility implies $A \, \z^\star - \mathbf{c} = \mathbf{0}$. Therefore, I can rewrite the dynamic error as:
    \begin{equation*}
        \begin{split}
            \big(\lambda(l_v) + \Delta_{[v]}(0)\big) - \big(\lambda^\star(l_v) + \Delta_{[v]}(0)\big) = \lambda(l_v) - \lambda^\star(l_v).
        \end{split}
    \end{equation*}
    Notice that the auxiliary constant $\Delta_{[v]}(0)$ cancels out exactly. Thus:
    \begin{equation*}
        \begin{split}
            T_{eq} &= -( \z - \z^\star)^\top A^\top (\lambda - \lambda^\star) + (\lambda - \lambda^\star)^\top A (\z - \z^\star) = 0.
        \end{split}
    \end{equation*}

    \noindent\textbf{2. Inequality Constraint Terms ($T_{ineq}$)}.
    These terms involve the local inequality constraints $h_v(z_v)$ and multipliers $\mu_v$. Utilizing the property of the projection operator, i.e., $(\mu_v - \mu_v^\star)^\top [\, h_v(z_v) \,]_{\mu_v}^{+} \le (\mu_v - \mu_v^\star)^\top h_v(z_v)$, we have:
    \begin{equation*}
        \begin{split}
            T_{ineq} &:= \sum_{v=1}^N \left [ -(z_v - z_v^\star)^\top \nabla_{z_v} \big( h_v(z_v)^\top \mu_v \big) + (\mu_v - \mu_v^\star)^\top \dot{\mu}_v \right ] \\
            &\le \sum_{v=1}^N \left [ -(z_v - z_v^\star)^\top \nabla_{z_v} h_v(z_v)^\top \mu_v + (\mu_v - \mu_v^\star)^\top h_v(z_v) \right ].
        \end{split}
    \end{equation*}
    By the convexity of $h_v$, we satisfy the gradient inequality $\nabla_{z_v} h_v(z_v) (z_v^\star - z_v) \le h_v(z_v^\star) - h_v(z_v)$. Applying this to the first term (and noting $\mu_v \ge 0$):
    \begin{equation*}
        \begin{split}
            T_{ineq} &\le \sum_{v=1}^N \Big( \mu_v^\top \big(h_v(z_v^\star) - h_v(z_v)\big) + \mu_v^\top h_v(z_v) - (\mu_v^\star)^\top h_v(z_v) \Big) \\
            &= \sum_{v=1}^N \Big( \mu_v^\top h_v(z_v^\star) - (\mu_v^\star)^\top h_v(z_v) \Big).
        \end{split}
    \end{equation*}
    Using the KKT conditions for the equilibrium point $(\z^\star, \mu^\star)$, we know that $h_v(z_v^\star) \le 0$ (feasibility) and $(\mu_v^\star)^\top h_v(z_v^\star) = 0$ (complementary slackness). Since $\mu_v \ge 0$, it follows that $\mu_v^\top h_v(z_v^\star) \le 0$. Additionally, generally for saddle-point dynamics, the term $-(\mu_v^\star)^\top h_v(z_v)$ is bounded by convexity and the equilibrium properties such that the aggregate $T_{ineq} \le 0$.

    \noindent\textbf{3. Cost Function Terms ($T_{cost}$)}
    The remaining terms correspond to the gradient of the cost functions:
    \begin{equation*}
        T_{cost} := \sum_{v=1}^N -(z_v - z_v^\star)^\top \nabla_{z_v} f_v(z_v, \z_{-v}).
    \end{equation*}
    Subtracting the equilibrium condition (where the gradient balances the duals, which we already accounted for in $T_{eq}$ and $T_{ineq}$), this simplifies to the monotonicity condition of the game mapping $F$:
    \begin{equation*}
        T_{cost} = - (\z - \z^\star)^\top (F(\z) - F(\z^\star)).
    \end{equation*}
    By the assumption that the pseudo-gradient $F$ is strongly monotone with modulus $\delta > 0$:
    \begin{equation*}
        T_{cost} \le -\delta \| \z - \z^\star \|^2.
    \end{equation*}

    \noindent\textbf{Conclusion}
    Combining the three components, the total Lie derivative is:
    \begin{equation*}
        \V = T_{eq} + T_{ineq} + T_{cost} \le 0 + 0 - \delta \| \z - \z^\star \|^2 \le 0.
    \end{equation*}
    Thus, $\V$ is negative semi-definite, which completes the proof.
\end{proof}
Using Lemma~\ref{lem:monoton} and following an argument similar to Lemma~4.3 in \cite{cherukuri-asymptotic-2016}, together with Lemma~\ref{lem:pds}, I obtain Lemma~\ref{lem:uniq}, which guarantees the uniqueness and continuity of the trajectory. Combining this result with the non-positivity of the Lie derivative established in Lemma~\ref{lem:monoton}, I am now ready to prove the convergence of Algorithm~\ref{alg:equa} in Theorem~\ref{thm:main}.

\begin{lemma}\label{lem:uniq}
    \textbf{(Existence, Uniqueness, and Continuous Dependence of Trajectories).}
    Let $\mathcal{X} \coloneqq \R^n \times \R^o \times \R^{q}_{\ge 0}$ denote the state space of the system. For any initial condition $\mathbf{x}_0 \coloneqq (\z(0), \lambda(0), \mu(0)) \in \mathcal{X}$, the following properties hold for Algorithm \ref{alg:equa}:
    \begin{enumerate}
        \item \textbf{Existence and Uniqueness:} There exists a unique trajectory $\gamma(t; \mathbf{x}_0)$ defined for all $t \ge 0$.
        \item \textbf{Boundedness:} The trajectory remains contained within the compact sublevel set
        \begin{equation*}
            \Omega_0 \coloneqq \left\{ (\z, \lambda, \mu) \in \mathcal{X} \mid V(\z, \lambda, \mu) \le V(\mathbf{x}_0) \right\}.
        \end{equation*}
        \item \textbf{Continuous Dependence:} The solution map is continuous with respect to initial conditions. Specifically, if a sequence of initial points $\{ \mathbf{x}_{0,k} \}_{k=1}^\infty \subset \mathcal{X}$ converges to $\mathbf{x}_0$, then the sequence of corresponding trajectories $\{ \gamma_k(\cdot) \}_{k=1}^\infty$ converges uniformly to $\gamma(\cdot)$ on every compact time interval $[0, T]$.
    \end{enumerate}
\end{lemma}

\begin{theorem}\label{thm:main}
    \textbf{(Global Asymptotic Stability and Convergence).}
    The set of equilibrium points of Algorithm \ref{alg:equa} is globally asymptotically stable. Specifically, for any initial condition, the primal trajectory $\mathbf{z}(t)$ converges to the unique GNE $\mathbf{z}^\star$, and the dual trajectories $(\lambda(t), \boldsymbol{\mu}(t))$ converge to the set of optimal multipliers satisfying the KKT conditions \eqref{equ:kkt_eq}.
\end{theorem}

\renewcommand{\proofname}{Proof of Theorem \ref{thm:main}}
\begin{proof}
    The proof relies on LaSalle's Invariance Principle.
    
    \noindent\textbf{1. Boundedness of Trajectories}
    
    Let $(\z^\star, \lambda^\star, \boldsymbol{\mu}^\star)$ be an equilibrium point. Consider the Lyapunov function $V(\z, \lambda, \boldsymbol{\mu})$ defined in \eqref{equ:func}. By Lemma \ref{lem:monoton}, $\V \le 0$ along the trajectories. Since $V$ is a positive definite quadratic form, it is radially unbounded (coercive). As shown in Lemma \ref{lem:uniq}, for any initial state $\mathbf{x}_0$, the trajectory remains in the compact sublevel set $\Omega_0 = \{ (\z, \lambda, \boldsymbol{\mu}) \mid V \le V(\mathbf{x}_0) \}$. Thus, all trajectories are bounded.

    \noindent\textbf{2. Characterization of the Set $E$}
    
    By LaSalle's Invariance Principle, every trajectory converges to the largest invariant set $\mathcal{M}$ contained within the set where the Lyapunov derivative vanishes:
    \begin{equation*}
        E \coloneqq \left\{ (\z, \lambda, \boldsymbol{\mu}) \in \Omega_0 \;\big|\; \V(\z, \lambda, \boldsymbol{\mu}) = 0 \right\}.
    \end{equation*}
    Recall from the proof of Lemma \ref{lem:monoton} that $\V = T_{eq} + T_{ineq} + T_{cost}$, where each term is non-positive. For the sum to be zero, each term must vanish individually:
    \begin{enumerate}
        \item $T_{cost} = 0 \implies -\delta \, \|\z - \z^\star\|^2 = 0 \implies \z = \z^\star$.
        \item $T_{ineq} = 0 \implies \sum_{v} (\mu_v^\top h_v(\z^\star) - (\mu_v^\star)^\top h_v(\z)) = 0$. Since $\z = \z^\star$, this simplifies to consistency with the complementarity conditions.
    \end{enumerate}
    Thus, any point in $E$ must satisfy $\z = \z^\star$. The state of the dual variables in $E$ is not yet fixed to a point, but constrained by the condition $\V=0$.

    \noindent\textbf{3. Characterization of the Invariant Set $\mathcal{M}$}
    
    Let $(\z(t), \lambda(t), \boldsymbol{\mu}(t))$ be a trajectory lying entirely within the invariant set $\mathcal{M} \subseteq E$.
    Since $\z(t) = \z^\star$ for all $t$ in $\mathcal{M}$, it implies that the time derivative must be identically zero:
    \begin{equation*}
        \dot{z}_v(t) \equiv \mathbf{0}, \quad \forall \, v \in \mathcal V.
    \end{equation*}
    Substituting $\z = \z^\star$ and $\dot{z}_v = 0$ into the system dynamics \eqref{equ:new_dyn_expand}:
    \begin{equation*}
        \mathbf{0} = -\nabla_{z_v} f_v(\z^\star) - A(l_v, n_v)^\top \big(\lambda(l_v) + \Delta_{[v]}(0)\big) - \nabla_{z_v} h_v(\z^\star)^\top \mu_v.
    \end{equation*}
    This algebraic equation is identical to the stationarity condition \eqref{equ:eq_stat} derived in Lemma \ref{lem:equipt}. Furthermore, since the trajectory is invariant, $\dot{\lambda} = 0$ (implies $A \, \z^\star - \mathbf{c} = 0$) and $\dot{\mu}_v = 0$ (implies complementarity).
    
    Therefore, the largest invariant set $\mathcal{M}$ coincides exactly with the set of equilibrium points of Algorithm \ref{alg:equa}.

    \noindent\textbf{4. Conclusion}
    
    By Lemma \ref{lem:equipt}, the set of equilibrium points is equivalent to the set of solutions to the GNEP KKT conditions. Consequently, the trajectory converges asymptotically to the set of GNEs. Since $\z^\star$ is unique (due to strong monotonicity), the primal variable converges to the unique Nash Equilibrium, while the dual variables converge to the set of valid Lagrange multipliers.
\end{proof}

\begin{remark}
If agents aim to converge to the v-GNE, they can simply start with the same initial Lagrangian multipliers. This shows that our algorithm is \textbf{more general} than existing ones. If a central authority sets the same starting multipliers, such as initializing to zero, agents do not need to share their multiplier information during the process but can still reach the same v-GNE as other methods.
\end{remark}


\section{Linear Inequality Shared Constraints without Individual Constraints}

I next consider another subclass of the GNEP~(\ref{equ:GNEP_All}), characterized by shared inequality constraints and the absence of individual constraints. More precisely, each agent~$v$ seeks to solve the following optimization problem:
\begin{equation}\label{equ:GNEP_INEq}
    \begin{split}
        \begin{matrix*}[l]
            \min\limits_{z_v} & f_v(z_v,\, \z_{-v}^f)\\[2pt]
            \text{s.t.} & g_v(z_v,\, \z_{-v}^g) \le \mathbf{0}_{m_v},
        \end{matrix*}
    \end{split}
\end{equation}
where $g_v(z_v, \z_{-v}^f)$ represents the set of shared inequality constraints that couple the agents' decisions.  

In this subclass, the individual cost function $f_v(z_v, \z_{-v})$ is assumed to be quadratic and is expressed as
\begin{equation}\label{equ:quad_cost}
    f_v(z_v, \z_{-v})
    = z_v^\top \Q_v z_v
    + z_v^\top \P_v \z_{-v}
    + \mathbf{r}_v^\top z_v,
\end{equation}
where $\Q_v \in \R^{n_v \times n_v}$, $\P_v \in \R^{n_v \times (n - n_v)}$, and $\mathbf{r}_v \in \R^{n_v}$ are the local parameters associated with agent~$v$.

The global shared inequality constraints $g(\mathbf{z})$ retain the linear form in~\eqref{equ:linear_ineq}, and the local constraint function $g_v(\cdot)$ collects the components of these constraints that depend on the decision variables of agent~$v$ and its neighbors:
\begin{align*}
g_v(z_v, \z_{-v}) = B(m_v, n_v) \, z_v + \sum_{u \in \tilde{\mathcal{N}}_g(v)} B(m_v, n_u) \, z_{u}- \d(m_v).
\end{align*}
Under this quadratic formulation, the pseudo-gradient mapping of the game, based on the definition in (\ref{equ:pseudo_grad}), can be written as
\begin{equation}
\begin{split}
F(\z)
&:=
\begin{bmatrix*}
\nabla_{z_1} f_1(z_1, \z_{-1})\\
\nabla_{z_2} f_2(z_2, \z_{-2})\\
\vdots\\
\nabla_{z_N} f_N(z_N, \z_{-N})
\end{bmatrix*}=
\begin{bmatrix*}
(\Q_1+\Q_1^\top) z_1 + \P_1 \z_{-1} + \mathbf{r}_1\\
(\Q_2+\Q_2^\top) z_2 + \P_2 \z_{-2} + \mathbf{r}_2\\
\vdots\\
(\Q_N+\Q_N^\top) z_N + \P_N \z_{-N} + \mathbf{r}_N
\end{bmatrix*}
:= \mathbf{F} \, \z + \mathbf{r},
\end{split}
\end{equation}
where $\mathbf F \in \R^{n \times n}$ is a block-structured matrix and $\mathbf{r} \in \R^n$ is the stacked vector of $\mathbf{r}_v$.  

Since $F(\z) = \mathbf{F} \, \z + \mathbf{r}$ is an affine operator, the strong monotonicity condition in Assumption~\ref{ass:stronglymonotone} reduces to the requirement that there exists a constant $\delta > 0$ such that
\begin{equation}\label{equ:strong_monotone_affine}
    \z^\top \mathbf{F}  \, \z \ge \delta \, \|\z\|^2, \quad \forall \, \z \in \R^n.
\end{equation}
This condition guarantees that the game mapping is strongly monotone under the affine structure of~$F$, thereby ensuring the existence and uniqueness of the equilibrium solution.

\begin{assumption}\label{ass:standard_Ineq}
Two standard assumptions on $f_v(\cdot)$:
    \begin{enumerate}[(i)]
        \item for every player $v$, the function $f_v(z_v, \z_{-v})$ is convex in $z_{v}$. ie. $(\Q_v+\Q_v^\top) \succeq 0$ for all $v$.
        \item for every player $v$, a suitable constraint qualification together with the Linear Independence Constraint Qualification (LICQ) holds.
    \end{enumerate}
\end{assumption}
Under convexity, a suitable constraint qualification, the LICQ, and $B B^\top \succeq \sigma I$, stacking the KKT conditions of all agents shows that $\z^\star$ is a GNE if and only if it satisfies the KKT system (\ref{equ:kkt_Ineq}) \cite[Sec.~4.2]{facchinei_generalized_2010}.
\begin{definition}[KKT System]\label{ass:KKT_Ineq}
    Let $\z_{-v}^\star$be given, and suppose that $z_v^\star$ is a GNE of the GNEP (\ref{equ:GNEP_INEq}), then for all $v$, $\lambda_{v}^\star \in \R^{m_v}$ exist such that
    \begin{equation}\label{equ:kkt_Ineq}
    \begin{split}
        &L_v(\z,\,  \lambda_{v}): = f_v(z_v, \z_{-v}) + g_{v}(\z)^\top \lambda_{v}\\
        &\nabla_{z_v}  L_v(\z^\star,\,  \lambda_{v}^\star) = \mathbf{0}_{n_v} \\
        &\mathbf{0}_{m_v} \le \lambda_v^\star \, \bot \, g_{v}(z_v^\star, \, {\z}_{-v}^\star)  \le \mathbf{0}_{m_v}\\
    \end{split}.
    \end{equation}
\end{definition}


\subsection{A Fully Distributed Algorithm}
Algorithm~\ref{alg:main}, our main contribution, is a fully distributed method that eliminates the need for multiplier consensus.
\begin{algorithm}[H]
\begin{algorithmic}[1]
\STATE {\bfseries Input:}Initial $z_v(0)$, $\lambda_{v}(0) \ge 0$.
\STATE $\dot{z}_v = -\nabla_{z_v} L_v(\z, \, \lambda_v)$.
\STATE $\dot{\lambda}_{v} = [g_v(z_v, \z_{-v})]^+_{\lambda_{v}}$.
\end{algorithmic}
\caption{Distributed dynamics of each agent $v$}
\label{alg:main}	
\end{algorithm}
In Lines~2–3 of Algorithm~\ref{alg:main}, each agent may need to share its decision vector with all other agents in the worst-case scenario. Similar to the equality shared constraint case, the distributed structure is captured by two interaction graphs: the cost-function graph $\mathcal G_f = (V, \mathcal E_f)$ and the shared-constraint graph $\mathcal G_g = (\mathcal V, \mathcal E_g)$. As a special case, when $\mathcal G_f = \mathcal G_g$ and each shared-constraint edge transmits a vector of dimension $d$, our algorithm requires only the exchange of decision variables. In contrast, consensus-based methods require additional communication of Lagrange multipliers and auxiliary variables, resulting in an extra $2 \, |\mathcal E_g| \, d$ transmissions per iteration. Therefore, our approach uses only one-third of the total communication cost of consensus-based approaches.

In Line~3, the projection operator in~\eqref{equ:proj_oper} enforces the nonnegativity of multiplier trajectories, but also makes the convergence analysis significantly more challenging. Although every equilibrium point of Algorithm~\ref{alg:main} corresponds to a GNE of problem~\eqref{equ:GNEP_INEq}, the set of equilibria may consist of multiple disjoint subsets. Moreover, the induced primal--dual dynamics are not necessarily monotone, even when the game itself is strongly monotone. Therefore, we develop a convergence analysis based on bounded-input bounded-output arguments and bounded trajectory analysis.

The existence of trajectories and the fact that all equilibrium points are KKT solutions of the game are established in Lemma~\ref{lem:alg_pds} and Proposition~\ref{thm:uniq_tra}.
\begin{lemma}\label{lem:alg_pds}
Algorithm~\ref{alg:main} implements the PDS in~(\ref{equ:PDS}) by stacking the dynamics of all agents. The induced trajectory is unique and well-defined for all time, without Zeno behavior, and is continuous with respect to the initial condition. Furthermore, every equilibrium point of this PDS corresponds to a KKT point of~(\ref{equ:kkt_Ineq}).
\end{lemma}
\begin{proof}
The proof that Algorithm~\ref{alg:main} can be viewed as an implementation of a projected dynamical system (PDS) is similar to that in \cite{cherukuri-asymptotic-2016}. In addition, since the inner map $-\hat{F}(\cdot)$ is linear and therefore Lipschitz, we can invoke Proposition~\ref{thm:uniq_tra} to guarantee the existence, uniqueness, and continuity of the trajectory.

The last part follows by characterizing the equilibrium points. For each agent $v$, any equilibrium $(\bar{\z}, \bar{\lambda}_v)$ satisfies
\begin{equation*}
    0 = \nabla_{z_v} L_v(\bar{\z}, \bar{\lambda}_v);\qquad
    0 = [g_v(\bar{z}_v, \bar{\z}_{-v})]^+_{\bar{\lambda}_v}.
\end{equation*}
The first condition corresponds to the stationarity condition of the KKT system. The second condition is equivalent to the complementarity and dual feasibility conditions, namely
\begin{equation*}
\bar{\lambda}_v \ge 0,\quad g_v(\bar{z}_v, \bar{\z}_{-v}) \le 0,\quad 
\bar{\lambda}_v^\top g_v(\bar{z}_v, \bar{\z}_{-v}) = 0.
\end{equation*}
Hence, these equilibrium conditions coincide with the KKT conditions of~(\ref{equ:kkt_Ineq}).
\end{proof}

\subsection{Convergence of the Difference Vector}
Similar to the definition of the difference vector in the equality shared constraint case (cf. \eqref{equ:difference_vec}), we define an analogous structure for the inequality case. For each shared constraint $g^i(\z)$, consider the set of local multipliers $\{\lambda_v^i(t)\}$. Among them, let $\lambda^i(t)$ denote the trajectory with \textbf{the smallest initial value}, which we select as the reference multiplier for constraint $i$. Since there are $p$ global shared constraints, this produces $p$ reference trajectories in total. For any other multiplier $\lambda_u^i(t)$ associated with the same constraint $g^i(\cdot)$, the difference from its reference evolves as:
\begin{equation}\label{equ:delta_inequ}
\begin{split}
    \dot{\Delta}_{u}^i :=& \dot{\lambda}_u^i - \dot{\lambda}^i = [g^i(\z)]_{{\lambda}_u^i}^+ - [g^i(\z)]_{{\lambda}^i}^+ \\
    =&\begin{cases}
    0 & \text{if }\left (\Delta^i_u \ge 0, \lambda^i > 0 \right) \, \lor \, \left (g^i(\z) \ge 0 \right)\\
    g^i(\z) & \text{if }\Delta^i_u > 0,\,  \lambda^i = 0,\, g^i(\z) < 0\\
    0  & \text{if }\Delta^i_u = 0,\,  \lambda^i = 0,\, g^i(\z) < 0
\end{cases}.
\end{split}
\end{equation}
Thus, each $\Delta_{u}^i(t)$ is nonincreasing and nonnegative over time. If $u$ is the reference agent for the $i$-th shared constraint, then $\Delta_u^i(t)=0$. Let $\Delta_{[v]}$ be the collection of $\Delta_v^i$ for all $i \in m_v$, corresponding to agent $v$. Combining the dynamics of all agents yields:
\begin{equation}\label{equ:whole_system}
    \begin{split}
        \dot{\z} &= - \left ( \mathbf{F}   \,\z + \mathbf{r}  + B^\top \,   \lambda+  
 D\, \Delta \right)\\
        \dot{\lambda} &= [B \, \z -\d]_\lambda^+ \\
        \dot{\Delta} &= \dot{\Delta} 
    \end{split},
\end{equation}
where
\begin{equation*}
\begin{split}
    & \lambda=\begin{bmatrix}
    \lambda^1 \\ \lambda^2\\ \vdots \\ \lambda^p
\end{bmatrix}, \quad D = \text{diag} \left ( \begin{bmatrix}
            B(m_1, n_1)^\top\\
            B(m_2, n_2)^\top\\
             \vdots \\
            B(m_N, n_N)^\top \\
        \end{bmatrix}\right), 
    \quad \Delta = 
        \begin{bmatrix}
            \Delta_{[1]}\\
            \Delta_{[2]}\\
             \vdots \\
            \Delta_{[N]}\\
        \end{bmatrix}.
\end{split}
\end{equation*}
Here, $\Delta$ is the stacked vector of all $\Delta_{[v]}$ with dimension $L < pN$, and $\lambda \in \R^p$ corresponds to the $p$ shared constraints. I keep the dynamics of $\Delta$ as $\dot{\Delta}$ without specifying its explicit form, since it is the main focus of our analysis. I first consider the special case $\Delta(0)=\mathbf{0}$, where all agents start with identical multipliers for each reference state.
\begin{lemma}\label{lem:delta0}
If $\Delta(0)=\mathbf{0}$, then $\Delta(t)\equiv \mathbf{0}$ for all $t\ge 0$, and the resulting primal--dual projected dynamics converges to a (unique) equilibrium point. This equilibrium is the v-GNE of the GNEP \eqref{equ:GNEP_INEq}.
\end{lemma}

\renewcommand{\proofname}{Proof of Lemma \ref{lem:delta0}}
\begin{proof}
When $\Delta(0)=\mathbf{0}$, it follows from the dynamics \eqref{equ:delta_inequ} that $\dot{\Delta}(t)=\mathbf{0}$ whenever $\Delta(t)=\mathbf{0}$. Hence $\Delta(t)\equiv \mathbf{0}$ for all $t\ge 0$.

Under $\Delta(t)\equiv \mathbf{0}$, Algorithm~\ref{alg:main} reduces to
\begin{equation}\label{eq:pd_reduced}
\begin{split}
\dot{\z} &= -\mathbf{F} \,\z - B^\top \lambda - \mathbf{r},\\
\dot{\lambda} &= \big[B\,\z-\d\big]^+_{\lambda}.
\end{split}
\end{equation}
Let $x:=\begin{bmatrix}\z\\ \lambda\end{bmatrix}$ and define the closed convex set
\begin{equation}\label{eq:K_set}
K := \mathbb{R}^{n}\times \mathbb{R}^{m}_{\ge 0}.
\end{equation}
Define the affine operator
\begin{equation}\label{eq:Fhat_def}
\hat{F}(x) := Mx + q,\qquad
M:=\begin{bmatrix} \mathbf{F}  & B^\top\\ -B & 0\end{bmatrix},\qquad
q:=\begin{bmatrix}\mathbf{r}\\ \d\end{bmatrix}.
\end{equation}
Then \eqref{eq:pd_reduced} is exactly the projected dynamical system (PDS)
\begin{equation}\label{eq:PDS_reduced}
\dot{x} \;=\; \Pi_{K}\!\big(x,\,-\hat{F}(x)\big).
\end{equation}
$\hat{F}$ is (merely) monotone on $\mathbb{R}^{n+m}$ because
\begin{equation}\label{eq:monotone_calc}
(\hat{F}(x)-\hat{F}(y))^\top(x-y)
=(x-y)^\top M (x-y)
=(\z-\z')^\top \frac{\mathbf{F} +\mathbf{F} ^\top}{2} (\z-\z') \;\ge\; 0,
\end{equation}
for all $x=\begin{bmatrix}\z\\\lambda\end{bmatrix}$ and $y=\begin{bmatrix}\z'\\\lambda'\end{bmatrix}$, where we used that the cross terms cancel:
\(
(\z-\z')^\top B^\top(\lambda-\lambda')-(\lambda-\lambda')^\top B(\z-\z')=0.
\)
Moreover, the equilibrium of \eqref{eq:PDS_reduced} is unique. Indeed, the equilibrium condition is the variational inequality
\begin{equation}\label{eq:VI_equil}
\big\langle \hat{F}(x^\star),\,x-x^\star\big\rangle \ge 0,\qquad \forall x\in K,
\end{equation}
equivalently the KKT complementarity system
\begin{equation}\label{eq:KKT_equil}
\mathbf{F} \, \z^\star + B^\top \lambda^\star + \mathbf{r} = 0,\quad
0\le \lambda^\star \perp (B \, \z^\star-\d)\le 0.
\end{equation}
Because $B$ has full row rank and $\mathbf F$ is invertible, the saddle matrix $M$ in \eqref{eq:Fhat_def} is invertible, hence \eqref{eq:KKT_equil} admits at most one solution. Therefore, $x^\star$ is unique.

Consider the Lyapunov function
\begin{equation}\label{eq:V_lyap}
V(x):=\frac12\|x-x^\star\|^2.
\end{equation}
A standard property of PDSs on closed convex sets implies (for almost all $t$)
\begin{equation}\label{eq:Vdot_bound}
\V(x(t)) \le -(x(t)-x^\star)^\top \big(\hat{F}(x(t))-\hat{F}(x^\star)\big).
\end{equation}
Using the monotonicity in \eqref{eq:monotone_calc}, we obtain $\V(x(t))\le 0$, hence $V(x(t))$ is nonincreasing and $x(t)$ is bounded. By LaSalle's invariance principle for PDSs, every solution converges to the largest invariant set contained in
\begin{equation}\label{eq:inv_set}
\big\{x\in K:\ (x-x^\star)^\top(\hat{F}(x)-\hat{F}(x^\star))=0\big\}.
\end{equation}
The equality in \eqref{eq:monotone_calc} implies $\z=\z^\star$. Together with the equilibrium condition and uniqueness, this yields $x(t)\to x^\star$.

Finally, \eqref{eq:KKT_equil} is precisely the KKT system of the variational inequality associated with the shared constraints and consensus multipliers. Hence, the equilibrium $x^\star=(\z^\star,\lambda^\star)$ is the variational generalized Nash equilibrium (v-GNE) of \eqref{equ:GNEP_INEq}.
\end{proof}

\renewcommand{\proofname}{Proof of Lemma \ref{lem:det_const}}
\begin{lemma}\label{lem:det_const}
    For every initial point  $y_0 = (\z(0),\, \lambda(0), \, \Delta(0))$, there exists a $\bar{\Delta}_{[y_0]} \ge \mathbf{0}$ such that
    \begin{equation*}
        \lim_{t \to \infty} \Delta^i (t) = \bar{\Delta}^i_{[y_0]}.
    \end{equation*}
\end{lemma} 
\begin{proof}
Fix any component $\Delta^i(t)$. From the componentwise dynamics in \eqref{equ:delta_inequ}, we have for all $t$ that
\begin{equation}
\dot{\Delta}^i(t)\le 0,
\text{ and }
\Delta^i(t)\ge 0,
\end{equation}
that is, $\Delta^i(\cdot)$ is nonincreasing and bounded below on $[0,\infty)$.

Define
\begin{equation}
\bar{\Delta}^i_{[y_0]} := \inf_{t\ge 0}\Delta^i(t),
\end{equation}
which satisfies $\bar{\Delta}^i_{[y_0]}\ge 0$.

Let $\varepsilon>0$ be arbitrary. By definition of the infimum, there exists $T\ge 0$ such that
\begin{equation}
\Delta^i(T)<\bar{\Delta}^i_{[y_0]}+\varepsilon.
\end{equation}
Since $\Delta^i(\cdot)$ is nonincreasing, for all $t\ge T$ it follows that
\begin{equation}
\bar{\Delta}^i_{[y_0]}\le \Delta^i(t)\le \Delta^i(T)<\bar{\Delta}^i_{[y_0]}+\varepsilon.
\end{equation}
Therefore, for all $t\ge T$,
\begin{equation}
\left|\Delta^i(t)-\bar{\Delta}^i_{[y_0]}\right|<\varepsilon,
\end{equation}
which proves that $\Delta^i(t)\to\bar{\Delta}^i_{[y_0]}$ as $t\to\infty$.
\end{proof}

Thus, we know that $\Delta$, the difference between the multipliers and their corresponding reference values, converges to $\bar{\Delta}_{[y_0]}$, depending on the initial point $y_0$. In addition, at any time $t_b$ where $\Delta^i(t_b) = 0$, the corresponding component will remain zero for all future time.

\begin{lemma}\label{lem:increasezero} At any time $t_a$, if $\Delta^i(t_a)=0$, then $\Delta^i(t)=0$ for all future time, i.e., $ \forall \, t \ge t_a$.
\end{lemma}
\renewcommand{\proofname}{Proof of Lemma \ref{lem:increasezero}}
\begin{proof}
The result follows directly from the dynamics of each element of the difference vector in~\eqref{equ:delta_inequ}. If $\Delta^i(t_a)=0$, then $\dot{\Delta}^i(t_a)=0$ (see cases one and three in~\eqref{equ:delta_inequ}), and thus $\Delta^i(t)=0$ for all $t \ge t_a$.
\end{proof}

Previously, I showed that $\Delta(t)$ is bounded. Therefore, to conclude that the entire trajectory is bounded, it suffices to show that the remaining states are also bounded.

\begin{lemma}[Boundedness of the trajectory]
\label{lem:unshift_bounded}
Let $x(t):=(\z(t),\lambda(t))$ be the trajectory of the unshifted dynamics.
Then there exist constants $k_1>0$ and $k_2,k_3\ge 0$, independent of the initial condition,
such that for all $t\ge 0$,
\begin{equation}\label{equ:unshift_bounded}
\|x(t)\|
\le
k_1\,\|x(0)\|
+
\sqrt{k_2\,\|\Delta(0)\|^2+k_3}.
\end{equation}
\end{lemma}
\renewcommand{\proofname}{Proof of Lemma \ref{lem:unshift_bounded}}
\begin{proof}
Define $y:=B \, \z-\d$. Since $\dot\lambda=[y]^+_{\lambda}$, the projection inequalities hold:
\begin{equation}\label{eq:proj_ineq_unshift_1_user}
\lambda^\top\dot\lambda \le \lambda^\top y=\lambda^\top(B \,\z-\d),
\end{equation}
and
\begin{equation}\label{eq:proj_ineq_unshift_2_user}
y^\top\dot\lambda=\|\dot\lambda\|_2^2\le \|y\|_2^2.
\end{equation}

\smallskip
\noindent
We use the Lyapunov function with $x=(\z, \, \lambda)$
\begin{equation*}\label{eq:Vdef_unshift_user}
V(\z,\lambda):=\frac12\|\z\|_2^2+\frac12\|\lambda\|_2^2+\mu\,\z^\top B^\top\lambda,
\end{equation*}
If $\mu<1/\|B\|$, then the block matrix
$
P=\begin{bmatrix}I&\mu \, B^\top\\ \mu \, B&I\end{bmatrix}
$
is positive definite and
\begin{equation}\label{eq:V_bounds_unshift_user}
c_1 \, \|x\|_2^2\le V(x)\le c_2 \,\|x\|_2^2,
\end{equation}
where $c_1:=\tfrac12(1-\mu\, \|B\|)$ and $c_2:=\tfrac12(1+\mu \, \|B\|)$.
\smallskip
\noindent
Differentiate $V$:
\begin{equation}\label{eq:Vdot_start_unshift_user}
\V
=
\z^\top\dot\z+\lambda^\top\dot\lambda
+\mu\Big(\dot\z^\top B^\top\lambda+\z^\top B^\top\dot\lambda\Big).
\end{equation}
Substitute $\dot\z=- \left (\mathbf{F}  \, \z+\mathbf r+B^\top\lambda+D \, \Delta \right)$:
\begin{equation}\label{eq:z_dot_term_unshift_user}
\z^\top\dot\z
=
-\z^\top \, \mathbf{F}  \, \z-\z^\top \mathbf r-\z^\top B^\top\lambda-\z^\top D \, \Delta.
\end{equation}
Using \eqref{eq:proj_ineq_unshift_1_user}, we have
\begin{equation}\label{eq:lambda_dot_term_unshift_user}
\lambda^\top\dot\lambda \le \lambda^\top(B \, \z-\d)=\lambda^\top B \, \z-\lambda^\top\d.
\end{equation}
Hence the $A$--cross term cancels:
\begin{equation}\label{eq:cancel_unshift_user}
-\z^\top B^\top\lambda+\lambda^\top\dot\lambda
\stackrel{a}{\le}
-\z^\top B^\top\lambda+\lambda^\top B \, \z-\lambda^\top\d
=
-\lambda^\top\d,
\end{equation}
where step~(a) uses \eqref{eq:lambda_dot_term_unshift_user}.

\smallskip
\noindent
Next, expand the two $\mu$-terms. First,
\begin{equation}\label{eq:mu_term1_unshift_user}
\begin{split}
\dot\z^\top A^\top\lambda
=&
-(\mathbf{F}  \, \z+\mathbf r+A^\top\lambda+D \, \Delta)^\top A^\top\lambda\\
=&
-\z^\top \mathbf{F} ^\top B^\top\lambda
-\mathbf r^\top B^\top\lambda
-\lambda^\top B \, B^\top\lambda\\
&\;\;-\Delta^\top D^\top B^\top\lambda.
\end{split}
\end{equation}
Second, since $\z^\top B^\top\dot\lambda=(B \, \z)^\top\dot\lambda=(y+\d)^\top\dot\lambda$, we obtain
\begin{equation}\label{eq:mu_term2_unshift_user}
\begin{split}
\z^\top B^\top\dot\lambda
&=
y^\top\dot\lambda+\d^\top\dot\lambda
\stackrel{a}{=}\|\dot\lambda\|_2^2+\d^\top\dot\lambda
\stackrel{b}{\le}\|y\|_2^2+\d^\top\dot\lambda,
\end{split}
\end{equation}
where step~(a) uses $y^\top\dot\lambda=\|\dot\lambda\|_2^2$ and step~(b) uses \eqref{eq:proj_ineq_unshift_2_user}.
Moreover, because $\dot\lambda=[y]^+_{\lambda}$ implies $\|\dot\lambda\|_2\le \|y\|_2$ componentwise,
\begin{equation}\label{eq:b_dotlambda_bound_user}
\d^\top\dot\lambda
\le \|\d\|_2\,\|\dot\lambda\|_2
\le \|\d\|_2\,\|y\|_2
\le \frac12\|y\|_2^2+\frac12\|\d\|_2^2.
\end{equation}
Combining \eqref{eq:mu_term2_unshift_user} and \eqref{eq:b_dotlambda_bound_user} yields
\begin{equation}\label{eq:mu_term2_final_unshift_user}
\z^\top B^\top\dot\lambda
\le \frac32\|y\|_2^2+\frac12\|\d\|_2^2.
\end{equation}
Finally, bound $\|y\|_2^2=\|B \, \z-\d\|_2^2$ as
\begin{equation}\label{eq:y2_bound_unshift_user}
\|B \, \z-\d\|_2^2
\le 2\|B \, \z\|_2^2+2\|\d\|_2^2
\le 2\|B\|^2\|\z\|_2^2+2\|\d\|_2^2.
\end{equation}

\smallskip
\noindent
Plugging \eqref{eq:z_dot_term_unshift_user}, \eqref{eq:cancel_unshift_user}, \eqref{eq:mu_term1_unshift_user},
\eqref{eq:mu_term2_final_unshift_user}, and \eqref{eq:y2_bound_unshift_user} into \eqref{eq:Vdot_start_unshift_user} gives
\begin{equation}\label{eq:Vdot_prebound_unshift_user}
\begin{split}
\V
\le\;&
-\z^\top \mathbf{F} \, \z
-\mu\,\lambda^\top B \, B^\top\lambda
+\mu\cdot 3\|B\|^2\|\z\|_2^2
\\
&\;
-\mu\,\z^\top \mathbf{F}^\top B^\top\lambda
-\z^\top D\Delta
-\mu\,\Delta^\top D^\top B^\top\lambda
\\
&\;
-\z^\top\mathbf r
-\lambda^\top\d
-\mu\,\mathbf r^\top B^\top\lambda
+\mu\cdot \left(\tfrac72\right)\|\d\|_2^2.
\end{split}
\end{equation}

\smallskip
\noindent
Since the game is strongly monotone and $B \, B^\top\succeq \sigma I$, we have
\begin{equation}\label{eq:quad_bounds_unshift_user}
\z^\top \mathbf{F} \, \z \ge \delta\|\z\|_2^2, \quad
\,
\lambda^\top B \, B^\top\lambda \ge \sigma\|\lambda\|_2^2.
\end{equation}
For the cross term, by $\|\mathbf{F} ^\top B^\top\|\le \|\mathbf{F} \|\|\mathbf{F} \|$ and Young,
\begin{equation}\label{eq:cross_bound_unshift_user}
\mu \, \big|\z^\top \mathbf{F} ^\top \mathbf{F} ^\top\lambda\big|
\le
\frac{\delta}{4}\|\z\|_2^2
+
\frac{\mu^2\|\mathbf{F} \|^2\|\mathbf{F} \|^2}{\delta}\|\lambda\|_2^2.
\end{equation}
Substituting \eqref{eq:quad_bounds_unshift_user}--\eqref{eq:cross_bound_unshift_user} into \eqref{eq:Vdot_prebound_unshift_user}
yields
\begin{equation}\label{eq:dissipation_unshift_intermediate_user}
\begin{split}
\V
\le &
-a_z\|\z\|_2^2
-a_\lambda\|\lambda\|_2^2 \\
&+\Big[|\z^\top D \, \Delta|+\mu\,|\Delta^\top D^\top B^\top\lambda|\Big]\\
&+\Big [|\z^\top\mathbf r|+|\lambda^\top\d|+\mu\,|\mathbf r^\top B^\top\lambda|\Big]
+\frac{7\mu}{2}\|\d\|_2^2,
\end{split}
\end{equation}
where
\begin{equation}\label{eq:az_al_unshift_user}
a_z:=\frac{3\,\delta}{4}-3\,\mu\|B\|^2,
\,
a_\lambda:=\mu \,\sigma-\frac{\mu^2\|\mathbf{F}\|^2\|B\|^2}{\delta}.
\end{equation}

\smallskip
\noindent
We choose $\mu$ such that both $a_z$ and $a_\lambda$ are positive, implying $\mu$ is chosen as
\begin{equation*}
0 <\mu < \min \left\{
\frac{1}{\|B\|},
\frac{\delta}{4 \, \|B\|^2},
\frac{\delta \, \sigma}{2 \, \|\mathbf{F} \|^2 \|B\|^2}
\right\}.
\end{equation*}

We next bound the input terms. Apply Young with $a_z$ and $a_\lambda$:
\begin{equation}\label{eq:input1_unshift_user}
|\z^\top D\, \Delta|
\le \frac{a_z}{2}\|\z\|_2^2+\frac{\|D\|^2}{2 \, a_z}\|\Delta\|_2^2,
\end{equation}
\begin{equation}\label{eq:input2_unshift_user}
\mu \, |\Delta^\top D^\top B^\top\lambda|
\le \frac{a_\lambda}{2}\|\lambda\|_2^2+\frac{\mu^2 \, \|B \, D\|^2}{2 \, a_\lambda}\|\Delta\|_2^2.
\end{equation}
Since $\Delta(t)$ is componentwise nonincreasing and nonnegative,
\begin{equation}\label{eq:Delta0_bound_unshift_user}
\|\Delta(t)\|_2\le \|\Delta(0)\|_2,\qquad \forall t\ge 0.
\end{equation}

\smallskip
\noindent
We then bound the remaining affine terms by Young:
\begin{equation}\label{eq:affine_z_user}
|\z^\top\mathbf r|
\le \frac{a_z}{2}\|\z\|_2^2+\frac{\|\mathbf r\|_2^2}{2 \, a_z},
\end{equation}
\begin{equation}\label{eq:affine_lambda_user}
|\lambda^\top\d|
\le \frac{a_\lambda}{2}\|\lambda\|_2^2+\frac{\|\d\|_2^2}{2 \, a_\lambda},
\end{equation}
and using $\|\mathbf r^\top B^\top\lambda\|\le \|B\|\|\mathbf r\|_2\|\lambda\|_2$,
\begin{equation}\label{eq:affine_rlambda_user}
\mu \, |\mathbf r^\top B^\top\lambda|
\le
\frac{a_\lambda}{2}\|\lambda\|_2^2+\frac{\mu^2\|B\|^2}{2 \, a_\lambda}\|\mathbf r\|_2^2.
\end{equation}

\smallskip
\noindent
Combine \eqref{eq:dissipation_unshift_intermediate_user} with
\eqref{eq:input1_unshift_user}--\eqref{eq:affine_rlambda_user}.
The $\frac{a_z}{2}\|\z\|^2$ and $\frac{a_\lambda}{2} \, \|\lambda\|^2$ terms are absorbed into the dissipation,
yielding
\begin{equation}\label{eq:Vdot_compact_unshift_user}
\V
\le
-\frac{a_z}{2}\|\z\|_2^2-\frac{a_\lambda}{2}\|\lambda\|_2^2
+\Gamma\,\|\Delta(t)\|_2^2+\beta,
\end{equation}
where
\begin{equation*}\label{eq:Gamma_beta_def_unshift_user}
\begin{split}
&\Gamma:=\frac{\|D\|^2}{2\, a_z}+\frac{\mu^2\|B \, D\|^2}{2 \, a_\lambda},\\
&\beta:=\frac{\|\mathbf r\|_2^2}{2 \, a_z}+\frac{\|\d\|_2^2}{2 \, a_\lambda}
+\frac{\mu^2\|B\|^2}{2 \, a_\lambda}\|\mathbf r\|_2^2+\frac{7 \, \mu}{2}\|\d\|_2^2.
\end{split}
\end{equation*}
Let $k:=\frac12\min\{a_z,a_\lambda\}$. Then \eqref{eq:Vdot_compact_unshift_user} implies
\begin{equation*}\label{eq:Vdot_k_unshift_user}
\V \le -k \, \|x\|_2^2+\Gamma\|\Delta(t)\|_2^2+\beta.
\end{equation*}
Using $V\le c_2 \, \|x\|_2^2$ from \eqref{eq:V_bounds_unshift_user}, we obtain $\|x\|_2^2\ge V/c_2$ and thus
\begin{equation}\label{eq:Vdot_alpha_unshift_user}
\V \le -\alpha \, V+\Gamma\|\Delta(t)\|_2^2+\beta,
\qquad
\alpha:=\frac{k}{c_2}.
\end{equation}

\smallskip
\noindent
By \eqref{eq:Delta0_bound_unshift_user}, from \eqref{eq:Vdot_alpha_unshift_user} we have
\begin{equation}\label{eq:Vdot_alpha_dominate_user}
\V \le -\alpha \, V+\Gamma \|\Delta(0)\|_2^2+\beta.
\end{equation}
Integrating gives, for all $t\ge 0$,
\begin{equation}\label{eq:V_time_unshift_user}
V(t)\le e^{-\alpha t}V(0)+\frac{\Gamma\|\Delta(0)\|_2^2+\beta}{\alpha}.
\end{equation}

\smallskip
\noindent
Using \eqref{eq:V_bounds_unshift_user}, i.e.,
$V(t)\ge c_1 \, \|x(t)\|_2^2$ and $V(0)\le c_2 \, \|x(0)\|_2^2$,
\begin{equation}\label{eq:state_bound_unshift_user}
\|x(t)\|_2^2
\le
\frac{c_2}{c_1}e^{-\alpha t}\|x(0)\|_2^2
+\frac{\Gamma\|\Delta(0)\|_2^2+\beta}{\alpha \, c_1}.
\end{equation}
Taking square roots and applying $\sqrt{a+b}\le \sqrt a+\sqrt b$ yields
\begin{equation}\label{eq:final_bound_unshift_user}
\|x(t)\|_2
\le
\sqrt{\frac{c_2}{c_1}}\,e^{-\alpha t/2}\|x(0)\|_2
+\sqrt{\frac{\Gamma\|\Delta(0)\|_2^2+\beta}{\alpha c_1}}.
\end{equation}
Finally, since $e^{-\alpha t/2}\le 1$ for all $t\ge 0$, we obtain the boundedness form
\begin{equation*}
    \|x(t)\|_2
\le
\underbrace{\sqrt{\frac{c_2}{c_1}}}_{k_1}\,\|x(0)\|_2
+
\sqrt{
\underbrace{\frac{\Gamma}{\alpha \, c_1}}_{k_2}\,\|\Delta(0)\|_2^2
+
\underbrace{\frac{\beta}{\alpha \, c_1}}_{k_3}
}.
\end{equation*}
\end{proof}


\chapter{Discretization}
\label{discretization_chapter}

Although the continuous-time algorithm provides clean theoretical guarantees, practical deployment requires discretization. Discretization inevitably introduces approximation errors and typically necessitates additional conditions to preserve the qualitative behavior of the continuous dynamics. To address this, I develop three discretization schemes for Algorithm~\ref{alg:equa}, each accompanied by an additional assumption tailored to the corresponding discretization method. Unlike the continuous counterpart, beyond the standard Assumptions~\ref{ass:stronglymonotone}--\ref{ass:KKT_Eq}, each scheme requires an additional condition to ensure convergence. As a result, the choice of discretization scheme usually depends on the properties of the problem.

In this section, I focus exclusively on the discretization of the equality shared-constraint setting with individual constraints. For clarity, I recall that the problem of interest is the following:
\begin{equation}
    \begin{split}
        \begin{matrix*}[l]
            \min\limits_{z_v} & f_v(z_v,\, \z_{-v}^f) \\[2pt]
            \text{s.t.} & h_v(z_v) \le \mathbf{0}_{q_v}, \\[2pt]
            & \psi_v(z_v,\, \z_{-v}^\psi) = \mathbf{0}_{l_v},
        \end{matrix*}
    \end{split}
\end{equation}
where $f_v(\cdot)$ is the local cost function of agent~$v$, $h_v(\cdot)$ represents individual inequality constraints, and $\phi_v(\cdot)$ denotes the shared equality constraints that couple the agents’ decisions according to the global structure in \eqref{equ:linear_eq}, i.e.,
\begin{align*}
\psi_v(z_v,, \z_{-v}) = A(l_v, n_v)\, z_v + \sum_{u \in \tilde{\mathcal{N}}_\psi(v)} A(l_v, n_u)\, z_u - \c(l_v).
\end{align*}

\begin{assumption}[Additional Assumptions]\label{ass:discrete}
The discretized algorithms require one or more of the following assumptions:
\begin{enumerate}
    \item The pseudo-gradient in~(\ref{equ:pseudo_grad}) is globally Lipschitz with constant $L_f$. 
    \item For every agent $v$ and any $z_v, z_v' \in \mathbb{R}^{n_v}$,
    \begin{equation*}
        \|h_v(z_v) - h_v(z_v')\|^2 \le L_v \|z_v - z_v'\|^2.
    \end{equation*}
    \item The individual constraint set $\Z_v := \{z_v \,|\, h_v(z_v) \le \mathbf{0}\}$ is compact for all agents.
\end{enumerate}
\end{assumption}

\section{Preliminaries}
Before introducing the discretized algorithms, I first outline the discretization schemes and the corresponding convergence results. I consider two types of discretization: (i) diminishing step sizes, which are commonly used to mimic the behavior of continuous-time dynamics in the limit, and (ii) Euler-type discretizations with constant step size, including both forward Euler and backward Euler schemes.

\subsection{Notation}
We operate in a finite-dimensional Euclidean space $\mathcal{X}$ equipped with the standard inner product $\langle \cdot, \cdot \rangle$ and induced norm $\|\cdot\|$. For a symmetric positive definite matrix $H \succ 0$, we define the $H$-weighted norm as $\|x\|_H := \sqrt{\langle x, Hx \rangle}$ and denote the minimum eigenvalue of $H$ by $\rho_{\min}(H)$. We denote the identity operator by $I$. An operator $\Phi: \mathcal{X} \to 2^{\mathcal{X}}$ is said to be \emph{monotone} if $\langle u - v, x - y \rangle \ge 0$ for all pairs $(x, u)$ and $(y, v)$ in the graph of $\Phi$, while a single-valued operator $\Psi: \mathcal{X} \to \mathcal{X}$ is called \emph{$\beta$-cocoercive} (where $\beta > 0$) if it satisfies $\langle \Psi(x) - \Psi(y), x - y \rangle \ge \beta \|\Psi(x) - \Psi(y)\|^2$ for all $x, y \in \mathcal{X}$. The set of zeros of an operator $\Phi$ is denoted by $\mathrm{zer}(\Phi) := \{x \in \mathcal{X} \mid 0 \in \Phi(x)\}$, and the set of fixed points of an operator $T$ is denoted by $\mathrm{Fix}(T) := \{x \in \mathcal{X} \mid x = T(x)\}$. We define the resolvent of $\Phi$ as $J_\Phi := (I + \Phi)^{-1}$. Finally, for a nonempty closed convex set $C$, $\N_C(x)$ denotes the normal cone to $C$ at $x$, defined as $\{u \in \mathcal{X} \mid \langle u, z - x \rangle \le 0, \forall z \in C\}$ if $x \in C$, and $\emptyset$ otherwise.

\subsection{Diminishing Step Size}
Diminishing step size is often employed to ensure that a discrete-time algorithm asymptotically mimics the behavior of its continuous-time counterpart. However, this does not imply that an arbitrary continuous-time dynamics can be discretized with diminishing step sizes while still guaranteeing convergence. To establish convergence under this scheme, I introduce a deterministic version of the almost-supermartingale convergence lemma from \cite{ROBBINS1971233}.

\begin{lemma}\label{lem:converging}
    \textbf{(Convergence of Non-Negative Almost-Supermartingales).} 
    Let $\{x^k\}$, $\{\gamma^k\}$, $\{a^k\}$, and $\{s^k\}$ be non-negative sequences satisfying the following recursive inequality for all $k$:
    \begin{equation*}
        x^{k+1} \le (1 + \gamma^k) \, x^k + a^k - s^k.
    \end{equation*}
    If the accumulated perturbations are finite, i.e.,
    \begin{equation*}
        \sum_{k=1}^{\infty} \gamma^k = \Gamma < \infty \quad \text{and} \quad \sum_{k=1}^\infty a^k = \mathcal A < \infty,
    \end{equation*}
    then:
    \begin{enumerate}[(a)]
        \item The sequence $\{x^k\}$ converges to a finite value: $\lim_{k \to \infty} x^k = X_0 \ge 0$.
        \item The sum of the subtracted terms converges: $\sum_{k=1}^\infty s^k = S_0 < \infty$.
    \end{enumerate}
\end{lemma}

\renewcommand{\proofname}{Proof of Lemma \ref{lem:converging}}
\begin{proof}
    This proof adapts the supermartingale convergence theorem established in \cite{ROBBINS1971233} to the deterministic context.
    
    \noindent\textbf{1. Transformation of Variables}
    
    Let us define the integrating factor $P_k \coloneqq \prod_{n=1}^{k} (1 + \gamma^n)^{-1}$ for $k \ge 1$, with $P_0 = 1$. We introduce the scaled sequences:
    \begin{equation}\label{equ:scaled_vars}
        (x^k)' \coloneqq x^k P_{k-1}, \quad (a^k)' \coloneqq a^k P_k, \quad (s^k)' \coloneqq s^k P_k.
    \end{equation}
    Multiplying the original inequality $x^{k+1} \le (1 + \gamma^k) x^k + a^k - s^k$ by $P_k$, we obtain:
    \begin{equation*}
        x^{k+1} P_k \le (1 + \gamma^k) P_k x^k + a^k P_k - s^k P_k.
    \end{equation*}
    Observing that $(1+\gamma^k)P_k = P_{k-1}$, this simplifies to the recursive relation in terms of the primed variables:
    \begin{equation}\label{equ:recur_seq}
        (x^{k+1})' \le (x^k)' + (a^k)' - (s^k)'.
    \end{equation}

    \noindent\textbf{2. Construction of a Monotone Sequence}
    
    Define the auxiliary sequence $u^k$ as:
    \begin{equation*}
        u^k \coloneqq (x^k)' - \sum_{n=1}^{k-1} (a^n)' + \sum_{n=1}^{k-1} (s^n)'.
    \end{equation*}
    Using the inequality \eqref{equ:recur_seq}, we examine the difference between consecutive terms:
    \begin{equation*}
        \begin{split}
            u^{k+1} - u^k &= (x^{k+1})' - (x^k)' - (a^k)' + (s^k)' \\
            &\le \left[ (x^k)' + (a^k)' - (s^k)' \right] - (x^k)' - (a^k)' + (s^k)' = 0.
        \end{split}
    \end{equation*}
    Thus, $u^k$ is a non-increasing sequence ($u^{k+1} \le u^k$). To verify convergence, we must show it is bounded from below. Since $x^k, a^k, s^k \ge 0$ and $P_k > 0$, we have $(x^k)' \ge 0$ and $\sum (s^n)' \ge 0$. Furthermore, noting that $P_k \le 1$, we have $(a^n)' \le a^n$. Therefore:
    \begin{equation*}
        u^k \ge - \sum_{n=1}^{k-1} (a^n)' \ge - \sum_{n=1}^{\infty} a^n = -\mathcal A > -\infty.
    \end{equation*}
    Since $u^k$ is monotonic and bounded, it converges to a finite limit $U_0$.

    \noindent\textbf{3. Convergence of Component Sequences}
    
    The convergence of $u^k$ implies the convergence of its components. Specifically, rearranging the definition of limit $U_0$:
    \begin{equation*}
        \lim_{k \to \infty} \left[ (x^k)' + \sum_{n=1}^{k-1} (s^n)' \right] = U_0 + \sum_{n=1}^\infty (a^n)'.
    \end{equation*}
    Since both terms on the LHS are non-negative, they must individually converge:
    \begin{equation}\label{equ:primed_limits}
        \lim_{k \to \infty} (x^k)' = X'_0 < \infty \quad \text{and} \quad \sum_{k=1}^\infty (s^k)' = S'_0 < \infty.
    \end{equation}

    \noindent\textbf{4. Recovery of Original Limits}
    
    We now relate the primed limits back to the original sequences. Consider the product term $P_k$. Since $\gamma^k \ge 0$ and $\sum \gamma^k < \infty$, the infinite product converges to a strictly positive constant:
    \begin{equation*}
        \lim_{k \to \infty} P_k^{-1} = \prod_{n=1}^\infty (1 + \gamma^n) = \exp\left(\sum_{n=1}^\infty \ln(1+\gamma^n)\right) \le \exp(\Gamma) < \infty.
    \end{equation*}
    Let $P_\infty = \lim_{k\to\infty} P_k$. Crucially, $0 < P_\infty \le 1$.
    
    For the sequence $x^k$:
    \begin{equation*}
        \lim_{k \to \infty} x^k = \lim_{k \to \infty} \frac{(x^k)'}{P_{k-1}} = \frac{X'_0}{P_\infty} \eqqcolon X_0 < \infty.
    \end{equation*}
    For the sum of $s^k$:
    Since $P_k$ is non-increasing, $P_k \ge P_\infty > 0$ for all $k$. Thus, $s^k = (s^k)' / P_k \le (s^k)' / P_\infty$.
    \begin{equation*}
        \sum_{k=1}^\infty s^k \le \frac{1}{P_\infty} \sum_{k=1}^\infty (s^k)' = \frac{S'_0}{P_\infty} < \infty.
    \end{equation*}
    This implies $\sum s^k$ converges to some finite value $S_0$. The proof is complete.
\end{proof}
Lemma \ref{lem:converging} shows that the sequence converges asymptotically to its limit. However, with a proper schedule for the learning rate, the convergence rate can be improved.

\subsection{Backward Euler Discretization (Resolvent)}
For discretizations with a constant step size, the most common approaches are the forward Euler and backward Euler schemes, where the backward Euler scheme is also known as the resolvent method. Similar in spirit to Runge--Kutta methods, hybrid forward--backward schemes are also used in practice to achieve better numerical behavior. In this subsection, I first introduce the convergence theorem associated with the backward Euler discretization.

The presentation in \cite{ryu-large-scale-2022} focuses on the special case where the backward Euler (resolvent) and forward–backward updates use an identity step matrix. For completeness, I extend their analysis by deriving the convergence of both the resolvent and the forward–backward discretizations under a more general $H$-step.

Let $H\in\R^{n\times n}$ be symmetric positive definite. Define
\begin{equation*}
\langle x,y\rangle_H := \langle Hx,y\rangle,\qquad
\|x\|_H := \sqrt{\langle x,x\rangle_H}.
\end{equation*}
Let $\Phi:\R^n\to\mathbb{R}^n$ be (Euclidean) monotone: $\langle u-v,\,\Phi(u)-\Phi(v)\rangle \ge 0$ for all $u,v$. Define the $H$-metric resolvent and reflector of $\Phi$ by
\begin{equation*}
    J := (I+H^{-1}\Phi)^{-1},\qquad R:=2 \, J-I.
\end{equation*}
Equivalently, $u=J(x)$ if and only if $x-u=H^{-1}\Phi(u)$, and $R(x)=2 \, u-x$. We begin by establishing that monotonicity is preserved under the $H$-step transformation.

\begin{lemma}[Monotonicity transport to the H–metric]
\label{lem:H-mono}
For all $u,v$,
\begin{equation*}
\langle u-v,\,H^{-1}\,\Phi(u)-H^{-1} \, \Phi(v)\rangle_H
= \langle u-v,\,\Phi(u)-\Phi(v)\rangle \ge 0.
\end{equation*}
\end{lemma}

\begin{proof}
By definition of $\langle\cdot,\cdot\rangle_H$,
\begin{equation*}
    \begin{split}
        \langle u-v,\,H^{-1}(\Phi(u)-\Phi(v))\rangle_H
&=\left \langle H \, (u-v),\,H^{-1}(\Phi(u)-\Phi(v)) \right \rangle \\
&=\langle u-v,\,\Phi(u)-\Phi(v)\rangle.
    \end{split}
\end{equation*}
The last quantity is nonnegative by monotonicity of $\Phi$.
\end{proof}

\begin{proposition}[Firm nonexpansiveness of the H–resolvent]
\label{prop:fne}
The mapping $J=(I+H^{-1} \, \Phi)^{-1}$ is firmly nonexpansive in $\|\cdot\|_H$, i.e.,
for all $x,y$,
\begin{equation*}
\|J \, x-J \, y\|_H^2 \;\le\; \langle J \, x-J \, y,\; x-y\rangle_H.
\end{equation*}
Equivalently, the reflector $R=2 \, J-I$ is nonexpansive in $\|\cdot\|_H$:
\begin{equation*}
    \|R(x)-R(y)\|_H \le \|x-y\|_H
\end{equation*}
for all $x$, $y$.
\end{proposition}
\renewcommand{\proofname}{Proof of Proposition \ref{prop:fne}}
\begin{proof}
Let $u=J \, x$, $v=J \, y$. Then $x-u=H^{-1} \, \Phi(u)$ and $y-v=H^{-1} \, \Phi(v)$.
Hence
\begin{equation*}
    \langle u-v,\;x-y-(u-v)\rangle_H
=\langle u-v,\;H^{-1} \, \Phi(u)-H^{-1} \, \Phi(v)\rangle_H
\stackrel{\text{Lemma }\ref{lem:H-mono}}{\ge}0.
\end{equation*}
Expanding the left side gives $\langle u-v,\,x-y\rangle_H - \|u-v\|_H^2 \ge 0$, which is the firm nonexpansiveness inequality.

The nonexpansiveness of $R=2 \, J-I$ follows from the standard FNE–reflector equivalence:
for firmly nonexpansive $J$,
\begin{equation*}
    \|R(x)-R(y)\|_H^2=\|x-y\|_H^2-4\| (I-J) \, x-(I-J) \, y\|_H^2 \le \|x-y\|_H^2.
\end{equation*}
\end{proof}

\begin{lemma}[Quadratic identity for convex combinations in the H–norm]
\label{lem:convex-quad}
For all $x,y\in\mathbb{R}^n$ and $t\in[0,1]$,
\begin{equation*}
\|(1-t) \, x+t \, y\|_H^2
=(1-t) \, \|x\|_H^2+t \, \|y\|_H^2 - t \, (1-t) \,\|x-y\|_H^2.
\end{equation*}
\end{lemma}
\renewcommand{\proofname}{Proof of Lemma \ref{lem:convex-quad}}
\begin{proof}
This is the identity in the inner product space $(\R^n,\langle\cdot,\cdot\rangle_H)$.
\end{proof}

\begin{theorem}[Resolvent is $\tfrac12$–averaged and a descent for KM]
\label{thm:resolvent-averaged}
The resolvent $J$ is $\tfrac12$-averaged in $\|\cdot\|_H$, hence for any fixed point $x^\star\in\mathrm{Fix}(J)$ and the Picard iteration $x^{k+1}=J(x^k)$,
\begin{equation}\label{eq:KM-descent}
\left \|x^{k+1}-x^\star \right\|_H^2 \;\le\;
\left \|x^k-x^\star \right \|_H^2 \;-\; \tfrac14\, \left \|R(x^k)-x^k \right\|_H^2.
\end{equation}
In particular, $\{\|x^k-x^\star\|_H\}$ is nonincreasing and
$\sum_k \|R(x^k)-x^k\|_H^2 < \infty$.
\end{theorem}
\renewcommand{\proofname}{Proof of Theorem \ref{thm:resolvent-averaged}}
\begin{proof}
Firm nonexpansiveness of $J$ implies that $J$ is $\tfrac12$-averaged:
$J=\tfrac12 \, I + \tfrac12 \, R$ with $R$ nonexpansive (Proposition \ref{prop:fne}).
Fix $x^\star\in\mathrm{Fix}(J)$; then also $R(x^\star)=x^\star$.
Write $x^{k+1}=\tfrac12 \, R(x^k)+\tfrac12 \, x^k$.

Apply Lemma \ref{lem:convex-quad} with $t=\tfrac12$, $x=R(x^k)-x^\star$, $y=x^k-x^\star$:
\begin{equation*}
    \left \|x^{k+1}-x^\star \right\|_H^2
=\tfrac12 \, \left \|R(x^k)-x^\star \right \|_H^2+\tfrac12 \, \left \|x^k-x^\star \right\|_H^2-\tfrac14 \, \left \|R(x^k)-x^k \right\|_H^2.
\end{equation*}
By nonexpansiveness of $R$ and $R(x^\star)=x^\star$, $\|R(x^k)-x^\star\|_H \le \|x^k-x^\star\|_H$.

Insert this bound to obtain \eqref{eq:KM-descent}. The summability follows by telescoping the resulting series.
\end{proof}

\begin{lemma}[Forward step nonexpansiveness under cocoercivity and H]
\label{lem:forward}
Let $\Psi:\R^n\to\R^n$ be \(\beta\)-cocoercive in the Euclidean metric:
$\langle \Psi(x)-\Psi(y),\,x-y\rangle \ge \beta \, \|\Psi(x)-\Psi(y)\|^2$.
Then the mapping $\mathcal F:=I-H^{-1} \, \Psi$ is nonexpansive in $\|\cdot\|_H$ provided
\begin{equation*}
    2 \, \beta\,\rho_{\min}(H)\;\ge\;1,
\end{equation*}
where $\rho_{\min}(H)$ is the smallest eigenvalue of $H$.
\end{lemma}
\renewcommand{\proofname}{Proof of Lemma \ref{lem:forward}}
\begin{proof}
For any \(x,y\),
\begin{align*}
\|\mathcal F(x)-\mathcal F(y)\|_H^2
&= \|x-y\|_H^2 - 2\langle x-y,\,H^{-1}(\Psi(x)-\Psi(y))\rangle_H
+ \|H^{-1}(\Psi(x)-\Psi(y))\|_H^2\\
&= \|x-y\|_H^2 - 2\langle \Psi(x)-\Psi(y),\,x-y\rangle
+ \|\Psi(x)-\Psi(y)\|_{H^{-1}}^2.
\end{align*}
By cocoercivity and the spectral bound \(\|z\|_{H^{-1}}^2 \le \rho_{\min}(H)^{-1}\|z\|^2\),
\begin{equation*}
    \begin{split}
- 2\langle \Psi(x)-\Psi(y),\,x-y\rangle + \|\Psi(x)-\Psi(y)\|_{H^{-1}}^2
\le (-2 \, \beta + \rho_{\min}(H)^{-1})\|\Psi(x)-\Psi(y)\|^2.
    \end{split}
\end{equation*}
Thus 
\begin{equation*}
    \|\mathcal F(x)-\mathcal F(y)\|_H^2 \le \|x-y\|_H^2
\end{equation*} whenever $-2 \, \beta+\rho_{\min}(H)^{-1}\le 0$, which is $2 \, \beta\,\rho_{\min}(H)\ge 1$.
\end{proof}

\begin{theorem}
\label{thm:fb}
Let $\Psi$ be $\beta$-cocoercive (Euclidean), $\Phi$ monotone, and assume
$\mathrm{zer}(\Psi+\Phi)\neq\emptyset$.
Define the forward–backward operator
\begin{equation*}
T_{\mathrm{FB}} := J\circ \mathcal F
\quad\text{with}\quad
\mathcal F=I-H^{-1}\Psi,\quad J=(I+H^{-1}\Phi)^{-1}.
\end{equation*}
If 
\begin{equation*}
    2 \, \beta\,\rho_{\min}(H)>1,
\end{equation*}
then $F$ is nonexpansive in $\|\cdot\|_H$
(Lemma \ref{lem:forward}), $J$ is $\tfrac12$-averaged
(Theorem \ref{thm:resolvent-averaged}), and hence
$T_{\mathrm{FB}}$ is $\tfrac12$-averaged in $\|\cdot\|_H$.

Therefore, for the iteration $x_{k+1}=T_{\mathrm{FB}}(x_k)$ and any fixed point $x_\star\in\mathrm{Fix}(T_{\mathrm{FB}})$ (which corresponds to a zero of $\Psi+\Phi$),
\begin{equation}\label{eq:fb-descent}
\left \|x^{k+1}-x^\star \right\|_H^2
\;\le\;
\left\|\mathcal F(x^k)-\mathcal F(x^\star) \right\|_H^2 - \tfrac14 \, \left \|R \left (\mathcal F(x^k) \right )-\mathcal F(x^k) \right\|_H^2
\;\le\;
\left \|x^k-x^\star \right\|_H^2 - \Gamma^k,
\end{equation}
where
\begin{equation*}
    \Gamma^k
:= \bigl(2 \, \beta\,\rho_{\min}(H)-1\bigr)\, \left \|H^{-1}(\Psi(x^k)-\Psi(x^\star)) \right\|_H^2
+ \tfrac14 \, \left \|R\left (\mathcal F (x^k) \right)-\mathcal F(x^k) \right\|_H^2 \;\ge\;0.
\end{equation*}
Consequently, $\{\|x^k-x^\star\|_H\}$ is nonincreasing and $\sum_k \Gamma^k<\infty$.
\end{theorem}
\renewcommand{\proofname}{Proof of Theorem \ref{thm:fb}}
\begin{proof}
By Lemma \ref{lem:forward}, $\mathcal F$ is nonexpansive in $\|\cdot\|_H$ under
$2 \,\beta\,\rho_{\min} (H)\ge 1$, strict $>$ yields the strict decrease term below.

By Theorem \ref{thm:resolvent-averaged} applied at the point $\mathcal F(x^k)$,
\begin{equation*}
    \left \|J \left (\mathcal F(x^k) \right )-x^\star \right\|_H^2
\le \left \|\mathcal F(x^k)-x^\star \right\|_H^2 - \tfrac14 \, \left \|R\left (\mathcal F(x^k) \right)-\mathcal F(x^k) \right\|_H^2.
\end{equation*}
Since $x^\star\in\mathrm{Fix}(T_{\mathrm{FB}})$ implies $\mathcal F(x^\star)\in\mathrm{Fix}(J)$,
we may write $x^\star=J\left (\mathcal F(x^\star) \right)$, and use nonexpansiveness of $\mathcal F$ with the calculation in the proof of Lemma \ref{lem:forward} to bound
\begin{equation*}
    \left \|F(x^k)-x^\star \right\|_H^2 \le
\left \|x^k-x^\star \right \|_H^2 - \bigl(2 \,\beta\,\rho_{\min}(H)-1\bigr)\,
\left \|H^{-1}(\Psi(x^k)-\Psi(x^\star)) \right\|_H^2.
\end{equation*}
Combine the two displays to obtain \eqref{eq:fb-descent}.
\end{proof}
\begin{remark}
When $H=I$, these results reduce to the standard Euclidean statements: the resolvent of a monotone operator is firmly nonexpansive, and the forward–backward operator with $\tau\le 2 \, \beta$ is averaged, hence convergent.
\end{remark}

\subsection{Convergence Rate}
In the analysis of discretized algorithms, determining the convergence rate of a chosen scheme is a fundamental question. However, comparing these rates is not always straightforward, as the interpretation depends heavily on the specific asymptotic notation and error measures employed. To standardize this comparison, Table \ref{tab:asymptotic_not} summarizes common asymptotic notations, while Table \ref{tab:conv_measures} lists typical error metrics.

To quantify convergence, let $\tau^k > 0$ denote a vanishing error measure (such as $\tau^k = \|\nabla f(x^k)\|$ or $\|x^k - x^\star\|$, as detailed in Table \ref{tab:conv_measures}). The practical goal is to determine the smallest iteration count $k$ required to satisfy $\tau^k \le \varepsilon$. Table \ref{tab:rate_to_iters} relates these standard convergence bounds to their corresponding iteration complexity.

\noindent\textbf{1. Diminishing Step Size.}

While constant step sizes offer faster convergence for well-conditioned problems (e.g., those involving smooth or cocoercive operators), many practical applications lack these favorable properties. In non-smooth convex optimization, the subgradient does not naturally vanish near the solution, causing constant step size algorithms to oscillate rather than converge. Similarly, for general monotone inclusions where the operator acts as a rotation (lacking cocoercivity), simple forward steps can lead to divergence. In these robust settings, one must employ a \textbf{diminishing step size schedule} (e.g., $\alpha^k \propto 1/\sqrt{k}$) to dampen oscillations and force convergence. However, this robustness comes at the cost of speed: the convergence rate typically drops from $O(1/k)$ to $O(1/\sqrt{k})$.

The following theorems formalize these rates for non-smooth optimization and general monotone variational inequalities, respectively.

\begin{theorem}[Non-Smooth Convex Optimization \cite{Nesterov2004}]
\label{thm:nonsmooth-rate}
Let $f: Q \to \mathbb{R}$ be a convex function defined on a closed convex set $Q$, and assume $f$ is Lipschitz continuous with constant $L$ (i.e., $\|\partial f(x)\| \le L$ for all $x \in Q$). Let $R$ be a known upper bound on the distance from the initial point $x_0$ to the optimal set, satisfying:
\begin{equation*}
    \|x^0 - x^\star\| \le R.
\end{equation*}
Consider the Projected Subgradient Method iteration:
\begin{equation*}
    x^{k+1} = P_Q \left [ x^k - \alpha^k \, g^k \right ], \quad \text{where } g^k \in \partial f(x^k).
\end{equation*}
If the step size is chosen as $\alpha_k = \frac{R}{L\sqrt{k+1}}$, then the convergence rate of the best objective value found after $k$ iterations is given by:
\begin{equation*}
    \min_{0 \le i \le k} f(x_i) - f(x_\star) \le \frac{L \, R}{\sqrt{k+1}} = O\left(\frac{1}{\sqrt{k}}\right).
\end{equation*}
\end{theorem}

\begin{theorem}[General Monotone Operators \cite{Nemirovski2009}]
\label{thm:monotone-gap-rate}
Let $T: Z \to \mathbb{R}^n$ be a monotone and bounded operator on a convex set $Z$ (i.e., $\|T(z)\| \le M$). Consider the recurrence $z^{k+1} = P_Z \left [ z^k - \gamma^k T(z^k) \right]$ using a diminishing step size $\gamma^k \propto \frac{1}{\sqrt{k}}$. The convergence is measured using the \emph{ergodic average} of the iterates, $\hat{z}^k = \frac{\sum_{t=1}^k \gamma^t z^t}{\sum_{t=1}^k \gamma^t}$. The restricted gap function for $\hat{z}^k$ satisfies:
\begin{equation*}
    \epsilon_{\mathrm{gap}}(\hat{z}^k) := \sup_{u \in Z} \langle T(u), \hat{z}^k - u \rangle \le \frac{D_Z^2 + M^2 \sum_{t=1}^k (\gamma^t)^2}{2 \sum_{t=1}^k \gamma^t} = O\left(\frac{1}{\sqrt{k}}\right),
\end{equation*}
where $D_Z$ is the diameter of the set $Z$.
\end{theorem}

\noindent\textbf{2. Backward Euler (Resolvent) and Averaged Operators.}

Within this framework, the convergence properties of splitting algorithms—specifically Backward Euler (resolvent methods) and Forward-Backward splitting—have been studied extensively. For general monotone inclusions, the current tightest convergence rate for the fixed-point residual was established by Davis and Yin \cite{Davis2016}. Their work demonstrates that if $T$ is a general averaged operator, the squared fixed-point residual decays at a rate of $o(1/k)$. This result,
\begin{equation*}
\left \| T(x^k) - x^k \right\|^2 = o\left(\frac{1}{k}\right),
\end{equation*}
represents a strict improvement over the previously known rate of $O(1/k)$.

The analysis changes, however, when the operator $T$ is derived from the gradient of a convex function. In this context of convex optimization, the focus typically shifts from operator residuals to the objective function value. Here, the standard convergence rate is given by
\begin{equation*}
f(x^k) - f(x^\star) = O\left(\frac{1}{k}\right),
\end{equation*}
where $x^\star$ is a minimizer. It is worth noting that while $O(1/k)$ is standard for non-accelerated schemes, accelerated methods can further improve this rate to $O(1/k^2)$.

\begin{theorem}[Convergence Rate of Residuals \cite{Davis2016}]
\label{thm:residual-rate}
Let $T$ be an $\alpha$-averaged operator with $\alpha \in (0, 1)$ and $\mathrm{Fix}(T) \neq \emptyset$. Consider the iteration $x^{k+1} = T(x^k)$. The squared fixed-point residual converges at a rate of $o(1/k)$. That is:
\begin{equation*}
\left \| T(x^k) - x^k \right\|^2 = o\left(\frac{1}{k}\right).
\end{equation*}
\end{theorem}

\begin{table}[ht!]
\centering
\small
\renewcommand{\arraystretch}{1.25}
\setlength{\tabcolsep}{5pt}
\caption{Asymptotic Notations}
\label{tab:asymptotic_not}
\setlength{\tabcolsep}{4pt}
\begin{tabular}{|c|p{3.2cm}|p{4.1cm}|p{2.8cm}|}
\hline
\textbf{Symbol} 
& \textbf{Formal Definition} 
& \textbf{Meaning} 
& \textbf{Example ($g(k)=k^2$)} \\ 
\hline
$f(k)=O(g(k))$
& $\exists\,C,K>0$ s.t. $|f(k)|\le C|g(k)|$ for $k>K$
& $f$ grows no faster than $g$
& $3 \, k^2+2 \, k+1 = O(k^2)$ \\ 
\hline
$f(k)=o(g(k))$
& $\lim_{k\to\infty}\! \frac{f(k)}{g(k)}=0$
& $f$ grows strictly slower than $g$
& $k=o(k^2)$ \\ 
\hline
$f(k)=\Omega(g(k))$
& $\exists\,C,K>0$ s.t. $|f(k)|\ge C|g(k)|$ for $k>K$
& $f$ grows no slower than $g$
& $k^3=\Omega(k^2)$ \\ 
\hline
$f(k)=\omega(g(k))$
& $\lim_{k\to\infty}\! \frac{f(k)}{g(k)}=\infty$
& $f$ grows strictly faster than $g$
& $k^3=\omega(k^2)$ \\ 
\hline
$f(k)=\Theta(g(k))$
& $\exists\,C_1,C_2,K>0$ s.t. $C_1 g(k)\!\le\! f(k)\!\le\!C_2 g(k)$
& Same growth rate (tight bound)
& $k^2=\Theta(k^2)$ \\ 
\hline
\end{tabular}
\end{table}

\begin{table}[ht!]
\centering
\small
\renewcommand{\arraystretch}{1.25}
\setlength{\tabcolsep}{4pt}
\caption{Summary of common convergence measures}
\label{tab:conv_measures}
\begin{tabular}{|l|p{4cm}|p{5.7cm}|}
\hline
\textbf{Measure} & \textbf{Definition} & \textbf{Interpretation} \\
\hline

\textbf{Residual} $\|r^k\|$
& $\|T (x^{k}) - x^{k}\|$
& Measures violation of the fixed-point condition $x = T(x)$. \\

\hline
\textbf{Ergodic Residual}
& $\displaystyle \frac{1}{k+1}\sum_{i=0}^k \|r^{i}\|^{2}$
& Average squared residual. Used for sublinear $O(1/k)$ ergodic convergence guarantees. \\

\hline
\textbf{Best-iterate residual}
& $\displaystyle \min_{0\le i\le k} \|r^{i}\|$
& “Best so far” convergence indicator; follows from the ergodic bound via Jensen's inequality. \\

\hline
\textbf{Distance to solution}
& $\|x^{k} - x^{\star}\|$
& Measures closeness to a solution. Only qualitative unless strong monotonicity or contractivity holds. \\

\hline
\textbf{Gap residual}
& $T(x^{k}) - T(x^{\star})$
& Suboptimality or primal–dual gap; standard in convex optimization and variational inequalities. \\

\hline
\textbf{Function error}
& $f(x^{k}) - f^{\star}$
& Measures $\varepsilon$-optimality based on objective value. \\

\hline
\textbf{First-order condition}
& $\|\nabla f(x^{k})\|$
& Measures $\varepsilon$-stationarity (gradient nearly zero). \\

\hline
\end{tabular}
\end{table}

\begin{table}[ht!]
\centering
\small
\renewcommand{\arraystretch}{1.25}
\setlength{\tabcolsep}{5pt}
\caption{Typical decay bounds and corresponding iteration complexity}
\label{tab:rate_to_iters}
\begin{tabular}{|p{4.7cm}|p{4.7cm}|}
\hline
\textbf{Decay Bound on $\tau^k$} 
& \textbf{Iterations for $\tau^k \le \varepsilon$} \\ 
\hline

$\tau^k \le \dfrac{\mathcal A}{k}$ 
& $O\!\left(\dfrac{1}{\varepsilon}\right)$ \\

\hline
$\tau^k \le \dfrac{\mathcal A}{k^2}$ 
& $O\!\left(\dfrac{1}{\sqrt{\varepsilon}}\right)$ \\

\hline
$\tau^k \le \dfrac{\mathcal A}{\sqrt{k}}$ 
& $O\!\left(\dfrac{1}{\varepsilon^{2}}\right)$ \\

\hline
$\tau^k \le (1-\alpha)^k \, \tau^0$ 
& $O\!\left(\log \tfrac{1}{\varepsilon}\right)$ \\

\hline
$\tau^k \le \beta^{2^{k-1}} \, \tau^0$ 
& $O\!\left(\log\log \tfrac{1}{\varepsilon}\right)$ \\

\hline
\end{tabular}
\end{table}

\section{Discretized Algorithms Under Different Schemes}

I now proceed to discretize our continuous-time algorithm to facilitate its application in real-world scenarios. We first describe an explicit algorithm utilizing diminishing step sizes. While Theorems \ref{thm:nonsmooth-rate} and \ref{thm:monotone-gap-rate} provide elegant convergence guarantees and rates for such methods, they typically assume the iterates are projected onto a bounded, closed convex set. As this assumption does not hold in our specific discretization scheme, I cannot rely on standard projection-based analysis. Instead, I prove convergence by utilizing the deterministic counterpart of the martingale convergence theorem provided in Lemma \ref{lem:converging}.

Although this direct discretization is intuitive and closely mirrors the continuous dynamics, its convergence guarantees are primarily asymptotic (ensuring convergence in the limit) without a strong rate. To address this limitation, we subsequently propose a Backward Euler discretization utilizing a fixed step size, which achieves a superior theoretical convergence rate.

\subsection{Diminishing Step Size}
Based on the work in \cite{ebrahimi-robust-2019}, Algorithm \ref{alg:disc_equa} established a discrete version of the continuous algorithm using a standard diminishing step size. I can choose a step size $\alpha^k$, where the superscript $k$ indicates the step at the $k$-th iteration, that satisfies:
\begin{equation}\label{equ:disminishing_step_size}
    \begin{split}
        \sum_{k=1}^\infty \alpha^k = \infty, \qquad  \sum_{k=1}^\infty (\alpha^k)^2 = \Gamma < \infty.
    \end{split}
\end{equation}
It is important to highlight that this step size is independent of any agent's private data, meaning it can be selected without requiring access to agents' private information. 
\begin{algorithm}
\begin{algorithmic}[1]
\STATE {\bfseries Input:}Initial $z_v^0$, $\lambda_v^0$, $\mu_v^0 \geq 0$, $\alpha^k$.
\FOR{$t=1, 2, \dots, T$}
\STATE Receive the decision variables $\z_{-v}$ from other agents through the cost-function graph $\mathcal G_f$ and the shared constraint graph $\mathcal G_\psi$.
\STATE $z_v^{+} \leftarrow z_v-\alpha^t \cdot \nabla_{z_v} L_v(\z, \lambda_v, \mu_v)$.
\STATE $\lambda_{v} \leftarrow \lambda_v + \alpha^t \cdot \psi_v(z_v, \z_{-v})$.
\STATE $\mu_v^j \leftarrow P_{\ge0} \left[ \mu_v^j + \alpha^t \cdot h_v^j(z_v) \right ]$, for $j \in [q_v]$.
\STATE $z_v \leftarrow z_v^+$.
\ENDFOR
\end{algorithmic}
\caption{Forward Discretization with Diminishing Step Size for Each Agent $v$}
\label{alg:disc_equa}	
\end{algorithm}

With the additional assumptions that the pseudo-gradient is globally Lipschitz and each individual constraint is Lipschitz, as stated in Assumption~\ref{ass:discrete} (1) and (2), Lemmas~\ref{lem:converging} and~\ref{lem:monoton} together imply Theorem~\ref{thm:discrete}, guaranteeing the convergence of Algorithm~\ref{alg:disc_equa}.
\begin{theorem}\label{thm:discrete}
With a diminishing step size as in (\ref{equ:disminishing_step_size}), and under the additional Assumptions~\ref{ass:discrete} (1) and (2), the sequence ${\z^k}$ generated by Algorithm~\ref{alg:disc_equa} converges to a fixed point $\z^\star$, which is a GNE.
\end{theorem}
\renewcommand{\proofname}{Proof of Theorem \ref{thm:discrete}}
\begin{proof}
Let $(z_v^\star, \mu_v^\star, \lambda^\star)$ be a KKT point, which is also a fixed point of the algorithm.  
Define the sequence
\begin{equation}
    x^{k+1}
    := \sum_{v=1}^N \left [ \| z_v^{k+1} - z_v^\star\|^2 + \|\mu_v^{k+1} - \mu_v^\star\|^2 \right]
    + \| \lambda^{k+1} - \lambda^\star \|^2.
\end{equation}
Applying the update rule gives
\begin{equation}\label{equ:main_sequence}
\begin{split}
    x^{k+1}
    \le\; &\sum_{v=1}^N \left [ \| z_v^{k} - z_v^\star\|^2 + \|\mu_v^{k} - \mu_v^\star\|^2 \right ]
            + \| \lambda^{k} - \lambda^\star\|^2 \\
        &+ (\alpha^k)^2 \sum_{v=1}^N 
            \underbrace{\| \nabla_{z_v} L_v(\z^k, \lambda_v^k, \mu_v^k)
            - \nabla_{z_v} L_v(\z^\star, \lambda_v^\star, \mu_v^\star)\|^2}_{(i)} \\
        &+ (\alpha^k)^2 \sum_{v=1}^N \| h_v(z_v^k) - h_v(z_v^\star)\|^2
        + (\alpha^k)^2 \|A\|^2 \| \lambda^k - \lambda^\star\|^2 \\
        &- 2\alpha^k \cdot \left [ -\V(\z^k, \lambda^k, \mu^k) \right ],
\end{split}
\end{equation}
where $\V(\cdot)$ has the same form as the Lie derivative in the continuous-time dynamics.

\noindent\textbf{1. Bounding term (i).}
For each agent $v$, using Cauchy–Schwarz and the definition of $\lambda^k$, there exist constants  
\( C_v^f,\, C_v^\lambda,\, C_v^{\mu,j,i} \ge 0 \) such that
\begin{equation}
\begin{split}
    &\| \nabla_{z_v} L_v(\z^k, \lambda_v^k, \mu_v^k)
        - \nabla_{z_v} L_v(\z^\star, \lambda_v^\star, \mu_v^\star) \|^2 \\
    \le\;& C_v^f
        \| \nabla_{z_v} f_v(z_v^k, \z_{-v}^k)
        - \nabla_{z_v} f_v(z_v^\star, \z_{-v}^\star)\|^2 \\
        &+ C_v^\lambda \| \lambda^k - \lambda^\star\|^2 \\
        &+ \sum_{i=1}^{n_v}\sum_{j=1}^{q_v}
        C_v^{\mu,j,i}
        \underbrace{[(\mu_v^{k,j}) \nabla_{z_v^i}h_v^j(z_v^{k,i})
                    - (\mu_v^{\star,j})\nabla_{z_v^i}h_v^j(z_v^{\star,i})]^2}_{(ii)} .
\end{split}
\end{equation}

\noindent\textbf{2. Bounding term (ii).}

\noindent Using the Lipschitz property of each constraint gradient:
\begin{equation}
\begin{split}
    & \left [(\mu_v^{k,j}) \nabla_{z_v^i}h_v^j(z_v^{k,i})
      - (\mu_v^{\star,j}) \nabla_{z_v^i}h_v^j(z_v^{\star,i}) \right]^2 \\
    \le\;& 
        2 \,\!\left[\nabla_{z_v^i}h_v^j(z_v^{k,i})\right]^2
        (\mu_v^{k,j} - \mu_v^{\star,j})^2
        + 2 (\mu_v^{\star,j})^2
        \left [\nabla_{z_v^i}h_v^j(z_v^{k,i})
            - \nabla_{z_v^i}h_v^j(z_v^{\star,i})\right ]^2 \\
    \le\;&
        2 L_v (\mu_v^{k,j} - \mu_v^{\star,j})^2
        + 8 L_v (\mu_v^{\star,j})^2 .
\end{split}
\end{equation}

Thus, for some constants \( C_v, C_v^0 \ge 0\),
\begin{equation}\label{equ:lip_naL}
\begin{split}
    &\| \nabla_{z_v} L_v(\z^k, \lambda_v^k, \mu_v^k)
        - \nabla_{z_v} L_v(\z^\star, \lambda_v^\star, \mu_v^\star) \|^2 \\
    \le\;& C_v \!\left [ 
        \sum_{v=1}^N 
            (\|z_v^k - z_v^\star\|^2 
            + \|\mu_v^k - \mu_v^\star\|^2)
        + \|\lambda^k - \lambda^\star\|^2
    \right ]
    + C_v^0 .
\end{split}
\end{equation}

\noindent\textbf{3. Substituting into the main recursion.} 

\noindent Using \eqref{equ:lip_naL} and the Lipschitz property of \(h_v\), we obtain constants  
\( C, C_0 \ge 0 \) such that
\begin{equation}
\begin{split}
    x^{k+1}
    \le\;& \left  [1 + C (\alpha^k)^2\right ] \, x^k + C_0 (\alpha^k)^2 \\
        &- 2\alpha^k
            \left [ 
                \delta \, \|\z^k - \z^\star\|^2
                - \sum_{v=1}^N \sum_{j=1}^{q_v}
                    (\mu_v^{k,j} - \mu_v^{\star,j})\, h_v^j(z_v^\star)
            \right ].
\end{split}
\end{equation}

The last bracket is nonnegative by Lemma~\ref{lem:monoton}, so
\begin{equation*}
    x^{k+1}
    \le \left [1 + C (\alpha^k)^2 \right ] \, x^k + C_0 (\alpha^k)^2 .
\end{equation*}

\noindent\textbf{4. Convergence.}

\noindent By Lemma~\ref{lem:converging},  
\begin{equation*}
    \sum_{k=1}^\infty \alpha^k \|\z^k - \z^\star\|^2
    = Z_0 < \infty.
\end{equation*}

\noindent Since $\sum_{k=0}^\infty \alpha^k = \infty$, we conclude
\begin{equation*}
    \lim_{k\to\infty} \|\z^k - \z^\star\|^2 = 0,
\end{equation*}
thus proving convergence.

This completes the proof.
\end{proof}

Theorem~\ref{thm:discrete} confirms that a diminishing step size yields a valid discretization of Algorithm~\ref{alg:equa}, ensuring asymptotic convergence under standard assumptions. However, the reliance on a diminishing schedule often results in slow convergence in practice. To address this limitation and improve efficiency, I consider an alternative discretization based on the Backward Euler method (resolvent steps). This approach, presented in Algorithm \ref{alg:resolvant_equa}, allows for a constant step size and achieves a faster theoretical convergence rate.

\subsection{Backward Euler and Forward-Backward Discretization}

A key feature of Algorithm \ref{alg:resolvant_equa} is its use of a fixed parameter $\alpha > 0$, which is independent of the agents' private data. Unlike the diminishing step size approach, this formulation allows the algorithm to achieve an $o(1/k)$ convergence rate.

However, a limitation of this algorithm is that it requires another distributed Nash Equilibrium Problem (NEP) Solver, Inner\_Distributed\_ALG($z_v^-, \lambda_v^-$), in line 3 to solve the inner NEP (\ref{equ:inner_NEP}) using the given ($z_v^-, \lambda_v^-$). These parameters depend only on the previous iteration of the agent's private decision variable and multipliers.
\begin{equation}\label{equ:inner_NEP}
    \begin{split}
         \min_{z_v \in \Z_v} & \phi_v(z_v, \, \z_{-v}) \\
    \end{split},
\end{equation}
with
\begin{equation}
    \begin{split}
     \phi_v(z_v, \, \z_{-v}) := &f_v(z_v, \z_{-v}) + \frac{1}{2 \, \alpha} \, \| z_v - z_v^- \|^2  \\ &+ 
 \psi_v(z_v, \z_{-v})^\top \lambda_v^-+ \frac{\alpha}{2} \, \| \psi_v(z_v, \z_{-v})\|^2 
    \end{split}.
\end{equation}
Note that the inner solver can be implemented via standard distributed projected gradient descent. However, this requires an external entity to verify the inner problem’s solution before the outer loop proceeds—a requirement absent in previous algorithms.

\begin{algorithm}
\begin{algorithmic}[1]
\STATE {\bfseries Input:}Initial $z_v^0$, $\lambda_v^0$, $\alpha>0$.
\FOR{$k = 1, 2, \dots, K$}
\STATE $z_v \leftarrow$Inner\_Distributed\_ALG$(z_v, \lambda_v)$.
\STATE Receive $\z_{-v}$ from the other agents over the communication graph $\mathcal G_\psi$ after the inner NEP is solved in a distributed manner.
\STATE $\lambda_{v} \leftarrow \lambda_v + \alpha \cdot \psi_v(z_v, \z_{-v})$.
\ENDFOR
\end{algorithmic}
\caption{Backward Discretization for Each Agent $v$}
\label{alg:resolvant_equa}	
\end{algorithm}

\begin{theorem}\label{thm:dis2}
    For any $\alpha > 0$, under Assumption~\ref{ass:discrete} (3), Algorithm~\ref{alg:resolvant_equa} is a $(1/2)$-averaged fixed-point iteration. Consequently, it converges to a fixed point and this fixed point corresponds to a GNE of the GNEP~\eqref{equ:GNEP_eq}.
\end{theorem}
\renewcommand{\proofname}{Proof of Theorem \ref{thm:dis2}}
\begin{proof}
Let $\z^k$ and $\lambda_v^k$ denote the values at iteration $k$, and let 
$\z^{k+1}$ and $\lambda_v^{k+1}$ denote the values at iteration $k+1$.  
For each agent $v$, Algorithm~\ref{alg:resolvant_equa} can be written as
\begin{equation*}
\begin{split}
    0 \in\; &\nabla_{z_v} f_v(z_v^{k+1}, \z_{-v}^{k+1})
            + \frac{1}{\alpha} \, \left [ z_v^{k+1} - z_v^{k} \right ]
            + A(l_v,n_v)^\top \lambda_v^{k}   \\
        &\quad + \alpha\, A(l_v,:)^\top \left [ A(l_v,:) \, \z^{k+1} - \c(l_v) \right ]
            + \N_{\Z_v}(z_v^{k+1}), \\
    \lambda_v^{k+1} =\;& \lambda_v^{k} + \alpha \, \left [ A(l_v,:) \, \z^{k+1} - \c(l_v) \right ].
\end{split}
\end{equation*}

Moving all terms indexed by $k$ to the left-hand side and all terms indexed by $k+1$ to the right-hand side yields
\begin{equation*}
\begin{split}
    z_v^{k} - \alpha \, A(l_v,n_v)^\top \lambda_v^{k}
        \in\;& \alpha \, \nabla_{z_v} f_v(z_v^{k+1}, \z_{-v}^{k+1})
            + z_v^{k+1} \\
        &\quad + \alpha^2 A(l_v,:)^\top \left [A(l_v,:) \, \z^{k+1} - \c(l_v) \right ]
            + \alpha \, \N_{\Z_v}(z_v^{k+1}), \\
    \lambda_v^{k}
        =\;& \lambda_v^{k+1} - \alpha \, \left [A(l_v,:) \, \z^{k+1} - \c(l_v) \right ].
\end{split}
\end{equation*}

\noindent Multiplying the $\lambda_v$-equation by $\alpha$ and adding it to the first inclusion gives
\begin{equation*}
\begin{split}
    z_v^{k} \in\;& 
        \alpha \left [
            \nabla_{z_v} f_v(z_v^{k+1}, \z_{-v}^{k+1})
            + A(l_v,:)^\top \lambda_v^{k+1}
            + \N_{\Z_v}(z_v^{k+1})
        \right ]
        + z_v^{k+1},  \\
    \lambda_v^{k}
        =\;& \lambda_v^{k+1} - \alpha \, \left [ A(l_v,:) \, \z^{k+1} - \c(l_v) \right ].
\end{split}
\end{equation*}

\noindent Using the same definitions of $\Delta_{[v]}$ and $\lambda$ as in the continuous-time case, we obtain
\begin{equation*}
\begin{split}
    z_v^{k} \in\;& 
        \alpha \, \left  [
            \nabla_{z_v} f_v(z_v^{k+1}, \z_{-v}^{k+1})
            + A(l_v,:)^\top \left [ \lambda^{k+1}(l_v) + \Delta_{[v]} \right ]
            + \N_{\Z_v}(z_v^{k+1})
        \right ]
        + z_v^{k+1}, \\
    \lambda^{k} =\;& \lambda^{k+1} 
        - \alpha \, \left [A \, \z^{k+1} - \c \right ].
\end{split}
\end{equation*}

Stacking all agents and defining $\mathcal N_{\mathbf{Z}}(\z^{k+1})$ as the concatenation of all normal cones $\mathcal N_{Z_v}(z_v^{k+1})$, we obtain
\begin{equation}\label{equ:forward_individual}
\begin{split}
    \z^{k} \in\;& \z^{k+1} 
        + \alpha \left [ F(\z^{k+1})
                     + A^\top \lambda^{k+1}
                     + D \, \Delta
                     + \N_{\mathbf{Z}}(\z^{k+1})
                 \right ], \\
    \lambda^{k} =\;& \lambda^{k+1} - \alpha \left [ A \, \z^{k+1} - \c \right ], 
\end{split}
\end{equation}
where
\begin{equation*}
D = \mathrm{diag}\!\begin{bmatrix}
    A(l_1,n_1)^\top \\
    A(l_2,n_2)^\top \\
    \vdots \\
    A(l_N,n_N)^\top
\end{bmatrix},
\qquad
\Delta =
\begin{bmatrix}
    \Delta_{[1]}\\
    \Delta_{[2]}\\
    \vdots \\
    \Delta_{[N]}
\end{bmatrix}.
\end{equation*}

\noindent Let $\xi = (\z,\lambda)$.  
Then \eqref{equ:forward_individual} can be written compactly as
\begin{equation*}
    \xi^{k} \in (I_d + \alpha \, \Phi)(\xi^{k+1}),
\end{equation*}
where
\begin{equation*}
\Phi(\xi^{k+1}) =
\begin{bmatrix}
    F(\z^{k+1}) + A^\top \lambda^{k+1} + D \,\Delta + \N_{\mathbf{Z}}(\z^{k+1}) \\
    -A \, \z^{k+1} + \c
\end{bmatrix}.
\end{equation*}

Because $F(\z)$ is strongly monotone and $\N_{\mathbf{Z}}(\cdot)$ is maximally monotone, the operator $\Phi$ is maximal monotone.  
Therefore, the algorithm is a fixed-point iteration of the resolvent
\begin{equation*}
    \xi^{k+1} = (I_d + \alpha \, \Phi)^{-1}(\xi^{k}),
\end{equation*}
which is a \(1/2\)-averaged operator \cite{ryu-large-scale-2022}.  
This shows that the method is a resolvent iteration of a maximal monotone operator, completing the proof.
\end{proof}

In addition to the pure Forward and Backward Euler discretizations, a hybrid scheme combining both approaches can be employed to eliminate the dependency on an inner oracle solver. Algorithm~\ref{alg:disc_equa_os} implements this strategy using a Forward-Backward splitting scheme. By treating specific components of the operator explicitly (forward step) and others implicitly (backward step), this method avoids the computationally expensive inner loops that characterize the fully implicit Backward Euler method. However, in the forward--backward discretization, the decision vector must be received twice in each iteration: once in the forward step (Line~3) and once in the backward step (Line~5).

However, this computational advantage imposes a coordination constraint: the dual learning rate $\alpha$ must be chosen globally. Consequently, all agents must agree on a common scalar value for $\alpha$, even though their initial decision variables and multipliers may differ.

\begin{theorem}\label{thm:alg_disc_os_converge}
Define the inverse learning-rate matrices as
\begin{equation*}
    \begin{split}
        \tau := \text{diag}([\tau_1^{-1}, \dots, \tau_N^{-1}]), \, \sigma:= \alpha^{-1} \, \mathbf{I}_{p \times p},
    \end{split}
\end{equation*}
where $\mathbf{I}_{p \times p}$ is the identity matrix of size $p$. 
If
\begin{equation*}
\min \{ \lambda_{\min}(\tau), \, \lambda_{\min}(\sigma) \} > \tfrac{L_f}{2 \, \delta} +  \| A \|_2,
\end{equation*} 
and additional Assumptions~\ref{ass:discrete} (1) and (3) hold, then Algorithm~\ref{alg:disc_equa_os} converges to a fixed point, which corresponds to a GNE of the GNEP~\eqref{equ:GNEP_eq}.
\end{theorem}

\renewcommand{\proofname}{Proof of Theorem \ref{thm:alg_disc_os_converge}}
\begin{proof}
To apply the convergence result in Theorem \ref{thm:fb}, which shows that the standard forward–backward splitting is averaged and converges to a fixed point corresponding to a KKT solution in~(\ref{ass:KKT_Eq}), 
we rewrite our distributed algorithm into their canonical forward–backward form.

We begin by defining the block operator
\begin{equation*}
H =
\begin{bmatrix}
    \tau & -A^\top \\
    -A & \sigma
\end{bmatrix},
\end{equation*}
where $\tau = \mathrm{diag}(\tau_v)$ collects the inverse step sizes for $\{z_v\}$, and 
$\sigma = \mathrm{diag}(\sigma)$ collects the inverse step sizes for $\lambda$.

The forward–backward inclusion we wish to match is
\begin{equation*}
    H \, \xi^k - \Psi(\xi^k)
        \;\in\;
    H \, \xi^{k+1} + \Phi(\xi^{k+1}),
\end{equation*}
with $\xi = (\z, \lambda)$, and
\begin{equation*}
    \Psi(\xi^k) =
    \begin{bmatrix}
        F(\z^k) \\
        -\c
    \end{bmatrix},\qquad
    \Phi(\xi^{k+1}) =
    \begin{bmatrix}
        A^\top \lambda^{k+1} + D \, \Delta \\
        -A \, \z^{k+1}
    \end{bmatrix}
    +
    \begin{bmatrix}
        \N_{\Z}(\z^{k+1}) \\
        0
    \end{bmatrix}.
\end{equation*}
Expanding the block equations gives
\begin{equation*}
    \begin{cases}
        \tau \, \z^k - A^\top \lambda^k - F(\z^k)
        \in \tau \z^{k+1} - A^\top \lambda^{k+1}
            + A^\top \lambda^{k+1} + D\Delta
            + \N_{\Z}(\z^{k+1}), \\
        -A \, \z^k + \sigma \, \lambda^k + \c
        = -A \, \z^{k+1} + \sigma\lambda^{k+1} - A \, \z^{k+1}.
    \end{cases}
\end{equation*}
Rearranging terms yields the equivalent iteration:
\begin{equation*}
\begin{cases}
    \z^k \in 
        \z^{k+1} +
        \tau^{-1}\!\bigl(
            F(\z^k) + A^\top \lambda^k + D\Delta + \N_{\Z}(\z^{k+1})
        \bigr), \\[2mm]
    \lambda^{k+1}
        = \lambda^k +
        \sigma^{-1}\!\left [ 2 \, A \, \z^{k+1} - (A \, \z^k - \c) \right].
\end{cases}
\tag{$\star$}
\end{equation*}
Step $(a)$ uses the facts that  
(i) each local multiplier satisfies $\lambda_v^k = \lambda^k(m_v) + \Delta_{[v]}$, and  
(ii) $\sigma^{-1} \N_{\Z}(\z) = \N_{\Z}(\z)$.

Thus, the (fully distributed) update finally becomes the standard resolvent form
\begin{equation*}
\begin{cases}
    \displaystyle
    \z^{k+1} = P_{\Z}\!\left[\,\z^k - \tau^{-1}(F(\z^k)+A^\top \lambda^k)\right],\\[2mm]
    \displaystyle
    \lambda^{k+1}
    = \lambda^k + 
      \sigma^{-1}\!\left [2 \, \psi(z^{k+1}) - \psi(z^k) + 2 \, \c \right],
\end{cases}
\tag{$\star\star$}
\end{equation*}
where $\psi_v(\cdot)$ denotes the shared constraint evaluated at agent $v$.

Hence, the algorithm is equivalent to a fixed-point iteration of the forward–backward operator:
\begin{equation*}
(I + H^{-1}\Phi)(\xi^{k+1})
    = (I - H^{-1}\Psi)(\xi^{k}).
\end{equation*}

\noindent\textbf{1. Cocoercivity of the forward operator}.
We compute
\begin{equation*}
\|\Psi(\xi) - \Psi(\xi')\|^2
    = \|F(\z) - F(\z')\|^2,
\end{equation*}
and
\begin{equation*}
\langle \Psi(\xi)-\Psi(\xi'),\, \xi-\xi'\rangle
    = (F(\z) - F(\z'))^\top (\z - \z')
    \ge \delta \, \| \z - \z'\|^2.
\end{equation*}
Using Lipschitz continuity $\|F(\z)-F(\z')\| \le L_f \, \|\z-\z'\|$ gives
\begin{equation*}
\langle \Psi(\xi)-\Psi(\xi'),\, \xi-\xi'\rangle
    \ge \frac{\delta}{L_f} \, \|F(\z)-F(\z')\|^2
    = \frac{\delta}{L_f} \, \|\Psi(\xi)-\Psi(\xi')\|^2.
\end{equation*}
Hence 
\begin{equation*}
    \Psi(\cdot)\;\;\text{is}\;\; \frac{\delta}{L_f}\text{-cocoercive}.
\end{equation*}

\noindent\textbf{2. Monotonicity of the backward operator.}
For $\Phi(\xi)$ we obtain
\begin{equation*}
\langle \Phi(\xi) - \Phi(\xi'),\, \xi - \xi' \rangle
    =
\begin{bmatrix}
    A^\top(\lambda-\lambda') \\
    -A(\z-\z')
\end{bmatrix}^\top
\begin{bmatrix}
    \z-\z' \\ \lambda-\lambda'
\end{bmatrix}
    = 0,
\end{equation*}
so $\Phi(\cdot)$ is monotone (but not strongly monotone).

\noindent\textbf{3. Step-size condition and convergence.}

\noindent Forward–backward splitting converges when  
\begin{equation*}
\lambda_{\min}(H) > \frac{1}{2 \, \beta}
    \qquad\text{with}\qquad 
    \beta = \frac{\delta}{L_f}.
\end{equation*}
Thus we require
\begin{equation*}
\lambda_{\min}(H) > \frac{L_f}{2 \, \delta}.
\end{equation*}

Applying Weyl’s inequality,
\begin{equation*}
\lambda_{\min}(H)
    \ge \min\{\lambda_{\min}(\tau),\, \lambda_{\min}(\sigma)\} - \|A\|_2.
\end{equation*}
Therefore, a sufficient condition for convergence is
\begin{equation*}
    \min\{\lambda_{\min}(\tau),\, \lambda_{\min}(\sigma)\}
    > \frac{L_f}{2 \, \delta} + \|A\|_2.
\end{equation*}

Under this condition, the iterates produced by~\eqref{thm:alg_disc_os_converge} converge to a generalized Nash equilibrium.
\end{proof}

\begin{algorithm}
\begin{algorithmic}[1]
\STATE \textbf{Input:}Initial $z_v^0$, $\lambda_v^0$, $\tau_v$, $\sigma$.
\FOR{$k=1, 2, \dots, K$}
\STATE Receive the decision variables $\z_{-v}$ from other agents through the cost-function graph $\mathcal G_f$ and the shared constraint graph $\mathcal G_\psi$.
\STATE $
z_v^+ \leftarrow P_{\Z_v} \Big\{ z_v - \tau_v \cdot \big[ 
    \nabla_{z_v} f_v(z_v, z_{-v})  + \lambda_v^\top \nabla_{z_v} \psi_v(z_v, \z_{-v}) \big] \Big\}. $
\STATE Receive the updated decision variables $\z_{-v}^+$ from other agents through the shared-constraint graph $\mathcal G_\psi$.
\STATE $
\lambda_{v} \leftarrow \lambda_v + \sigma \cdot \big[ 
    2 \cdot \psi_v(z_v^+, \z_{-v}^+) - \psi_v(z_v, \z_{-v}) + 2 \cdot \mathbf{c}(l_v) \big]$.
\STATE $z_v \leftarrow z_v^+$.
\ENDFOR
\end{algorithmic}
\caption{Forward-Backward for Each Agent $v$}
\label{alg:disc_equa_os}	
\end{algorithm}

\chapter{Simulation}
In this section, I apply our algorithm to real-world scenarios. I begin with a multi-robot placement problem that includes individual constraints and shared linear equality constraints, demonstrating the performance of our fully distributed method. I then apply our discretized algorithm to the same problem and experimentally compare its convergence rate with that of existing discrete algorithms that require consensus on multipliers. Next, I model the half-space covering problem as a GNEP with shared inequality constraints and illustrate the effectiveness of our approach. Finally, beyond multi-agent settings, I evaluate our algorithm on a Cournot competition model and show that our fully distributed method converges in this economic game as well.

\label{simulation_chapter}
\section{Multi-Robot Problems}
A large body of work on multi-robot coordination problems \cite{shorinwa_distributed_2024} formulates them as distributed optimization tasks, where each agent retains its local data privately. In this chapter, I take this framework one step further by modeling the Multi-Robot Placement Problem and the Half-Space Covering Problem as instances of a Generalized Nash Equilibrium Problem (GNEP). This formulation captures not only individual objectives and constraints but also shared constraints arising from task-level coordination, enabling a richer and more flexible representation of multi-robot interaction.

\subsection{Multi-robot Placement Problem}

In this section, I will implement our fully distributed algorithm to solve the Multi-robot Placement Problem. In this scenario, multiple agents (robots) act as simple integrators, collaborating to complete shared tasks. Each task, labeled as $j$, has a central location defined by $c_j \in \R^2$. A robot $v$ is assigned to handle task $j$, which has a set of assigned robots $N_j$ and a total number of robots $n_j$. Each robot $v$ may be responsible for multiple task centers located at different positions. 

This setting induces global shared constraints requiring that the collective placement of all robots assigned to task $j$ be centered at $c_j$. Formally,
\begin{equation*}
    \psi_v^j(z_v, \mathbf{z}_{-v})
    = \sum_{v \in N_j} z_v - n_j \, c_j = 0,
\end{equation*}
where $z_v = (x_v, y_v)$ denotes the placement decision of robot $v$. Let $M_v$ denote the set of all tasks in which robot $v$ participates. Then $\psi_v(z_v, \mathbf{z}_{-v})$ is the collection of all $\psi_v^j$ with $j \in M_v$. Although a global constraint $\psi(\mathbf{z})$ exists, each robot only has access to its own local component $\psi_v(z_v, \mathbf{z}_{-v})$.

In addition to the task locations, each robot $v$ has an anchor point $A_v$, representing the location of its charging station or remote control station. Each robot must stay within a specific range centered on its anchor point. Since the robots are heterogeneous, the size and location of this range depend on each robot's unique capabilities and vary from one robot to another. Specifically, this is expressed as  $\| z_v - A_v\| \le r_v$, which is an individual local constraint that does not depend on the positions of other robots.

For the objective function, each robot balances two main goals: staying close to its individual anchor point and staying close to other robots $u \in N_j$  for tasks $j \in M_v$ to enable easier coordination. To handle this trade-off, each agent’s cost function includes a private penalty term $\rho_v$. Specifically, the cost function is defined as:
\begin{equation}
    f_v(z_v, \z_{-v}) = \| z_v - A_v\|^2 + \rho_v \sum_{u \in N_j, j \in M_v} \| z_v - z_u\|^2.
\end{equation}
This leads to the following GNEP with convex individual constraints and linear shared constraints:
\begin{equation*}
    \begin{matrix*}[l]
        \min_{z_v} & f_v(z_v, \z_{-v}) \\
        \text{subject to} & \| z_v - A_v\| \le r_v\\
        &\psi_v(z_v,\, \z_{-v}) = 0
    \end{matrix*}.
\end{equation*}

Figure~\ref{fig:setup_eq} illustrates a setup with 12 robots and 4 tasks. Our algorithm is fully distributed: each robot keeps its cost function, constraints, and local shared-constraint information private, and no multipliers are exchanged.

Since GNEs depend on initialization, Figure~\ref{fig:trajectory} shows several equilibrium points reached by the continuous algorithm~\ref{alg:equa} under different initializations. Circles represent robots, rectangles denote anchor points $A_v$, dotted regions represent feasible sets, and yellow stars indicate task locations $c_j$. Arrows depict robot--task assignments. This non-uniqueness highlights both the difficulty of computing GNEs and the relevance of our approach.

Figure~\ref{fig:compare_consensus_trajectory} compares the convergence rate of our forward--backward discretization algorithm (Algorithm~\ref{alg:disc_equa_os}) with a consensus-based method adapted from \cite{yi-pavel-2019}. Both achieve comparable convergence rates, but our algorithm requires substantially less communication per iteration. Since the cost-function graph $\mathcal G_f$ coincides with the shared-constraint graph $\mathcal G_\psi$, our method reduces communication by a factor of three compared with consensus-based approaches.

\begin{figure}[H]
\centering
\includegraphics[width=0.7\columnwidth]{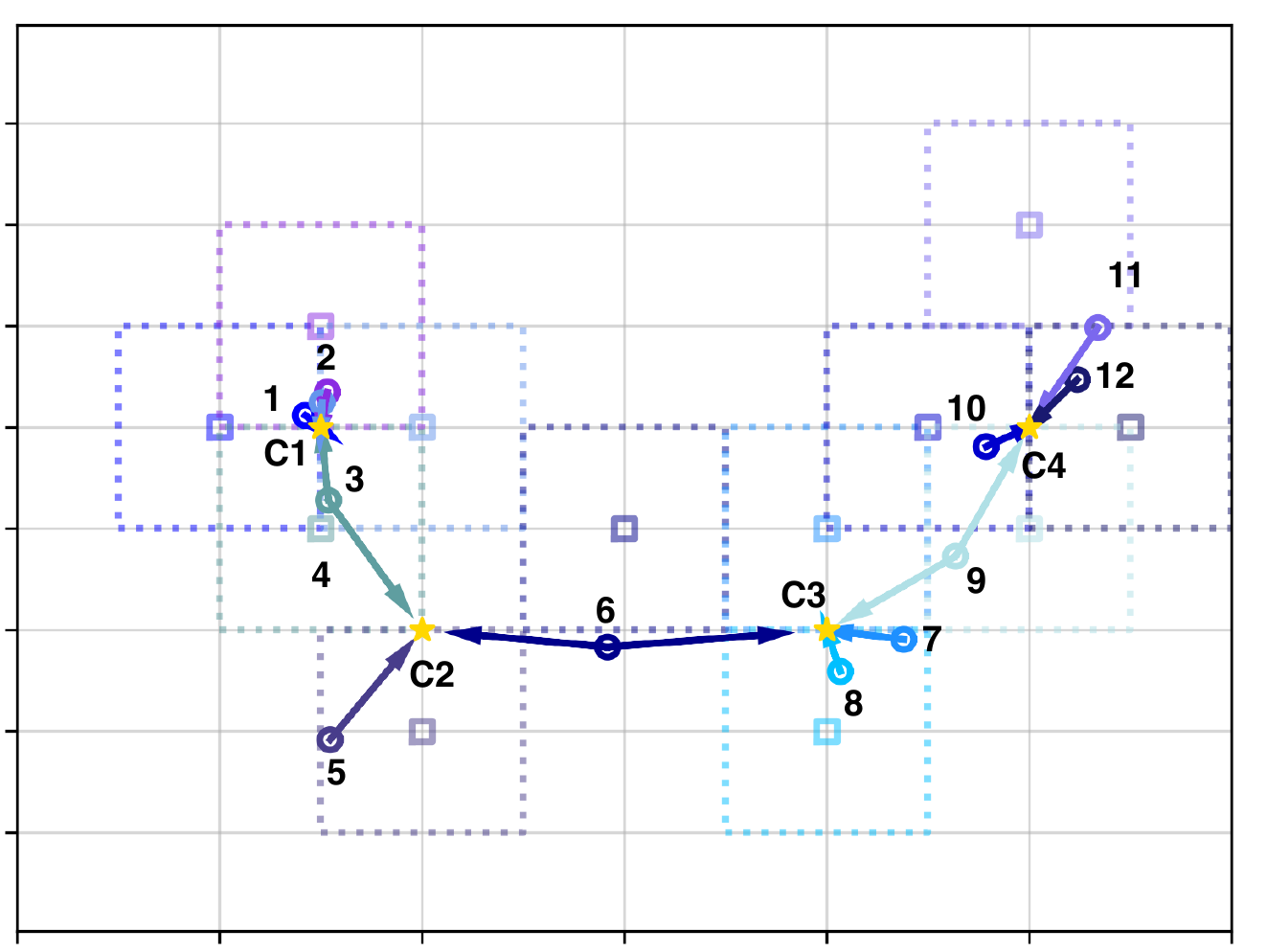} 
    \caption{Multi-robot placement with 12 robots and 4 tasks.}
    \label{fig:setup_eq}
\end{figure}

\begin{figure}[H]
\centering
\includegraphics[width=0.7\columnwidth]{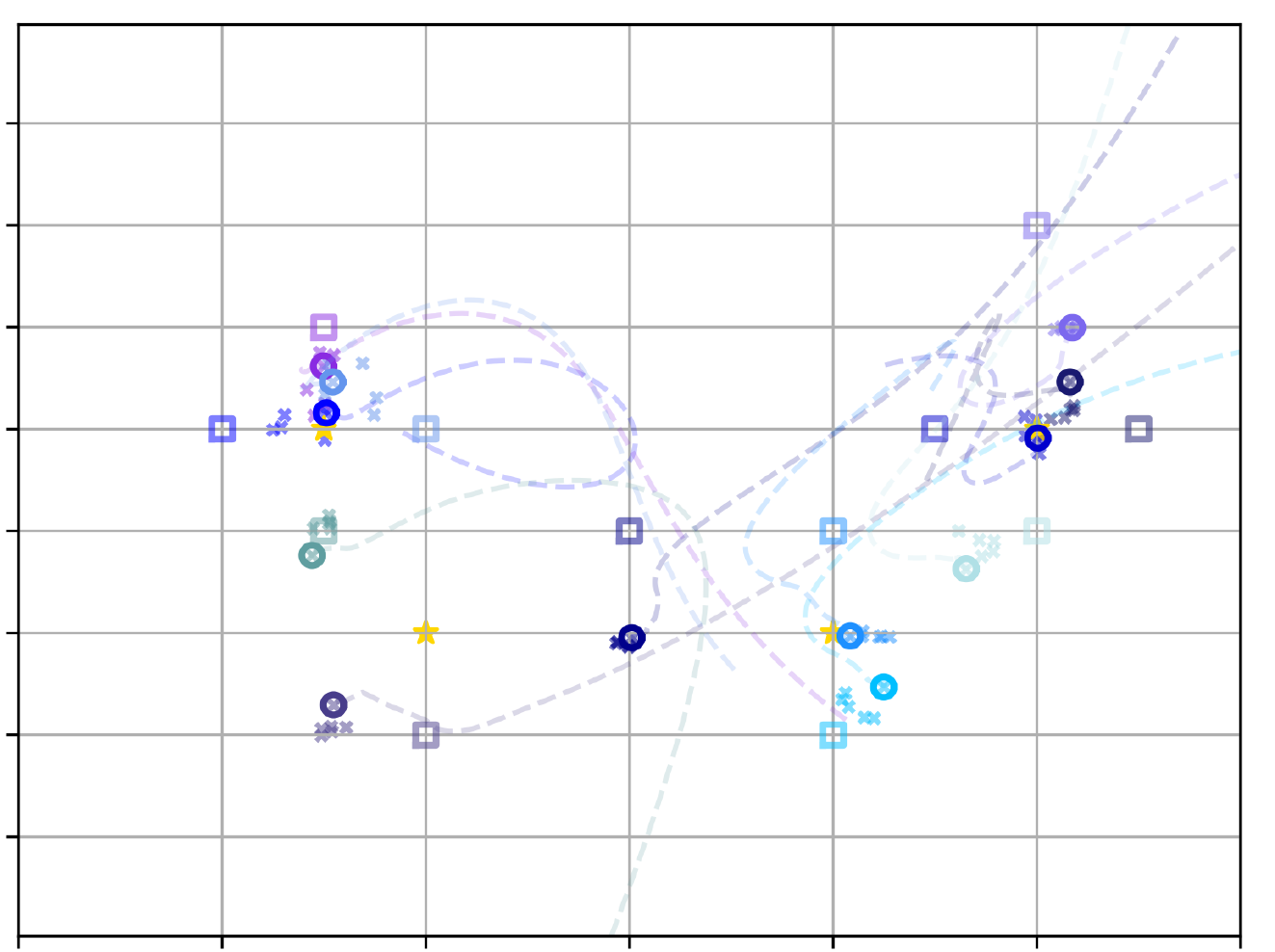}
    \caption{Converging trajectories from six random initializations (seeds 0–5).}
    \label{fig:trajectory}
\end{figure}

\begin{figure}[H]
\centering
\includegraphics[width=0.7\columnwidth]{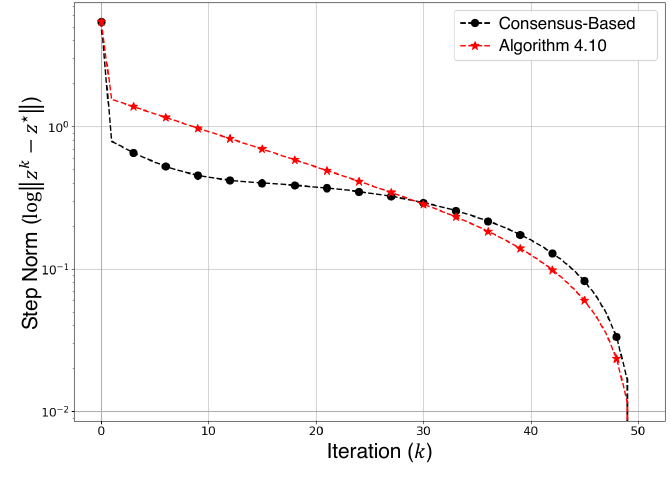}
    \caption{
Comparison of Algorithm~\ref{alg:disc_equa_os} (forward–backward discretization) with consensus-based approaches from~\cite{yi-pavel-2019}. Both are initialized at zero with a learning rate of $0.01$, converging to a v-GNE.
    \label{fig:compare_consensus_trajectory}}
\end{figure}

\subsection{Multi-Robot Half Space Covering}
In the multi-agent sensing placement problem, each agent $v$ tracks its assigned target $T_v$ while remaining close to neighboring agents for communication. The local cost function is
\begin{equation*}
    f_v(z_v, \mathbf{z}_{-v})
    = \| z_v - T_v \|^2
      + \rho_v \sum_{u \neq v} \| z_v - z_u \|^2,
\end{equation*}
where $z_v \in \mathbb{R}^2$ denotes the position of agent~$v$. 

The environment is partitioned into half-spaces, and to ensure sensing coverage, the geometric center of each group must remain on the same side of a separating hyperplane as its designated service region. For a group $\mathbf{G}$ constrained to the half-space $-n_\mathbf{G}^\top \, p - c_\mathbf{G} \le 0$, the corresponding shared inequality constraint is  
\begin{equation*}
    g^\mathbf{G}(\mathbf{z})
    = -n_\mathbf{G}^\top 
      \left( \frac{1}{|\mathbf{G}|} \sum_{u \in \mathbf{G}} z_u \right)
      - c_\mathbf{G}.
\end{equation*}

Figure~\ref{fig:setup} shows the outcomes of Algorithm~\ref{alg:main} from five different initializations. Ten targets (squares) are assigned to agents partitioned into three groups, each subject to a corresponding half-plane constraint. Colored regions represent the feasible half-planes for each group. Dotted lines trace the trajectories from different initial points. Circles denote the converged generalized Nash equilibria (GNE), while crosses (“x”) mark distinct GNEs obtained from varying initializations. The figure illustrates both the non-uniqueness of the GNE set and the dependence of Algorithm~\ref{alg:main} on initialization.

\begin{figure}[H]
\centering
\includegraphics[width=0.7\columnwidth]{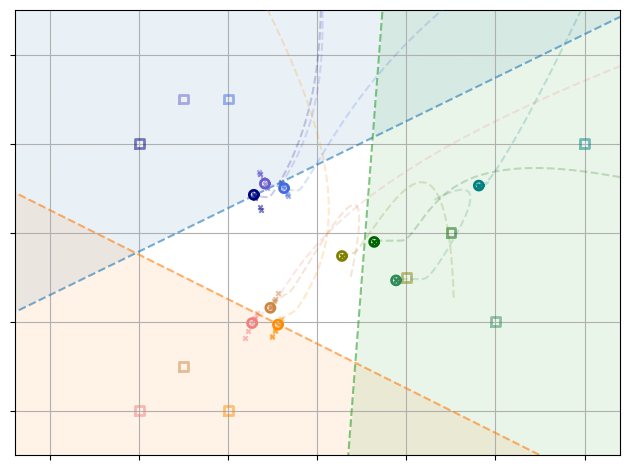} 
    \caption{Sensing placement problem setup and results of Algorithm~\ref{alg:main} under five initial conditions.}
\label{fig:setup}
\end{figure}


\section{Cournot Competition}
Consider the $N$-agent Cournot game described in \cite{yi-pavel-2019}, where $N$ agents produce $p$ different types of energy for the market, with each agent producing a quantity $z_v$. For each type of energy, there is a total market demand, leading to global shared equality constraints. Variants of this game, such as the river-basin game, are introduced in \cite{meanfield,krawczyk-relaxation-2000}. To be precise, a three-agent game with two types of energy markets, as shown in Figure \ref{fig:cournot32}, can be mathematically modeled as follows:
\begin{equation}
    \begin{matrix*}[l]
        \min_{z_v} & f_{v}(\z) = c_v(z_v) - r_v(\z)\\
        \text{subject to}& z_v \ge 0\\
        & g_v(z_v, \z_{-v}) = 0
    \end{matrix*}.
\end{equation}
The net cost function $f_{v}(\z)$ is defined as the difference between the production cost $c_v(z_v)$ and the revenue $r_v(\z)$. The production cost is defined by:
$$c_v(z_v) = c_{1, v}^\top \,z_v + c_{2, v} \, (\mathbf{1}^\top z_v)^2,$$
where $c_{1, v}$ and $c_{2, v}$ represent coefficients specific to each agent $v$. The vector $c_{1, v}$ is sampled uniformly between 1 and 4, while $c_{2, v}$ is sampled uniformly between 1 and 8. 

The revenue function $r_v(\z)$ is given by the sum of $r_v^j(\z)$ over all markets $j$ in which agent $v$ participates:
$$
r_v^j(\z) = \left [P_j - d_j\sum_{u \text{ involved in }j} z_u^j \right ]z_v^j,
$$
where $z_v^j$ represents the quantity produced by agent $v$ for market $j$. The constants $P_j$ and $d_j$ are constants that define the inverse demand law for market $j$, with $P_j$ sampled uniformly between 250 and 500, and $d_j$ sampled uniformly between 1 and 5.

In addition to individual cost functions, the two markets have total energy requirements that serve as global common constraints for the game:
\begin{equation*}
\begin{split}
    g^1(\z) &= z_1^1 + z_2^2 - r_1 \\
    g^2(\z) &= z_2^2 + z_3^2 - r_2
\end{split},
\end{equation*}
where the total energy required for each market, $r_j$ is sampled uniformly between 20 and 80.

Figure~\ref{fig:cournot32} shows a Cournot game with three agents and two markets, and Figure~\ref{fig:cournot_207} shows a Cournot game with twenty agents and seven markets based on \cite{yi-pavel-2019}. Although the number of agents and markets can be arbitrarily large, these examples use relatively small problem sizes for illustration.

Under the twenty–agent and seven–market setup, Figure~\ref{fig:cournot_207_decision} illustrates that the decision variables converge to a GNE of the game. Starting with any initial multipliers leads to convergence to a GNE, and this remains true for any problem size. Figure~\ref{fig:cournot_207_multipliers} shows the corresponding convergence of the multipliers. Depending on the initialization, the game may converge to different GNEs.

The convergence is theoretically guaranteed for any number of agents and any number of common constraints, provided that the problem satisfies the required assumptions. Moreover, the distributed algorithm preserves privacy for all agents even in arbitrarily large games.

\begin{figure*}[ht]
  \centering
  \begin{subfigure}[t]{0.3\textwidth}
    \includegraphics[width=\textwidth]{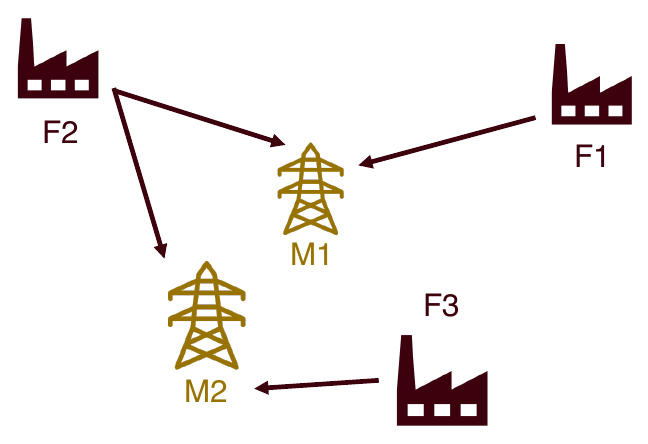}
    \caption{}
    \label{fig:cournot32}
  \end{subfigure}
  \hfill
  \begin{subfigure}[t]{0.48\textwidth}
    \includegraphics[width=\textwidth]{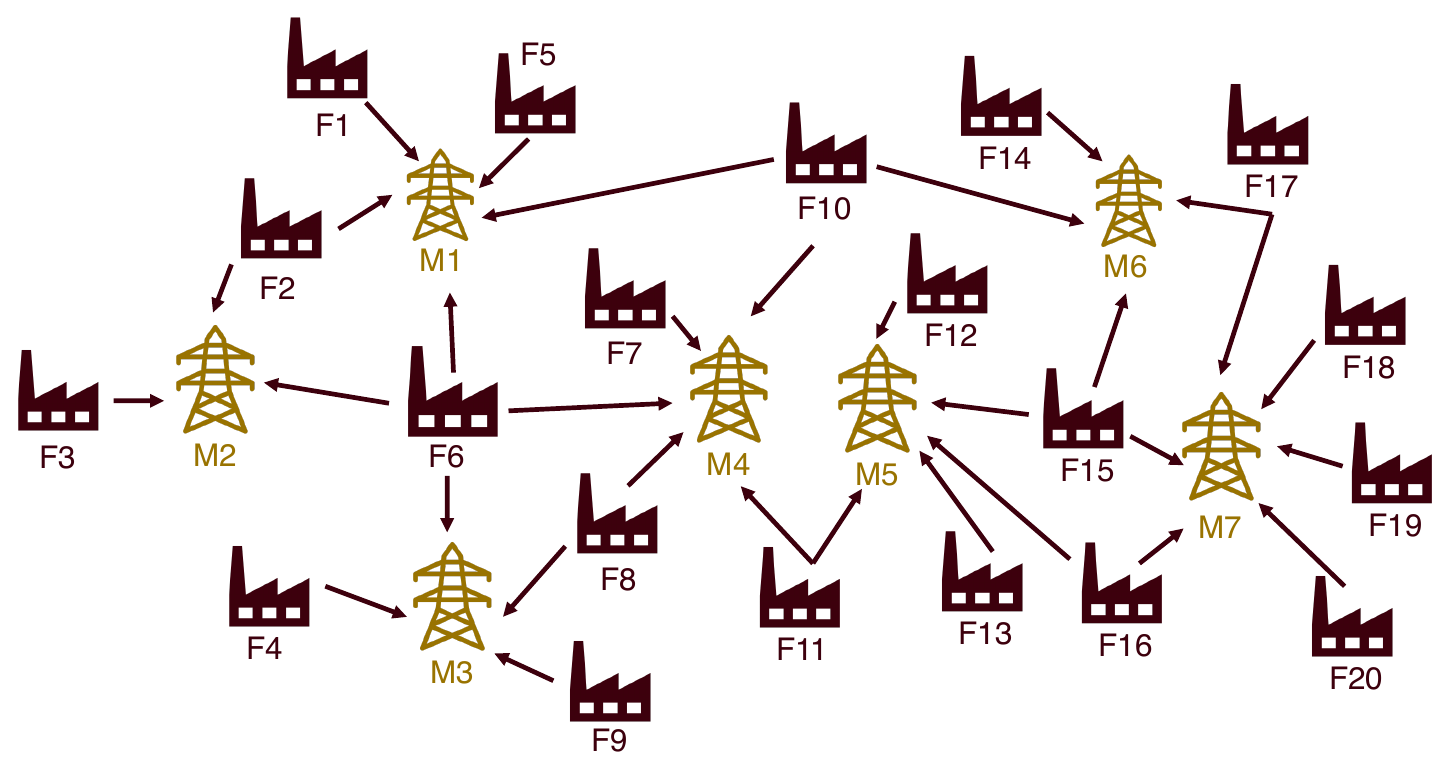}
    \caption{}
    \label{fig:cournot_207}
  \end{subfigure}

  \caption{\textbf{(a)} A Cournot Game with three agents and two markets. \textbf{(b)} A Cournot Game with twenty agents and seven markets.}
  \label{fig:cournot_setup}
\end{figure*}

\begin{figure}
    \centering
    \includegraphics[width=0.7\linewidth]{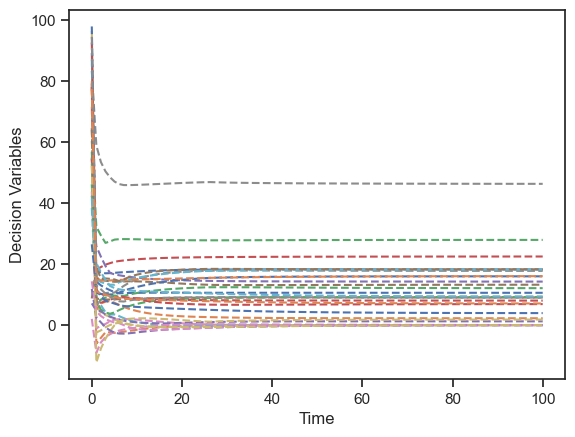}
    \caption{Decision variables for a Cournot game with 20 agents and 7 markets.}
    \label{fig:cournot_207_decision}
\end{figure}
\begin{figure}
    \centering
    \includegraphics[width=0.7\linewidth]{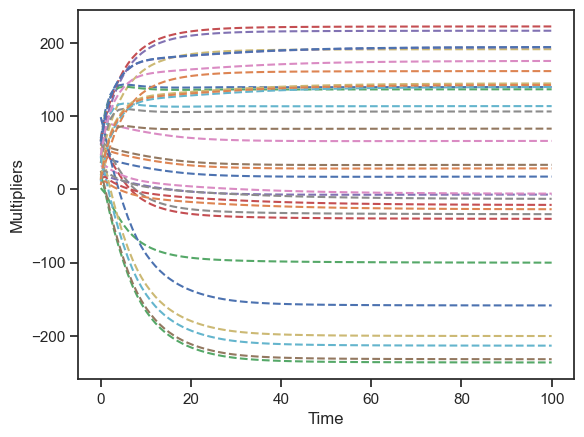}
    \caption{Multipliers for a Cournot game with 20 agents and 7 markets.}
    \label{fig:cournot_207_multipliers}
\end{figure}
\chapter{Contextual Bandit Based Active Learning}
\label{active_learning}
The challenge with active learning algorithms is the uncertainty of the statistical distribution of unlabeled data, making it difficult to choose the best hand-crafted strategy. To address this, I introduced \textbf{C}ontextual \textbf{A}daptive \textbf{A}ctive \textbf{L}earning (CAAL). In CAAL, each "arm" represents a hand-crafted strategy. Unlike existing frameworks that select strategies based only on feedback from labeled data, I dynamically choose strategies for labeling batches of data using reward prediction with external context information. This general framework allows for customization with domain knowledge to design more effective rewards and context candidates. In addition, I experimentally show that CAAL outperforms the existing baseline adaptive strategy on public datasets using our reward and context design. Our results are consistent regardless of batch size in each iteration.

\section{Introduction}

Active learning in the context of machine learning is a technique designed to minimize the effort required for querying labels from human annotators within a pool of unlabeled datasets while still achieving decent performance from the underlying supervised learning model \cite{settles_active_2010}. The majority of work in this area focuses on hand-craft strategies to reducing sample complexity \cite{balcan_true_2010}. These strategies have theoretical guarantees based on the statistical assumptions about the unlabeled dataset. However, they may only work well for certain types of data, and we often do not know what kind of data we are dealing with beforehand.

On the other hand, \cite{baram_online_2004} addressed this challenge by combining several hand-crafted active learning strategies, aiming to leverage the strengths of each. \cite{Hsu_Lin_2015} expanded on this idea by emphasizing adaptability and connecting it to bandit algorithms through an adversarial bandit framework. However, adversarial bandit methods are often overly conservative and tend to perform similarly to random selection with equal probability. This is because they always keep some level of exploration, without making clear assumptions about the environment. For example, in active learning, the improvement in the performance of a machine learning classifier as more labeled data is queried can be a type of contextual information. Therefore, I propose incorporating contextual information into the bandit-type methods to reduce this conservativeness and enhance adaptability.

The work has the following contributions:
\begin{enumerate}
    \item Compared to the recent trend of deep active learning, I focus on meta-learner-based active learning because deep learning often performs worse on tabular data \cite{shwartz-ziv_tabular_2022}, which makes up the majority of datasets in the industry. This \textbf{meta-strategy} approach is more \textbf{industry-friendly} and easily integrates into existing industry pipelines.
    \item Most researchers focus on selecting one unlabeled data point to label in each iteration. Instead, I focus on batch selection for active learning and its \textbf{robustness across different batch sizes}, which is more relevant to real industrial scenarios.
    \item I emphasize the adaptability to the best hand-crafted active learning methods for unknown datasets.
    \item Our approach suggests that incorporating additional environmental context information for \textbf{reward prediction} improves adaptability compared to the existing state-of-the-art algorithms, which mainly rely on adversarial bandits.
\end{enumerate}

The part is organized as follows: In Section 6.2, I present the problem definition, the active learning protocol, and our desired adaptive framework. I also introduce our algorithm along with the design of the reward and context. In Section 6.3, I apply our algorithm to open-source datasets, demonstrating the adaptability of our proposed algorithm. My findings reveal that our algorithm outperforms the state of the art in most datasets, suggesting the potential utility of our proposed algorithm in real-world scenarios.

\subsection{Related Work}
Existing hand-crafted active learning strategies fall into two main categories: distribution-based and rejection-based approaches. Distribution-based methods rely on knowledge of the data distribution, which may not always be accessible or precisely estimated. In contrast, rejection-based methods involve maintaining a pool of classifiers, which can be challenging to scale up for large-scale applications. Recognizing the constraints of these hand-crafted approaches, there is a growing interest in dynamically adapting active learning strategies to fit the specific dataset. Rather than relying on a fixed strategy, the emphasis is on selecting the most appropriate approach adaptively based on the characteristics of the data. In this section, I will begin by introducing individual hand-crafted strategies and then delve into adaptive strategy selection, which better reflects real-world usage scenarios.

\textbf{Hand-craft Single Strategies}

In recent advancements in hand-crafted active learning, techniques from \cite{NIPS2005_340a3904, hoi_batch_2006, dasgupta_hierarchical_2008, huang_active_2014, sener2018active} offer various heuristic approaches. However, these heuristics lack widely applicable guarantees, and I will demonstrate that their performance varies across datasets. One reason for this discrepancy is that these hand-crafted approaches do not consider the sequential feedback nature of the active learning process. Therefore, moving towards data-driven adaptive meta-learning-type active learning to select a better hand-crafted strategy is crucial to address this challenge. On the other hand, rejection-based approaches like \cite{NIPS2005_340a3904} were the first to provide theoretical guarantees without strict assumptions on the underlying data distribution. While these approaches narrow down the search space (version space) from the sequential feedback perspective, they require maintaining a pool of base machine learning models in every iteration, which is impractical given the current trend towards large-scale training.

\textbf{Deep Active Learning for Batch Selection}

Several works \cite{NEURIPS2019_95323660,Ash2020Deep,ash2021gone,citovsky2021batch} have developed handcrafted deep learning-based active learning approaches for batch selection, primarily targeting image datasets. However, it is well known that deep learning tends to perform worse than tree-based models on tabular data \cite{shwartz-ziv_tabular_2022}, which constitutes the majority of datasets in industry. Consequently, our work focuses on meta-learner-based batch selection, which is compatible with any type of base classifier rather than being specific to deep learning.

\textbf{Data Driven Active Learning}

The bandit setup is widely used in online learning and adaptive learning contexts. In short, in the bandit framework, the agent has a pool of potential actions called arms. Each arm produces a different reward, but the agent does not know in advance which arm will yield the best reward. Therefore, the agent must first explore by pulling arms to estimate the rewards from the environment until it is confident enough about the estimations. Then, it switches to exploitation, using the arm that has provided the best reward so far. The algorithm used for reward estimation varies depending on the assumption about the format of the reward. 

\cite{10.1007/978-3-540-74958-5_14, baram_online_2004} were the first to apply this idea within the active learning framework, where each arm represents a handcrafted active learning strategy, assuming an adversarial environment for the reward. \cite{Hsu_Lin_2015, 9093390} expanded on this concept by designing the reward as an unbiased estimator of the training loss. However, \cite{NIPS2017_8ca8da41} argued that treating arms as handcrafted active learning strategies limits performance. They developed a new active learning strategy, but it could only select one sample per iteration, which is impractical. Our approach is similar to \cite{zhang2023algorithm}, but they focus on the imbalance ratio of the queried dataset using Thompson sampling without contextual information, differing from our goal of improving machine learning classifier performance. Additionally, their regret bound assumes independent and identically distributed rewards for each arm, which is unrealistic when sequentially querying unlabeled data. Thus, \cite{Hsu_Lin_2015} remains the state-of-the-art for batch mode active learning.

The work builds upon the work of \cite{Hsu_Lin_2015}. I take it a step further by proposing that the non-stationary \textbf{reward can actually be predicted based on some external context}. Therefore, unlike assuming the non-stationary reward is unpredictable under the adversarial bandit setup, our method based on contextual bandit is more greedy and less conservative than \cite{Hsu_Lin_2015}, resulting in better performance.


\section{Contextual Adaptive Active Learning (CAAL)}
In this section, I start by defining the active learning framework and then explore how to integrate it with the bandit setup. Unlike relying on a single predefined active learning strategy, the adaptive framework uses a diverse pool of human-designed strategies. The agent dynamically selects the most effective strategy based on ongoing feedback.

Next, I introduce our solution framework with the concept of the contextual bandit, which involves predicting the rewards of different strategies using external contextual information. I then design the reward and context components and present our main proposed algorithm.

It is important to note that our design of reward and context is just one possible combination. Our proposed algorithm serves as a general framework, allowing for further customization. With additional domain knowledge of the dataset, one could design more effective reward and context candidates to better adapt to the most effective active learning strategy.

\subsection{Active Learning Framework}

Active learning is an iterative process where data is selected from an unlabeled dataset and queried for labels to train a classifier. The goal is to achieve a good classifier with a minimal number of queries. To be precise, at each iteration step $k$, with a given base machine learning classifier $M_k$, labeled dataset $\mathcal{D}^{lab}_k = \{ x^{lab}_k, \, y^{lab}_k\}$, and unlabeled dataset $\mathcal{D}^{unl}_k = \{ x^{unl}_k\}$, a hand-crafted batch mode active learning strategy $\mathcal{S}$ with batch size $b$ is a function that selects $b$ unlabeled data points $x^{sel}_k = \{ x^n \, | \, x^n \in \mathcal{D}^{unl}_k \}_{n=1}^b$ from $\mathcal{D}^{unl}_k$. This selection is based on information gathered from $M_k$ and $\mathcal{D}^{lab}_k$. Therefore, we can express the selection as:
\begin{equation}\label{equ:base_stratey}
    x^{sel}_k = \mathcal{S} \left (M_k, \, \mathcal{D}^{lab}_k, \, \mathcal{D}^{unl}_k \right )
\end{equation}
The annotator will then label the selected dataset $y^{sel}_k$, which will be added to the labeled dataset:
\begin{align}
    &\mathcal{D}^{lab}_{k+1} \leftarrow \mathcal{D}^{lab}_k \cup \{ x^{sel}_k, \, y^{sel}_k\},\\
    &\mathcal{D}^{unl}_{k+1} \leftarrow \mathcal{D}^{unl}_k \setminus x^{sel}_k
\end{align}
$\mathcal{D}^{lab}_{k+1}$ is then used to update the base machine learning model $M_{k+1}$, and the entire process repeats. The aim of active learning is to reduce the number of label queries while maintaining a satisfactory performance of the base model. Figure \ref{fig:active_learning} illustrates the framework of active learning.
\begin{figure}
        \centering
        \begin{subfigure}[b]{0.45\textwidth}
                \includegraphics[width=\textwidth]{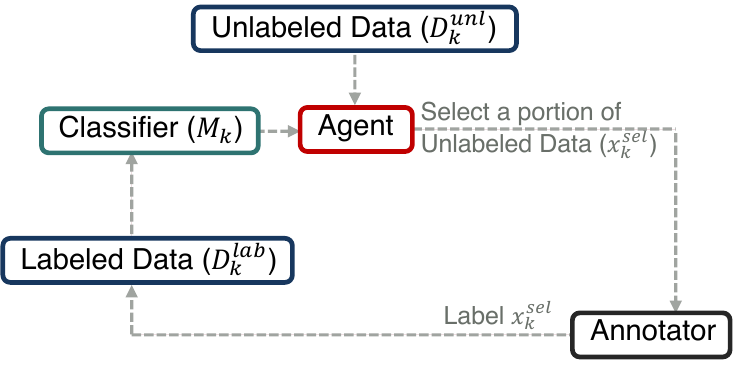}
                \caption{Active Learning Framework.}
        \label{fig:active_learning}
        \end{subfigure}       
        \begin{subfigure}[b]{0.45\textwidth}
                \includegraphics[width=\textwidth]{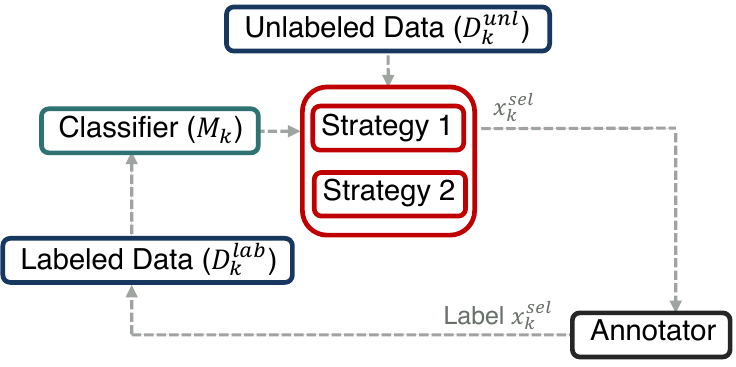}
                \caption{Desired Bandit Based Framework.}
\label{fig:bandit_framework}
        \end{subfigure}
        \caption{}
        \label{fig:lab} 
\end{figure}

\subsection{Adaptive framework}
Figure \ref{fig:bandit_framework} illustrates the desired adaptive bandit framework. Instead of employing a single hand-crafted active learning strategy defined in equation \ref{equ:base_stratey}, we utilize a collection of such strategies represented by $\mathcal{A}$, analogous to arms in the context of the bandit setup. The adaptive policy $\Pi: \mathbb{R} \to \mathcal{A}$ is a function mapping the context $c_k$ to an action, defined as:
\begin{equation}
a_k = \Pi(c_k).
\end{equation}
Given $r^\star_k$ as the best reward among all arms at time $k$, and $r(a_k)$ as the reward if the agent chooses arm $a$ at time $k$, the experimental regret over $N$ trials of trajectories with a horizon $T$ is defined as:
\begin{equation}
\widehat{\text{Regret}}(\Pi) = \frac{1}{N}\sum_{n = 1}^N \sum_{k=1}^T \left [ r^\star_k - r_n(a_k) \right ] \quad \text{with }a_k = \Pi(c_k).
\end{equation}
The objective of the bandit problem is generally to find a policy function $\Pi$ that minimizes the expected regret.

Assuming the reward is a linear function of the context, i.e., $r(a_k) = \theta_a^\top c_k$, I can leverage the existing LinUCB algorithm \cite{li_contextual-bandit_2010}. However, in the context of active learning, the reward, which I will define explicitly later, is generally non-stationary. Therefore, I can only assume that the reward is locally linear and takes advantage of non-stationary linUCB algorithms, such as discounted linUCB \cite{10.5555/3454287.3455365} (d-linUCB for short) or sliding window linUCB (SW-linUCB for short) \cite{pmlr-v89-cheung19b}.

Our main proposed meta-strategy, Contextual Adaptive Active Learning (CAAL), is described in Algorithm \ref{alg:main_al}. In every iteration, the agent begins by receiving context from the environment, which represents the current state of the active learning process. Next, the agent selects the best predictive reward for each arm based on this context. After selection, the environment reveals the reward, and the agent updates the reward prediction function accordingly. This process repeats until the end of the trajectory.
\begin{algorithm}
   \caption{Contextual Adaptive Active Learning (CAAL)}
   \label{alg:main_al}
\begin{algorithmic}[1]
   \STATE {\bfseries Input:} Active Learners ($\mathcal{A}$), Batch Size ($b$), Window Size ($w$). \STATE {\bfseries Initialize:} $\mathbf{C_a} = \mathbf{I}$, $d_a = 0$, $\theta_a = \mathbf{C_a}^{-1} \cdot d_a$ for all arms $a \in \mathcal{A}$.
   \WHILE{True}
   \STATE Observe Context $c_k$ from the environment.
   \STATE Select hand-craft active learner $a_k = \arg \max_{a \in \mathcal{A}} \{ \theta_a^\top c_k + \gamma \sqrt{c_k^\top \mathbf{C_a}^{-1} c_k}\}$.
   \STATE Select $b$ number of unlabeled data $x^{sel}$ based on $a_k$ from $\mathcal{D}^{unl}_k$.
   \STATE Receive $y^{sel}$ from the annotator.
   \STATE $\mathcal{D}^{lab}_{k} \leftarrow \mathcal{D}^{lab}_k \cup \{ x^{sel}_k, \, y^{sel}_k\}$. 
   \STATE $\mathcal{D}^{unl}_{k+1} \leftarrow \mathcal{D}^{unl}_{k} \setminus x^{sel}_k$.
  \STATE Train Classifier $M_k$ with $\mathcal{D}^{lab}_{k}$.
   \STATE Receive Reward $r_k$.
   \STATE $\mathbf{C_{a_k}} \leftarrow \sum_{t = k-w}^k c_t c_t^\top$.
   \STATE $d_{a_k} \leftarrow  \sum_{t = k-w}^k r_t c_t$.
   \STATE $\theta_{a_k} \leftarrow \mathbf{C_{a_k}}^{-1} d_{a_k}$.
   \STATE $k \leftarrow k+1$.
   \ENDWHILE
\end{algorithmic}
\end{algorithm}

\textbf{Reward Signal}

In practice, no matter how the unlabeled data is selected for labeling, I must maintain an IID control group to evaluate model performance. This control group is usually smaller than the dataset used for training. It provides a reward signal, which, to our knowledge, is not commonly used in existing Active Learning frameworks. I argue that by properly designing the reward signal from the control group, the reward can depend on some external context, allowing our CAAL algorithm to adapt to the best hand-crafted active learning strategy (arm).

\textbf{Design of Reward and Context}

It is relatively straightforward to confirm through experimentation that the classifier's performance generally improves with the increasing size of the labeled dataset. As such, I consider one potential context: the size of the labeled dataset. Taking inspiration from the reward design in \cite{Casanova2020Reinforced}, I utilize the \textbf{performance difference} between $M_k$ and $M_{k+1}$ as the reward. In other words, instead of using the absolute change in accuracy/ROC-AUC (or error rate reduction), I focus on the \textbf{relative} change. This reward design reduces dependency on previous arm selections, which is important because rewards in active learning are typically non-Markovian. Previous selections influence rewards, a correlation not managed by existing bandit algorithms. Because of this, our reward design can also help reduce high bias from the collected data and provide us with counterfactual information.

Furthermore, recognizing that the reward is non-stationary as the context increases, I adapt the SW-linUCB algorithm from \cite{pmlr-v89-cheung19b} to accommodate this by assuming the reward is locally linear in its external context. In essence, our algorithm extrapolates to predict the model improvement for each arm based on the context. Subsequently, it selects the arm with the better predictive reward, taking into account the uncertainty bound.


\section{Experimental Results}
In this section, I perform experiments on a real-world dataset to address a binary classification problem. To ensure fairness, I use logistic regression as the machine learning classifier in all experiments. I begin by detailing the dataset utilized, followed by an explanation of the evaluation metric that aligns with the active learning objective. Finally, I demonstrate that our proposed approach outperforms the current state-of-the-art adaptive algorithm ALBL across most datasets.

\subsection{Evaluation Metric}
The goal of active learning is to maximize the performance of a machine learning classifier while minimizing the number of labeled datasets needed, as labeling data is typically expensive. To assess the algorithm's performance, I use a fundamental evaluation metric. This involves fixing the percentage of queried labeled data among all training data and comparing the testing AUC of the ROC curve of the trained machine learning model with the queried labeled data. To ensure fair comparisons across datasets, I use the ROC-AUC ratio instead of the absolute ROC-AUC value. The ROC-AUC ratio is defined as:
\begin{equation*}
    \text{ROC-AUC Ratio} = \frac{\text{ROC-AUC}}{\text{ROC-AUC}_{max}} \times 100 \%,
\end{equation*}
where $\text{ROC-AUC}_{max}$ represents the maximum ROC-AUC achieved when using the entire training dataset.

Additionally, to obtain a consistent evaluation metric for assessing the overall average performance of a strategy, I compute the Area Under the Curve of the ROC-AUC ratio.

\subsection{Dataset Description}
When selecting the dataset, I considered two main criteria:
\begin{enumerate}
    \item The dataset should be large enough to allow testing of various batch-size active learning strategies.
    \item Different datasets may show different performances of hand-crafted strategies. 
\end{enumerate}
I selected datasets from \cite{zhan_comparative_2021}, which provides a summary of common benchmark datasets in active learning, all sourced from the UC Irvine Machine Learning Repository. Figure \ref{fig:dataset} illustrates that when using a batch size of 10, the performance of hand-crafted active learning strategies varies significantly across datasets. For instance, in the ILPD and I vs J datasets, the K-center strategy performs the best overall, while in the Australian Credit (AC) and German Credit (GC) datasets, the Information Diversity strategy is the most effective. Conversely, Information Diversity performs poorly in ILPD and I vs J, whereas K-center is less effective in the Australian Credit and German Credit datasets. I use the implementations of these handcrafted active learning strategies from the Google AL toolbox \cite{google_google/active-learning_nodate}.

\begin{figure}
  \centering
  \includegraphics[width=1\textwidth]{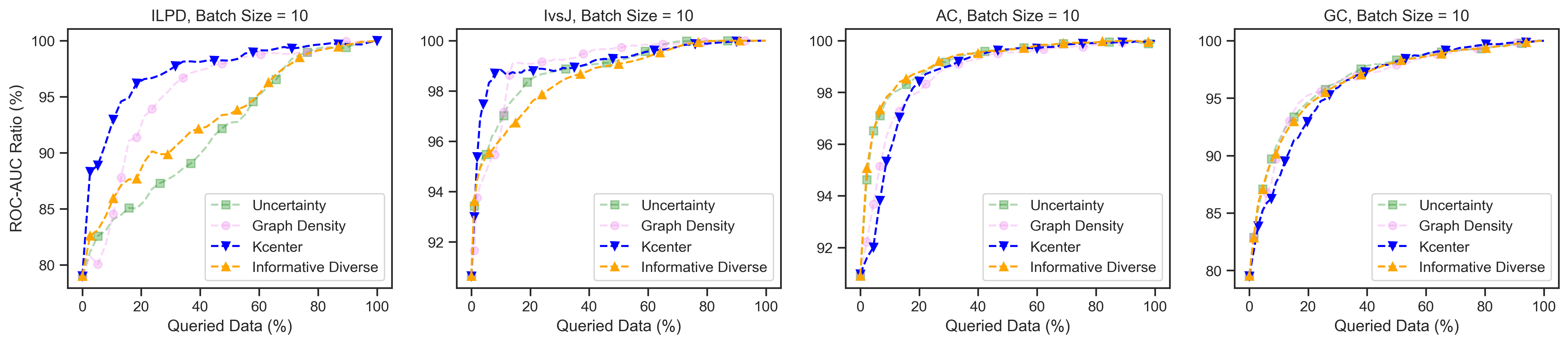}
  \caption{Selected dataset response to active learning strategies with batch size 10.}
  \label{fig:dataset}
\end{figure}

\subsection{Results}
Figure \ref{fig:results} summarizes the performance of various active learning strategies with different batch sizes, including two hand-crafted approaches: K-center and informative Diversity. The x-axis represents the percentage of the queried labeled dataset, while the y-axis indicates the ROC-AUC Ratio of the classifier on the testing dataset. With a batch size of 10, meaning that in each iteration, the agent selects 10 unlabeled data points to be labeled. I conducted 100 experiments. I chose these two hand-crafted strategies because they demonstrate varying performance across datasets.

I also compared ALBL, the baseline algorithm from \cite{Hsu_Lin_2015} and still considered the state-of-the-art adaptive active learning algorithm for batch selection to our knowledge, with our proposed algorithm, CAAL. The results show that while with a batch size of 5, CAAL only has comparable results compared to ALBL, in the other batch sizes, CAAL exhibits better adaptive ability than ALBL. This suggests that using our proposed method can benefit from the less conservative nature of the contextual bandit algorithm compared to those based on adversarial bandit algorithms.

Table \ref{tab:results} summarizes the Area under the Curve of ROC-AUC ratio and the standard deviation for batch sizes of 5, 10, 15, 20, and 25. It shows that except for a batch size of 5, CAAL only gives comparable performance with ALBL; for the other batch sizes, CAAL exhibits better performance. The potential limitation for a batch size of 5 is that the linear regression fitting requires some amount of data, while ALBL does not. However, given the fact that in reality we often have a large batch size, one can expect our performance will be better in real-world applications.
\subsection{Reward Prediction}
In this section, I focus on predicting rewards from context information. Figure \ref{fig:prediction} demonstrates that the predicted reward as a function of iteration from the linear regression matches the trend of the true received reward. This indicates that our proposed approach can leverage the predictability of the reward to avoid conservative arm selection, such as ALBL.
\begin{figure}
  \centering
  \includegraphics[width=1\textwidth]{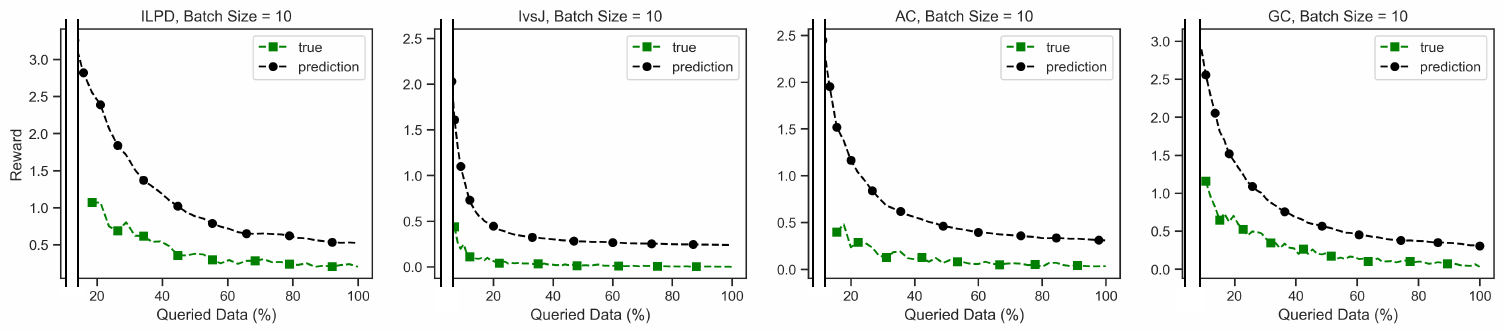}
  \caption{The prediction of reward for batch size of 10.}
  \label{fig:prediction}
\end{figure}

\subsection{Pull History}
In this section, I analyze the average percentage of times the algorithm selects the optimal arm. Figure \ref{fig:pullhis} shows that, in general, ALBL through adversarial bandit selects each arm with a 50\% chance. In contrast, after some exploration, our algorithm can switch to the optimal arm. This explains why our results are better than the existing approach. Another important observation is that after 60\% of the queries for unlabeled data points, the algorithm reaches a saturation point and appears to converge to a stable selection percentage.
\begin{figure}
  \centering
  \includegraphics[width=\textwidth]{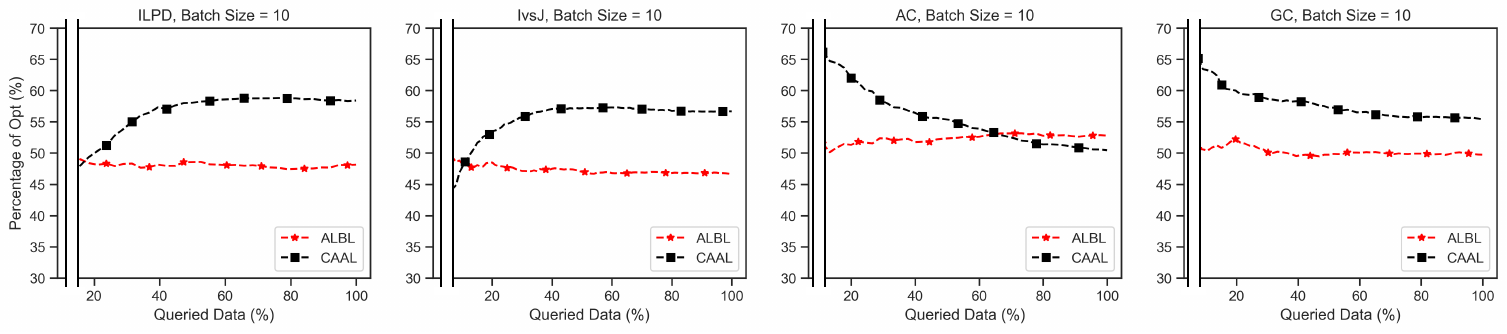}
  \caption{Comparison of the percentage of optimal strategies for batch size of 10.}
  \label{fig:pullhis}
\end{figure}

\begin{figure*}
  \centering
  \includegraphics[width=1\textwidth]{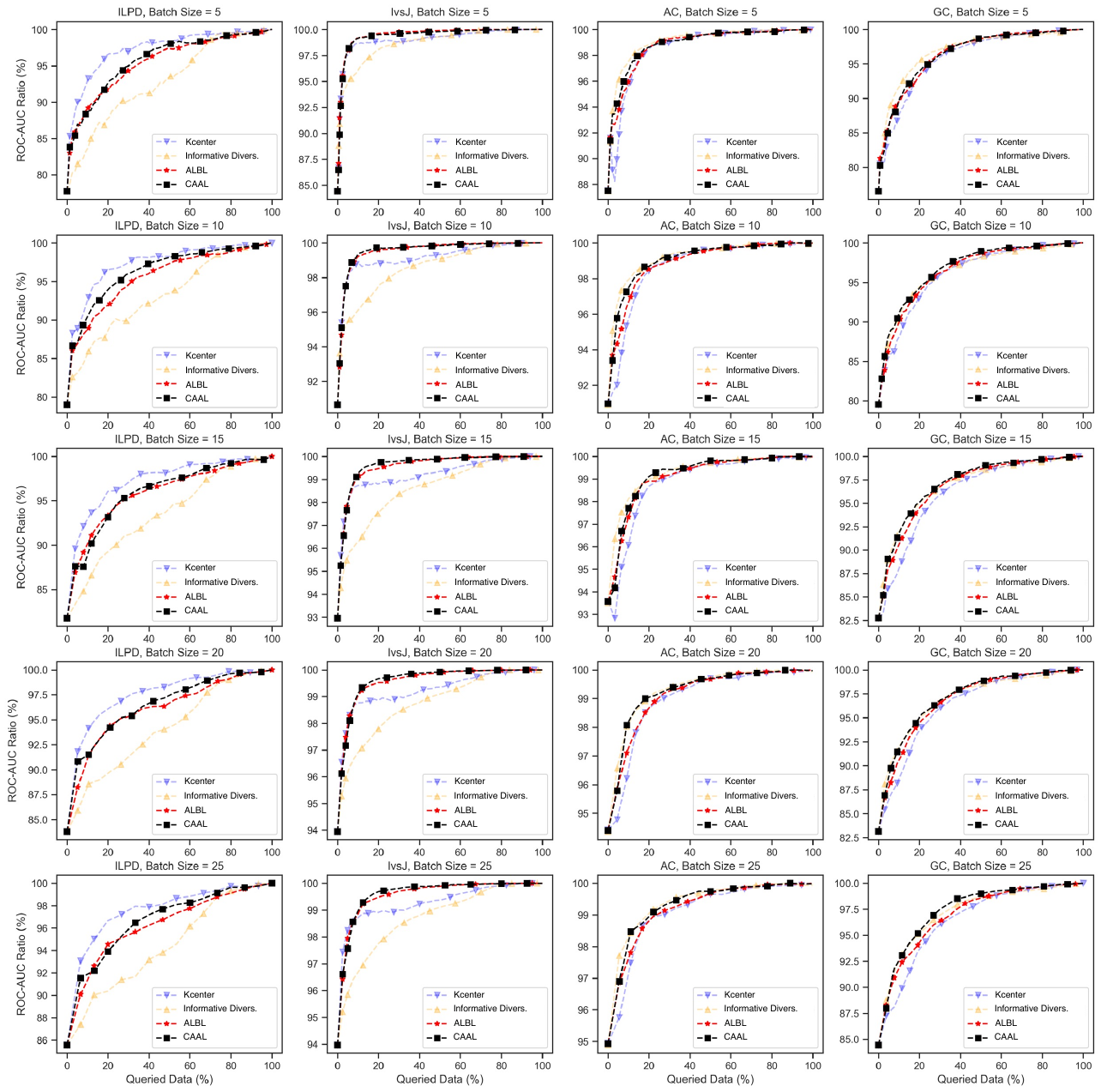}
  \caption{Comparison on ROC-AUC Ratio.}
  \label{fig:results}
\end{figure*}

\begin{table*}
\centering
\setlength{\tabcolsep}{3pt} 
  \begin{tabular}{l c c|c|c|c|c}
    \toprule
    Data &Strategies &\multicolumn{5}{c}{Batch Size}\\
    \cline{3-7}
    & & \textbf{5}&\textbf{10} & \textbf{15} & \textbf{20} & \textbf{25}\\
    \hline
    ILPD & Kc.&97.20$\pm$0.07&97.11$\pm$0.07  & 97.00$\pm$ 0.07&97.26$\pm$0.07&97.29$\pm$0.07\\
    &Info.&92.91$\pm$0.08&93.20$\pm$0.08&93.47$\pm$0.08&93.99$\pm$0.08& 94.32$\pm$0.08\\
    &\cellcolor{lightergray}ALBL&\cellcolor{lightergray}95.43$\pm$0.07&\cellcolor{lightergray}95.43$\pm$0.07&\cellcolor{lightergray}95.70$\pm$0.07&\cellcolor{lightergray}95.88$\pm$0.07&\cellcolor{lightergray}96.07$\pm$0.07\\
    &\cellcolor{lightergray}CAAL&\cellcolor{lightergray}\underline{\textbf{95.73}}$\pm$0.07&\cellcolor{lightergray}\underline{\textbf{96.22}}$\pm$0.07&\underline{\textbf{\cellcolor{lightergray}95.80}}$\pm$0.07&\cellcolor{lightergray}\underline{\textbf{96.28}}$\pm$0.07&\cellcolor{lightergray}\underline{\textbf{96.44}}$\pm$0.07\\
    \midrule
    IvsJ& Kc.&99.04$\pm$0.02&99.12$\pm$0.01&99.17$\pm$0.01&99.21$\pm$0.01&99.23$\pm$0.01\\
    &Info.&98.67$\pm$0.02&98.56$\pm$0.02&98.57$\pm$0.02&98.73$\pm$0.01&98.67$\pm$0.02\\
    &\cellcolor{lightergray}ALBL&\cellcolor{lightergray}\underline{\textbf{99.46}}$\pm$0.02&\cellcolor{lightergray}99.51$\pm$0.01&\cellcolor{lightergray}99.52$\pm$0.01&\cellcolor{lightergray}99.53$\pm$0.01&\cellcolor{lightergray}99.51$\pm$0.01 \\
    &\cellcolor{lightergray} CAAL&\cellcolor{lightergray}99.38$\pm$0.02&\cellcolor{lightergray}\underline{\textbf{99.53}}$\pm$0.01&\cellcolor{lightergray}\underline{\textbf{99.57}}$\pm$0.01&\cellcolor{lightergray}\underline{\textbf{99.56}}$\pm$0.01&\cellcolor{lightergray}\underline{\textbf{99.54}}$\pm$0.01\\
    \midrule
    AC& Kc.&98.55$\pm$0.03&98.65$\pm$0.03&98.81$\pm$0.03&98.92$\pm$0.03&99.01$\pm$0.03\\
    &Info.&
    99.05$\pm$0.03&99.12$\pm$0.03&99.22$\pm$0.03&99.23$\pm$ 0.03&99.28$\pm$ 0.03\\
    &\cellcolor{lightergray}ALBL&\cellcolor{lightergray}98.76$\pm$0.03&\cellcolor{lightergray}98.83$\pm$0.03&\cellcolor{lightergray}99.04$\pm$0.03&\cellcolor{lightergray}99.06$\pm$0.03&\cellcolor{lightergray}99.13$\pm$0.03\\
    &\cellcolor{lightergray}CAAL&\cellcolor{lightergray}\underline{\textbf{98.84}}$\pm$0.03&\cellcolor{lightergray}\underline{\textbf{99.00}}$\pm$0.03&\cellcolor{lightergray}\underline{\textbf{99.12}}$\pm$0.03&\cellcolor{lightergray}\underline{\textbf{99.18}}$\pm$0.03&\cellcolor{lightergray}\underline{\textbf{99.22}}$\pm$0.03 \\
    \midrule
    GC& Kc..&95.98$\pm$0.05&96.06$\pm$0.05&96.10$\pm$0.05&96.26$\pm$0.05&96.32$\pm$0.05\\
    &Info.&96.90$\pm$0.05&96.45$\pm$0.05&96.89$\pm$0.05&96.87$\pm$0.04&97.03$\pm$0.04\\
    &\cellcolor{lightergray}ALBL&\cellcolor{lightergray}\underline{\textbf{96.47}}$\pm$0.05&\cellcolor{lightergray}96.46$\pm$0.05&\cellcolor{lightergray}96.76$\pm$0.05&\cellcolor{lightergray}96.74$\pm$0.05&\cellcolor{lightergray}96.82$\pm$0.04 \\
    &\cellcolor{lightergray}CAAL&\cellcolor{lightergray}96.46$\pm$0.04&\cellcolor{lightergray}\underline{\textbf{96.81}}$\pm$0.04&\cellcolor{lightergray}\underline{\textbf{97.06}}$\pm$0.04&\cellcolor{lightergray}\underline{\textbf{96.98}}$\pm$0.04&\cellcolor{lightergray}\underline{\textbf{97.24}}$\pm$0.04\\
    \bottomrule
  \end{tabular}
\caption{Area Under the Curve of the ROC-AUC Ratio. In the table, Kc. represents the K-Center strategy, and Info. represents the Informative Diversity strategy. These are two arms used in the adaptive active learning algorithms.}
\label{tab:results}
\end{table*}


\section{Conclusion}
In this part, I introduced CAAL, a new adaptive batch selection active learning algorithm based on contextual bandit. Our results show that CAAL outperforms the existing ALBL method on real-world datasets, \textbf{regardless of batch size}, by \textbf{predicting rewards based on context}. In the future, I plan to explore other designs for context and rewards.

There are some limitations to our work. Bandit algorithms require exploration before exploitation, which depends on the number of hand-craft active learning strategies. If the strategy set is large, the agent may only switch between hand-craft active learning strategies with equal probabilities. Expert knowledge is still needed to select proper hand-craft active learning strategies set. Additionally, theoretical analysis is challenging. Although \cite{zhang2023algorithm} provides regret analysis in their setup, their assumptions about the stationary distribution and the independent, identical distribution of each arm do not hold even in their case. Therefore, I provide experimental results rather than theoretical analysis.
\chapter{Conclusion and Discussion}
\label{conclusion_chapter}

In the first part of this thesis, I focus on developing fully distributed algorithms for Generalized Nash Equilibrium (GNE) computation. I first develop a continuous-time formulation, accompanied by rigorous theoretical analysis establishing well-posedness and convergence properties. The proposed framework avoids the need for multiplier exchange, thereby preserving privacy and reducing communication overhead. I then introduce discretization schemes to bridge the gap between continuous-time dynamics and practical implementations, while maintaining convergence guarantees under suitable conditions. The overall framework accommodates different classes of shared constraints and provides a systematic approach for analyzing distributed equilibrium-seeking dynamics.

In the second part of the thesis, I study active learning in machine learning, with a particular focus on improving adaptability under unknown data distributions. I propose a contextual bandit-based framework that selects among multiple hand-crafted active learning strategies by leveraging contextual information derived from the learning process. This approach addresses the limitations of existing adversarial bandit methods, which tend to be overly conservative, and enables more effective strategy selection. The proposed method is validated on both Amazon’s internal datasets and publicly available external datasets, demonstrating improved performance and robustness across different data regimes. The method has also been deployed in Amazon’s internal production systems. However, only the external dataset results are presented due to confidentiality constraints. Furthermore, the framework is designed to be industrial-friendly and readily integrable into existing machine learning pipelines.

Overall, this thesis contributes to both distributed optimization and machine learning by developing scalable, adaptive, and communication-efficient algorithms. The results highlight the importance of incorporating structural insights and contextual information into algorithm design, and suggest several promising directions for future research, including extensions to more general equilibrium problems, stochastic environments, and large-scale real-world systems.

\section{Future work}
Future work for the GNEP part includes establishing convergence results under more general shared linear inequality constraints and further extending the framework to games with general convex individual constraints. These extensions would significantly broaden the applicability of the proposed fully distributed algorithms to more realistic multi-agent optimization and coordination problems. For the contextual bandit active learning part, one limitation of the current framework is that the exploration complexity depends on the number of arms. As a result, the number of candidate strategies cannot be excessively large, and the framework still requires domain expertise to design suitable candidate arms. Future work may focus on exploring more effective context and reward designs in experiments, as well as providing a stronger theoretical analysis for the proposed contextual bandit active learning framework.


\bibliography{references}

@article{facchinei-fischer-piccialli-2007,
  title = {Generalized {Nash} equilibrium problems and {Newton} methods},
  author = {Facchinei, Francisco and Fischer, Andreas and Piccialli, Veronica},
  journal = {Mathematical Programming},
  volume = {117},
  number = {1--2},
  pages = {163--194},
  year = {2009},
  month = mar,
  doi = {10.1007/s10107-007-0160-2}
}

@article{comple,
  title = {The complexity of computing a {Nash} equilibrium},
  author = {Daskalakis, Constantinos and Goldberg, Paul W. and Papadimitriou, Christos H.},
  journal = {SIAM Journal on Computing},
  volume = {39},
  number = {1},
  pages = {195--259},
  year = {2009},
  month = jan,
  doi = {10.1137/070699652}
}

@inproceedings{maiorano-dynamics-2000,
  title = {Dynamics of non-collusive oligopolistic electricity markets},
  author = {Maiorano, A. and Song, Y. H. and Trovato, M.},
  booktitle = {2000 IEEE Power Engineering Society Winter Meeting Conference Proceedings},
  volume = {2},
  pages = {838--844},
  address = {Singapore},
  publisher = {IEEE},
  year = {2000},
  doi = {10.1109/PESW.2000.850034}
}

@article{bacciotti-nonpathological-2006,
  title = {Nonpathological {Lyapunov} functions and discontinuous {Carathéodory} systems},
  author = {Bacciotti, Andrea and Ceragioli, Francesca},
  journal = {Automatica},
  volume = {42},
  number = {3},
  pages = {453--458},
  year = {2006},
  month = mar,
  doi = {10.1016/j.automatica.2005.10.014}
}

@inproceedings{pmlr-v139-liu21d,
  title = {Dynamic game theoretic neural optimizer},
  author = {Liu, Guan-Horng and Chen, Tianrong and Theodorou, Evangelos},
  booktitle = {Proceedings of the 38th International Conference on Machine Learning},
  editor = {Meila, Marina and Zhang, Tong},
  volume = {139},
  series = {Proceedings of Machine Learning Research},
  pages = {6759--6769},
  publisher = {PMLR},
  year = {2021},
  month = jul
}

@article{GAN,
  title = {Generative adversarial networks},
  author = {Goodfellow, Ian and Pouget-Abadie, Jean and Mirza, Mehdi and Xu, Bing and Warde-Farley, David and Ozair, Sherjil and Courville, Aaron and Bengio, Yoshua},
  journal = {Communications of the Association for Computing Machinery},
  volume = {63},
  number = {11},
  pages = {139--144},
  publisher = {Association for Computing Machinery},
  address = {New York, NY, USA},
  year = {2020},
  month = oct,
  doi = {10.1145/3422622}
}

@article{krawczyk-relaxation-2000,
  title = {Relaxation algorithms to find {Nash} equilibria with economic applications},
  author = {Krawczyk, Jacek B. and Uryasev, Stanislav},
  journal = {Environmental Modeling \& Assessment},
  volume = {5},
  number = {1},
  pages = {63--73},
  year = {2000},
  month = jan,
  doi = {10.1023/A:1019097208499}
}

@article{zhou-generalized-2005,
  title = {The generalized {Nash} equilibrium model for oligopolistic transit market with elastic demand},
  author = {Zhou, Jing and Lam, William H. K. and Heydecker, Benjamin G.},
  journal = {Transportation Research Part B: Methodological},
  volume = {39},
  number = {6},
  pages = {519--544},
  year = {2005},
  month = jul,
  doi = {10.1016/j.trb.2004.07.003}
}

@book{ryu-large-scale-2022,
  title = {Large-scale convex optimization: algorithms \& analyses via monotone operators},
  author = {Ryu, Ernest K. and Yin, Wotao},
  publisher = {Cambridge University Press},
  edition = {1},
  year = {2022},
  month = nov,
  doi = {10.1017/9781009160865}
}

@article{facchinei_generalized_2010,
  title = {Generalized {Nash} equilibrium problems},
  author = {Facchinei, Francisco and Kanzow, Christian},
  journal = {Annals of Operations Research},
  volume = {175},
  number = {1},
  pages = {177--211},
  year = {2010},
  month = mar,
  doi = {10.1007/s10479-009-0653-x}
}

@article{facchinei_generalized_2007,
  title = {On generalized {Nash} games and variational inequalities},
  author = {Facchinei, Francisco and Fischer, Andreas and Piccialli, Veronica},
  journal = {Operations Research Letters},
  volume = {35},
  number = {2},
  pages = {159--164},
  year = {2007},
  month = mar,
  doi = {10.1016/j.orl.2006.03.004}
}

@article{zhang_stability_1995,
  title = {On the stability of projected dynamical systems},
  author = {Zhang, D. and Nagurney, A.},
  journal = {Journal of Optimization Theory and Applications},
  volume = {85},
  number = {1},
  pages = {97--124},
  year = {1995},
  month = apr,
  doi = {10.1007/BF02192301}
}

@article{facchinei_computation_2011,
  title = {On the computation of all solutions of jointly convex generalized {Nash} equilibrium problems},
  author = {Facchinei, Francisco and Sagratella, Simone},
  journal = {Optimization Letters},
  volume = {5},
  number = {3},
  pages = {531--547},
  year = {2011},
  month = aug,
  doi = {10.1007/s11590-010-0218-6}
}

@article{pang_quasi-variational_2005,
  title = {Quasi-variational inequalities, generalized {Nash} equilibria, and multi-leader-follower games},
  author = {Pang, Jong-Shi and Fukushima, Masao},
  journal = {Computational Management Science},
  volume = {2},
  number = {1},
  pages = {21--56},
  year = {2005},
  month = jan,
  doi = {10.1007/s10287-004-0010-0}
}

@article{migot_parametrized_2020,
  title = {A parametrized variational inequality approach to track the solution set of a generalized {Nash} equilibrium problem},
  author = {Migot, Tangi and Cojocaru, Monica-G.},
  journal = {European Journal of Operational Research},
  volume = {283},
  number = {3},
  pages = {1136--1147},
  year = {2020},
  month = jun,
  doi = {10.1016/j.ejor.2019.11.054}
}

@article{kanzow_augmented_2016,
  title = {Augmented {Lagrangian} methods for the solution of generalized {Nash} equilibrium problems},
  author = {Kanzow, Christian and Steck, Daniel},
  journal = {SIAM Journal on Optimization},
  volume = {26},
  number = {4},
  pages = {2034--2058},
  year = {2016},
  month = jan,
  doi = {10.1137/16M1068256}
}

@article{rosen_existence_1965,
  title = {Existence and uniqueness of equilibrium points for concave {N}-person games},
  author = {Rosen, J. B.},
  journal = {Econometrica},
  volume = {33},
  number = {3},
  pages = {520},
  year = {1965},
  month = jul,
  doi = {10.2307/1911749}
}

@article{littman_1994,
  title = {Markov games as a framework for multi-agent reinforcement learning},
  author = {Littman, Michael L.},
  journal = {Machine Learning Proceedings 1994},
  pages = {157--163},
  year = {1994},
  doi = {10.1016/B978-1-55860-335-6.50027-1}
}

@inproceedings{meanfield,
  title = {Mean field multi-agent reinforcement learning},
  author = {Yang, Yaodong and Luo, Rui and Li, Minne and Zhou, Ming and Zhang, Weinan and Wang, Jun},
  booktitle = {Proceedings of the 35th International Conference on Machine Learning},
  editor = {Dy, Jennifer and Krause, Andreas},
  volume = {80},
  series = {Proceedings of Machine Learning Research},
  pages = {5571--5580},
  publisher = {PMLR},
  year = {2018},
  month = jul
}

@inproceedings{MPG,
  title = {Global convergence of multi-agent policy gradient in {Markov} potential games},
  author = {Leonardos, Stefanos and Overman, Will and Panageas, Ioannis and Piliouras, Georgios},
  booktitle = {ICLR 2022 Workshop on Gamification and Multiagent Solutions},
  year = {2022}
}

@article{lin_distributed_2022,
  title = {Distributed generalized {Nash} equilibrium seeking: a singular perturbation-based approach},
  author = {Lin, Wen-Ting and Chen, Guo and Li, Chaojie and Huang, Tingwen},
  journal = {Neurocomputing},
  volume = {482},
  pages = {278--286},
  year = {2022},
  month = apr,
  doi = {10.1016/j.neucom.2021.11.073}
}

@inproceedings{pmlr-v119-jin20e,
  title = {What is local optimality in nonconvex-nonconcave minimax optimization?},
  author = {Jin, Chi and Netrapalli, Praneeth and Jordan, Michael},
  booktitle = {Proceedings of the 37th International Conference on Machine Learning},
  editor = {Daumé III, Hal and Singh, Aarti},
  volume = {119},
  series = {Proceedings of Machine Learning Research},
  pages = {4880--4889},
  publisher = {PMLR},
  year = {2020},
  month = jul
}

@book{glynn-robinson-facchinei-2004,
  title = {Finite-dimensional variational inequalities and complementarity problems},
  author = {Facchinei, Francisco and Pang, Jong-Shi},
  series = {Springer Series in Operations Research and Financial Engineering},
  publisher = {Springer-Verlag New York},
  address = {New York, NY},
  year = {2004}
}

@article{harker-1991,
  title = {Generalized {Nash} games and quasi-variational inequalities},
  author = {Harker, Patrick T.},
  journal = {European Journal of Operational Research},
  volume = {54},
  number = {1},
  pages = {81--94},
  year = {1991},
  month = sep,
  doi = {10.1016/0377-2217(91)90325-P}
}

@article{yi-pavel-2019,
  title = {An operator splitting approach for distributed generalized {Nash} equilibria computation},
  author = {Yi, Peng and Pavel, Lacra},
  journal = {Automatica},
  volume = {102},
  pages = {111--121},
  year = {2019},
  month = apr,
  doi = {10.1016/j.automatica.2019.01.008}
}

@article{facchinei-generalized-2010,
  title = {Generalized {Nash} equilibrium problems},
  author = {Facchinei, Francisco and Kanzow, Christian},
  journal = {Annals of Operations Research},
  volume = {175},
  number = {1},
  pages = {177--211},
  year = {2010},
  month = mar,
  doi = {10.1007/s10479-009-0653-x}
}

@article{cherukuri-asymptotic-2016,
  title = {Asymptotic convergence of constrained primal--dual dynamics},
  author = {Cherukuri, Ashish and Mallada, Enrique and Cortés, Jorge},
  journal = {Systems \& Control Letters},
  volume = {87},
  pages = {10--15},
  year = {2016},
  month = jan,
  doi = {10.1016/j.sysconle.2015.10.006}
}

@inproceedings{ebrahimi-robust-2019,
  title = {Robust optimization via discrete-time saddle point algorithm},
  author = {Ebrahimi, Keivan and Vaidya, Umesh and Elia, Nicola},
  booktitle = {2019 IEEE 58th Conference on Decision and Control},
  pages = {2473--2478},
  address = {Nice, France},
  publisher = {IEEE},
  year = {2019},
  month = dec,
  doi = {10.1109/CDC40024.2019.9028953}
}

@article{debreu-social-1952,
  title = {A social equilibrium existence theorem},
  author = {Debreu, Gerard},
  journal = {Proceedings of the National Academy of Sciences},
  volume = {38},
  number = {10},
  pages = {886--893},
  year = {1952},
  month = oct,
  doi = {10.1073/pnas.38.10.886}
}

@inproceedings{huang-distributed-2021,
  title = {Distributed solution of {GNEP} over networks via the {Douglas}-{Rachford} splitting method},
  author = {Huang, Yuanhanqing and Hu, Jianghai},
  booktitle = {2021 IEEE 60th Conference on Decision and Control},
  pages = {3110--3116},
  address = {Austin, TX, USA},
  publisher = {IEEE},
  year = {2021},
  month = dec,
  doi = {10.1109/CDC45484.2021.9682862}
}

@article{NashQ,
  title = {{Nash} q-learning for general-sum stochastic games},
  author = {Hu, Junling and Wellman, Michael P.},
  journal = {Journal of Machine Learning Research},
  volume = {4},
  pages = {1039--1069},
  year = {2003},
  month = dec
}

@article{bianchi-continuous-time-2021,
  title = {Continuous-time fully distributed generalized {Nash} equilibrium seeking for multi-integrator agents},
  author = {Bianchi, Mattia and Grammatico, Sergio},
  journal = {Automatica},
  volume = {129},
  pages = {109660},
  year = {2021},
  month = jul,
  doi = {10.1016/j.automatica.2021.109660}
}

@article{cenedese-asynchronous-2021,
  title = {An asynchronous distributed and scalable generalized {Nash} equilibrium seeking algorithm for strongly monotone games},
  author = {Cenedese, Carlo and Belgioioso, Giuseppe and Grammatico, Sergio and Cao, Ming},
  journal = {European Journal of Control},
  volume = {58},
  pages = {143--151},
  year = {2021},
  month = mar,
  doi = {10.1016/j.ejcon.2020.08.006}
}

@incollection{ROBBINS1971233,
  title = {A convergence theorem for non-negative almost supermartingales and some applications},
  author = {Robbins, H. and Siegmund, D.},
  editor = {Rustagi, Jagdish S.},
  booktitle = {Optimizing Methods in Statistics},
  publisher = {Academic Press},
  pages = {233--257},
  year = {1971},
  doi = {10.1016/B978-0-12-604550-5.50015-8}
}

@book{glynn_robinson_facchinei_2004,
  title = {Finite-dimensional variational inequalities and complementarity problems},
  author = {Facchinei, Francisco and Pang, Jong-Shi},
  publisher = {Springer New York},
  address = {New York, NY},
  year = {2004}
}

@article{migot_nonsmooth_2020,
  title = {Nonsmooth dynamics of generalized {Nash} games},
  author = {Migot, Tangi and Cojocaru, Monica-G.},
  journal = {Journal of Nonlinear and Variational Analysis},
  volume = {4},
  number = {1},
  year = {2020},
  doi = {10.23952/jnva.4.2020.1.04}
}

@article{shapley_stochastic_1953,
  title = {Stochastic games},
  author = {Shapley, Lloyd Stowell},
  journal = {Proceedings of the National Academy of Sciences of the United States of America},
  volume = {39},
  number = {10},
  pages = {1095--1100},
  year = {1953},
  month = oct,
  doi = {10.1073}
}

@article{fukushima_restricted_2011,
  title = {Restricted generalized {Nash} equilibria and controlled penalty algorithm},
  author = {Fukushima, Masao},
  journal = {Computational Management Science},
  volume = {8},
  number = {3},
  pages = {201--218},
  year = {2011},
  month = aug,
  doi = {10.1007/s10287-009-0097-4}
}

@article{nabetani_parametrized_2011,
  title = {Parametrized variational inequality approaches to generalized {Nash} equilibrium problems with shared constraints},
  author = {Nabetani, Koichi and Tseng, Paul and Fukushima, Masao},
  journal = {Computational Optimization and Applications},
  volume = {48},
  number = {3},
  pages = {423--452},
  year = {2011},
  month = apr,
  doi = {10.1007/s10589-009-9256-3}
}

@article{dupuis_dynamical_1993,
  title = {Dynamical systems and variational inequalities},
  author = {Dupuis, Paul and Nagurney, Anna},
  journal = {Annals of Operations Research},
  volume = {44},
  number = {1},
  pages = {7--42},
  year = {1993},
  month = feb,
  doi = {10.1007/BF02073589}
}

@article{Xia_Wang_2000,
  title = {On the stability of globally projected dynamical systems},
  author = {Xia, Y. S. and Wang, J.},
  journal = {Journal of Optimization Theory and Applications},
  volume = {106},
  number = {1},
  pages = {129--150},
  year = {2000},
  doi = {10.1023/A:1004611224835}
}

@article{dreves_nonsmooth_2011,
  title = {Nonsmooth optimization reformulations characterizing all solutions of jointly convex generalized {Nash} equilibrium problems},
  author = {Dreves, Axel and Kanzow, Christian},
  journal = {Computational Optimization and Applications},
  volume = {50},
  number = {1},
  pages = {23--48},
  year = {2011},
  month = sep,
  doi = {10.1007/s10589-009-9314-x}
}

@article{facchinei_penalty_2010,
  title = {Penalty methods for the solution of generalized {Nash} equilibrium problems},
  author = {Facchinei, Francisco and Kanzow, Christian},
  journal = {SIAM Journal on Optimization},
  volume = {20},
  number = {5},
  pages = {2228--2253},
  year = {2010},
  month = jan,
  doi = {10.1137/090749499}
}

@article{tsaknakis_minimax_2023,
  title = {Minimax problems with coupled linear constraints: computational complexity and duality},
  author = {Tsaknakis, Ioannis and Hong, Mingyi and Zhang, Shuzhong},
  journal = {SIAM Journal on Optimization},
  volume = {33},
  number = {4},
  pages = {2675--2702},
  year = {2023},
  month = dec,
  doi = {10.1137/21M1462428}
}

@inproceedings{pmlr-v48-pedregosa16,
  title = {Hyperparameter optimization with approximate gradient},
  author = {Pedregosa, Fabian},
  booktitle = {Proceedings of the 33rd International Conference on Machine Learning},
  editor = {Balcan, Maria Florina and Weinberger, Kilian Q.},
  volume = {48},
  series = {Proceedings of Machine Learning Research},
  pages = {737--746},
  address = {New York, NY, USA},
  publisher = {PMLR},
  year = {2016},
  month = jun
}

@inproceedings{NEURIPS2021_174a61b0,
  title = {Convex-concave min-max {Stackelberg} games},
  author = {Goktas, Denizalp and Greenwald, Amy},
  booktitle = {Advances in Neural Information Processing Systems},
  editor = {Ranzato, M. and Beygelzimer, A. and Dauphin, Y. and Liang, P. S. and Wortman Vaughan, J.},
  volume = {34},
  pages = {2991--3003},
  publisher = {Curran Associates, Inc.},
  year = {2021}
}

@inproceedings{NEURIPS2022_24f420aa,
  title = {Turbocharging solution concepts: solving {NE}s, {CE}s and {CCE}s with neural equilibrium solvers},
  author = {Marris, Luke and Gemp, Ian and Anthony, Thomas and Tacchetti, Andrea and Liu, Siqi and Tuyls, Karl},
  booktitle = {Advances in Neural Information Processing Systems},
  editor = {Koyejo, S. and Mohamed, S. and Agarwal, A. and Belgrave, D. and Cho, K. and Oh, A.},
  volume = {35},
  pages = {5586--5600},
  publisher = {Curran Associates, Inc.},
  year = {2022}
}

@inproceedings{paccagnan_distributed_2016,
  title = {Distributed computation of generalized {Nash} equilibria in quadratic aggregative games with affine coupling constraints},
  author = {Paccagnan, Dario and Gentile, Basilio and Parise, Francesca and Kamgarpour, Maryam and Lygeros, John},
  booktitle = {2016 IEEE 55th Conference on Decision and Control},
  pages = {6123--6128},
  address = {Las Vegas, NV, USA},
  publisher = {IEEE},
  year = {2016},
  month = dec,
  doi = {10.1109/CDC.2016.7799210}
}

@article{contreras_numerical_2004,
  title = {Numerical solutions to {Nash}--{Cournot} equilibria in coupled constraint electricity markets},
  author = {Contreras, J. and Klusch, M. and Krawczyk, J. B.},
  journal = {IEEE Transactions on Power Systems},
  volume = {19},
  number = {1},
  pages = {195--206},
  year = {2004},
  month = feb,
  doi = {10.1109/TPWRS.2003.820692}
}

@book{zhang_projected_1996,
  title = {Projected dynamical systems and variational inequalities with applications},
  author = {Zhang, Ding and Nagurney, Anna},
  series = {International Series in Operations Research and Management Science},
  publisher = {Kluwer Academic Publishers},
  address = {Dordrecht, The Netherlands},
  year = {1996}
}

@inproceedings{Shen2023,
  title = {Multi-agent reinforcement learning for resource allocation in large-scale robotic warehouse sortation centers},
  author = {Shen, Yi and McClosky, Benjamin and Zavlanos, Michael},
  booktitle = {2023 IEEE Conference on Decision and Control},
  year = {2023}
}

@article{ma_decentralized_2013,
  title = {Decentralized charging control of large populations of plug-in electric vehicles},
  author = {Ma, Zhongjing and Callaway, Duncan S. and Hiskens, Ian A.},
  journal = {IEEE Transactions on Control Systems Technology},
  volume = {21},
  number = {1},
  pages = {67--78},
  year = {2013},
  month = jan,
  doi = {10.1109/TCST.2011.2174059}
}

@inproceedings{le_cleach_algames_2020,
  title = {{ALGAMES}: a fast solver for constrained dynamic games},
  author = {Le Cleac'h, Simon and Schwager, Mac and Manchester, Zachary},
  booktitle = {Robotics: Science and Systems XVI},
  publisher = {Robotics: Science and Systems Foundation},
  year = {2020},
  month = jul,
  doi = {10.15607/RSS.2020.XVI.091}
}

@inproceedings{10.5555/2969239.2969435,
  title = {Communication complexity of distributed convex learning and optimization},
  author = {Arjevani, Yossi and Shamir, Ohad},
  booktitle = {Proceedings of the 29th International Conference on Neural Information Processing Systems},
  series = {NIPS'15},
  pages = {1756--1764},
  address = {Cambridge, MA, USA},
  publisher = {MIT Press},
  year = {2015}
}

@inproceedings{McMahan2016CommunicationEfficientLO,
  title = {Communication-efficient learning of deep networks from decentralized data},
  author = {McMahan, H. B. and Moore, Eider and Ramage, Daniel and Hampson, Seth and Ag{\"u}era y Arcas, Blaise},
  booktitle = {International Conference on Artificial Intelligence and Statistics},
  year = {2016}
}

@inproceedings{NEURIPS2022_56bd2125,
  title = {Towards optimal communication complexity in distributed non-convex optimization},
  author = {Patel, Kumar Kshitij and Wang, Lingxiao and Woodworth, Blake E. and Bullins, Brian and Srebro, Nati},
  booktitle = {Advances in Neural Information Processing Systems},
  editor = {Koyejo, S. and Mohamed, S. and Agarwal, A. and Belgrave, D. and Cho, K. and Oh, A.},
  volume = {35},
  pages = {13316--13328},
  publisher = {Curran Associates, Inc.},
  year = {2022}
}

@inproceedings{pmlr-v162-mishchenko22b,
  title = {{ProxSkip}: yes! local gradient steps provably lead to communication acceleration! finally!},
  author = {Mishchenko, Konstantin and Malinovsky, Grigory and Stich, Sebastian and Richtarik, Peter},
  booktitle = {Proceedings of the 39th International Conference on Machine Learning},
  editor = {Chaudhuri, Kamalika and Jegelka, Stefanie and Song, Le and Szepesvari, Csaba and Niu, Gang and Sabato, Sivan},
  volume = {162},
  series = {Proceedings of Machine Learning Research},
  pages = {15750--15769},
  publisher = {PMLR},
  year = {2022},
  month = jul
}

@inproceedings{pmlr-v70-suresh17a,
  title = {Distributed mean estimation with limited communication},
  author = {Suresh, Ananda Theertha and Yu, Felix X. and Kumar, Sanjiv and McMahan, H. Brendan},
  booktitle = {Proceedings of the 34th International Conference on Machine Learning},
  editor = {Precup, Doina and Teh, Yee Whye},
  volume = {70},
  series = {Proceedings of Machine Learning Research},
  pages = {3329--3337},
  publisher = {PMLR},
  year = {2017},
  month = aug
}

@inproceedings{NEURIPS2018_3328bdf9,
  title = {Gradient sparsification for communication-efficient distributed optimization},
  author = {Wangni, Jianqiao and Wang, Jialei and Liu, Ji and Zhang, Tong},
  booktitle = {Advances in Neural Information Processing Systems},
  editor = {Bengio, S. and Wallach, H. and Larochelle, H. and Grauman, K. and Cesa-Bianchi, N. and Garnett, R.},
  volume = {31},
  publisher = {Curran Associates, Inc.},
  year = {2018}
}

@incollection{Davis2016,
  title = {Convergence rate analysis of several splitting schemes},
  author = {Davis, Damek and Yin, Wotao},
  booktitle = {Splitting Methods in Communication, Imaging, Science, and Engineering},
  pages = {115--163},
  publisher = {Springer International Publishing},
  address = {Cham},
  year = {2016},
  doi = {10.1007/978-3-319-41589-5_4}
}

@article{shorinwa_distributed_2024,
  title = {Distributed optimization methods for multi-robot systems: part 1, a tutorial},
  author = {Shorinwa, Ola and Halsted, Trevor and Yu, Javier and Schwager, Mac},
  journal = {IEEE Robotics \& Automation Magazine},
  volume = {31},
  number = {3},
  pages = {121--138},
  year = {2024},
  month = sep,
  doi = {10.1109/MRA.2024.3358718}
}

@article{Nemirovski2009,
  title = {Robust stochastic approximation approach to stochastic programming},
  author = {Nemirovski, Arkadi and Juditsky, Anatoli and Lan, Guanghui and Shapiro, Alexander},
  journal = {SIAM Journal on Optimization},
  volume = {19},
  number = {4},
  pages = {1574--1609},
  year = {2009},
  doi = {10.1137/070704277}
}

@book{Nesterov2004,
  title = {Introductory lectures on convex optimization: a basic course},
  author = {Nesterov, Yurii},
  series = {Applied Optimization},
  volume = {87},
  publisher = {Kluwer Academic Publishers},
  address = {Boston, MA},
  year = {2004},
  doi = {10.1007/978-1-4419-8853-9}
}

@book{settles_active_2010,
  title = {Active learning literature survey},
  author = {Settles, Burr},
  series = {Computer Sciences Technical Report},
  number = {1648},
  publisher = {University of Wisconsin--Madison},
  year = {2010},
  month = jan
}

@article{balcan_true_2010,
  title = {The true sample complexity of active learning},
  author = {Balcan, Maria-Florina and Hanneke, Steve and Vaughan, Jennifer Wortman},
  journal = {Machine Learning},
  volume = {80},
  number = {2--3},
  pages = {111--139},
  year = {2010},
  month = sep,
  doi = {10.1007/s10994-010-5174-y}
}

@inproceedings{NIPS2005_340a3904,
  title = {Query by committee made real},
  author = {Gilad-Bachrach, Ran and Navot, Amir and Tishby, Naftali},
  booktitle = {Advances in Neural Information Processing Systems},
  editor = {Weiss, Y. and Schölkopf, B. and Platt, J.},
  volume = {18},
  publisher = {MIT Press},
  year = {2005}
}

@inproceedings{sener2018active,
  title = {Active learning for convolutional neural networks: a core-set approach},
  author = {Sener, Ozan and Savarese, Silvio},
  booktitle = {International Conference on Learning Representations},
  year = {2018}
}

@inproceedings{hoi_batch_2006,
  title = {Batch mode active learning and its application to medical image classification},
  author = {Hoi, Steven C. H. and Jin, Rong and Zhu, Jianke and Lyu, Michael R.},
  booktitle = {Proceedings of the 23rd International Conference on Machine Learning},
  pages = {417--424},
  address = {Pittsburgh, PA, USA},
  publisher = {ACM Press},
  year = {2006},
  doi = {10.1145/1143844.1143897}
}

@article{huang_active_2014,
  title = {Active learning by querying informative and representative examples},
  author = {Huang, Sheng-Jun and Jin, Rong and Zhou, Zhi-Hua},
  journal = {IEEE Transactions on Pattern Analysis and Machine Intelligence},
  volume = {36},
  number = {10},
  pages = {1936--1949},
  year = {2014},
  month = oct,
  doi = {10.1109/TPAMI.2014.2307881}
}

@inproceedings{dasgupta_hierarchical_2008,
  title = {Hierarchical sampling for active learning},
  author = {Dasgupta, Sanjoy and Hsu, Daniel},
  booktitle = {Proceedings of the 25th International Conference on Machine Learning},
  pages = {208--215},
  address = {Helsinki, Finland},
  publisher = {ACM Press},
  year = {2008},
  doi = {10.1145/1390156.1390183}
}

@article{Hsu_Lin_2015,
  title = {Active learning by learning},
  author = {Hsu, Wei-Ning and Lin, Hsuan-Tien},
  journal = {Proceedings of the AAAI Conference on Artificial Intelligence},
  volume = {29},
  number = {1},
  year = {2015},
  month = feb,
  doi = {10.1609/aaai.v29i1.9597}
}

@inproceedings{NIPS2017_8ca8da41,
  title = {Learning active learning from data},
  author = {Konyushkova, Ksenia and Sznitman, Raphael and Fua, Pascal},
  booktitle = {Advances in Neural Information Processing Systems},
  editor = {Guyon, I. and Von Luxburg, U. and Bengio, S. and Wallach, H. and Fergus, R. and Vishwanathan, S. and Garnett, R.},
  volume = {30},
  publisher = {Curran Associates, Inc.},
  year = {2017}
}

@inproceedings{9093390,
  title = {Active adversarial domain adaptation},
  author = {Su, Jong-Chyi and Tsai, Yi-Hsuan and Sohn, Kihyuk and Liu, Buyu and Maji, Subhransu and Chandraker, Manmohan},
  booktitle = {2020 IEEE Winter Conference on Applications of Computer Vision},
  pages = {728--737},
  year = {2020},
  doi = {10.1109/WACV45572.2020.9093390}
}

@inproceedings{10.1007/978-3-540-74958-5_14,
  title = {Dual strategy active learning},
  author = {Donmez, Pinar and Carbonell, Jaime G. and Bennett, Paul N.},
  editor = {Kok, Joost N. and Koronacki, Jacek and Mantaras, Raomon Lopez de and Matwin, Stan and Mladeni{\v{c}}, Dunja and Skowron, Andrzej},
  booktitle = {Machine Learning: ECML 2007},
  pages = {116--127},
  address = {Berlin, Heidelberg},
  publisher = {Springer Berlin Heidelberg},
  year = {2007}
}

@article{baram_online_2004,
  title = {Online choice of active learning algorithms},
  author = {Baram, Yoram and Yaniv, Ran El and Luz, Kobi},
  journal = {The Journal of Machine Learning Research},
  volume = {5},
  pages = {255--291},
  year = {2004},
  month = may
}

@inproceedings{Casanova2020Reinforced,
  title = {Reinforced active learning for image segmentation},
  author = {Casanova, Arantxa and Pinheiro, Pedro O. and Rostamzadeh, Negar and Pal, Christopher J.},
  booktitle = {International Conference on Learning Representations},
  year = {2020}
}

@inproceedings{pmlr-v89-cheung19b,
  title = {Learning to optimize under non-stationarity},
  author = {Cheung, Wang Chi and Simchi-Levi, David and Zhu, Ruihao},
  booktitle = {Proceedings of the Twenty-Second International Conference on Artificial Intelligence and Statistics},
  editor = {Chaudhuri, Kamalika and Sugiyama, Masashi},
  volume = {89},
  series = {Proceedings of Machine Learning Research},
  pages = {1079--1087},
  publisher = {PMLR},
  year = {2019},
  month = apr
}

@inproceedings{10.5555/3454287.3455365,
  title = {Weighted linear bandits for non-stationary environments},
  author = {Russac, Yoan and Vernade, Claire and Cappé, Olivier},
  booktitle = {Advances in Neural Information Processing Systems},
  editor = {Wallach, H. and Larochelle, H. and Beygelzimer, A. and Alché-Buc, F. d' and Fox, E. and Garnett, R.},
  volume = {32},
  publisher = {Curran Associates, Inc.},
  year = {2019}
}

@inproceedings{zhang2023algorithm,
  title = {Algorithm selection for deep active learning with imbalanced datasets},
  author = {Zhang, Jifan and Shao, Shuai and Verma, Saurabh and Nowak, Robert},
  booktitle = {Advances in Neural Information Processing Systems},
  year = {2023}
}

@misc{google_google/active-learning_nodate,
  title = {Google/active-learning},
  author = {Google},
  year = {2017},
  howpublished = {GitHub repository}
}

@inproceedings{zhan_comparative_2021,
  title = {A comparative survey: benchmarking for pool-based active learning},
  author = {Zhan, Xueying and Liu, Huan and Li, Qing and Chan, Antoni B.},
  booktitle = {Proceedings of the Thirtieth International Joint Conference on Artificial Intelligence},
  pages = {4679--4686},
  address = {Montreal, Canada},
  publisher = {International Joint Conferences on Artificial Intelligence Organization},
  year = {2021},
  month = aug,
  doi = {10.24963/ijcai.2021/634}
}

@inproceedings{li_contextual-bandit_2010,
  title = {A contextual-bandit approach to personalized news article recommendation},
  author = {Li, Lihong and Chu, Wei and Langford, John and Schapire, Robert E.},
  booktitle = {Proceedings of the 19th International Conference on World Wide Web},
  pages = {661--670},
  address = {Raleigh, NC, USA},
  publisher = {ACM},
  year = {2010},
  month = apr,
  doi = {10.1145/1772690.1772758}
}

@inproceedings{citovsky2021batch,
  title = {Batch active learning at scale},
  author = {Citovsky, Gui and DeSalvo, Giulia and Gentile, Claudio and Karydas, Lazaros and Rajagopalan, Anand and Rostamizadeh, Afshin and Kumar, Sanjiv},
  booktitle = {Advances in Neural Information Processing Systems},
  editor = {Beygelzimer, A. and Dauphin, Y. and Liang, P. and Wortman Vaughan, J.},
  year = {2021}
}

@inproceedings{Ash2020Deep,
  title = {Deep batch active learning by diverse, uncertain gradient lower bounds},
  author = {Ash, Jordan T. and Zhang, Chicheng and Krishnamurthy, Akshay and Langford, John and Agarwal, Alekh},
  booktitle = {International Conference on Learning Representations},
  year = {2020}
}

@inproceedings{ash2021gone,
  title = {Gone fishing: neural active learning with fisher embeddings},
  author = {Ash, Jordan T. and Goel, Surbhi and Krishnamurthy, Akshay and Kakade, Sham M.},
  booktitle = {Advances in Neural Information Processing Systems},
  editor = {Beygelzimer, A. and Dauphin, Y. and Liang, P. and Wortman Vaughan, J.},
  year = {2021}
}

@inproceedings{NEURIPS2019_95323660,
  title = {{BatchBALD}: efficient and diverse batch acquisition for deep {Bayesian} active learning},
  author = {Kirsch, Andreas and van Amersfoort, Joost and Gal, Yarin},
  booktitle = {Advances in Neural Information Processing Systems},
  editor = {Wallach, H. and Larochelle, H. and Beygelzimer, A. and d'Alché-Buc, F. and Fox, E. and Garnett, R.},
  volume = {32},
  publisher = {Curran Associates, Inc.},
  year = {2019}
}

@article{shwartz-ziv_tabular_2022,
  title = {Tabular data: deep learning is not all you need},
  author = {Shwartz-Ziv, Ravid and Armon, Amitai},
  journal = {Information Fusion},
  volume = {81},
  pages = {84--90},
  year = {2022},
  month = may,
  doi = {10.1016/j.inffus.2021.11.011}
}


\end{document}